\documentclass[10pt,twocolumn,letterpaper]{article}

\usepackage{iccv}
\usepackage{times}
\usepackage{epsfig}
\usepackage{graphicx}
\usepackage{amsmath}
\usepackage{amssymb}

\usepackage[pagebackref=true,breaklinks=true,letterpaper=true,colorlinks,bookmarks=false]{hyperref}

\iccvfinalcopy 

\def\iccvPaperID{****} 

\ificcvfinal\fi

\usepackage[utf8]{inputenc} 
\usepackage[T1]{fontenc}    
\usepackage{url}            
\usepackage{booktabs}       
\usepackage{amsfonts}       
\usepackage{nicefrac}       
\usepackage{microtype}      
\usepackage{xcolor}         
\usepackage{graphicx}
\usepackage{algorithm}
\usepackage{amsmath}
\usepackage{amsthm}
\usepackage{amssymb}
\usepackage{bm}
\usepackage{mathtools}
\usepackage{listings}
\usepackage{color}
\usepackage{caption}

\newcommand{\argmin}{\mathop{\rm arg~min}\limits}

\theoremstyle{definition}
\newtheorem{definition}{Definition}[section]
\newtheorem{theorem}{Theorem}[section]

\newtheorem{lemma}[theorem]{Lemma}

\definecolor{OliveGreen}{rgb}{0.0,0.6,0.0}
\definecolor{Orenge}{rgb}{0.89,0.55,0}
\definecolor{SkyBlue}{rgb}{0.28, 0.28, 0.95}
\makeatletter
\newcommand{\figcaption}[1]{\def\@captype{figure}\caption{#1}}
\newcommand{\tblcaption}[1]{\def\@captype{table}\caption{#1}}
\makeatother

\usepackage[capitalize]{cleveref}
\crefname{section}{Sec.}{Secs.}
\Crefname{section}{Section}{Sections}
\Crefname{table}{Table}{Tables}
\crefname{table}{Tab.}{Tabs.}

\begin{document}

\title{SHIFT15M: Fashion-specific dataset for set-to-set matching with several distribution shifts}

\author{Masanari Kimura\\
ZOZO Research\\
masanari.kimura@zozo.com\\
\and
Takuma Nakamura\\
ZOZO Research\\
takuma.nakamura@zozo.com\\
\and
Yuki Saito\\
ZOZO Research\\
yuki.saito@zozo.com
}
\maketitle

\begin{center}
\centering
    \captionsetup{type=figure}
    \includegraphics[width=0.95\linewidth]{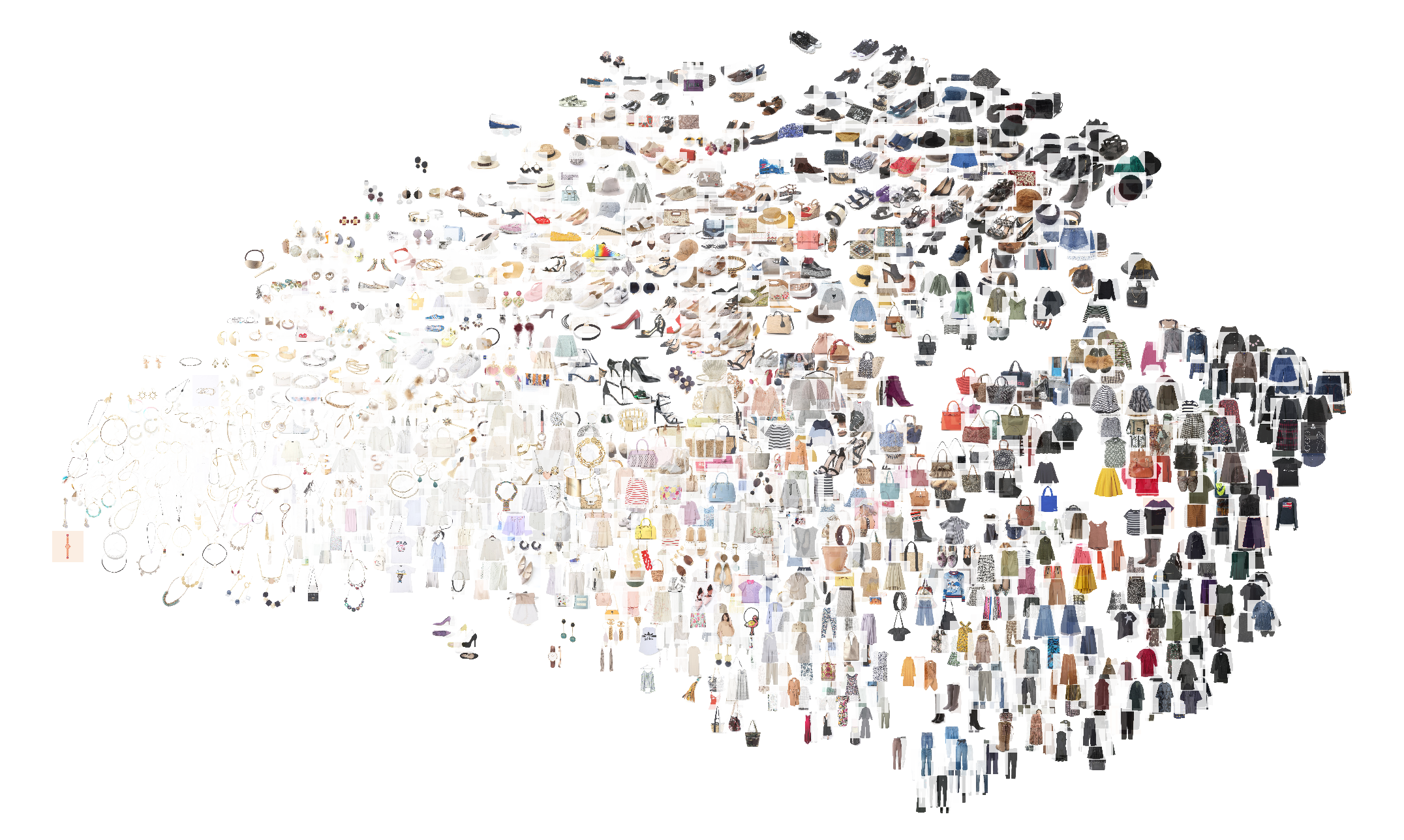}
    \captionof{figure}{t-SNE~\cite{van2008visualizing} visualization for the SHIFT15M.}
    \label{fig:shift15m_tsne}
\end{center}

\begin{abstract}
This paper addresses the problem of set-to-set matching, which involves matching two different sets of items based on some criteria, especially in the case of high-dimensional items like images. Although neural networks have been applied to solve this problem, most machine learning-based approaches assume that the training and test data follow the same distribution, which is not always true in real-world scenarios. To address this limitation, we introduce SHIFT15M, a dataset that can be used to evaluate set-to-set matching models when the distribution of data changes between training and testing. We conduct benchmark experiments that demonstrate the performance drop of naive methods due to distribution shift. Additionally, we provide software to handle the SHIFT15M dataset in a simple manner, with the URL for the software to be made available after publication of this manuscript. We believe proposed SHIFT15M dataset provide a valuable resource for evaluating set-to-set matching models under the distribution shift.
\end{abstract}

\section{Introduction}
One of the key problems for fashion data analysis is set-to-set matching~\cite{saito2020exchangeable,bai2018convolutional,arandjelovic2014discriminative}.
For example, we can consider a task that measures the degree of completion of an outfit by matching sets of clothing items (i.e., for two sets $A = \{\text{hat}, \text{shirt}, \text{skirt}\}$ and $B = \{\text{jacket}, \text{shoes} \}$, the matching score of $A$ and $B$ corresponds to the goodness of the outfit $A\cup B$).
To solve this, we need to investigate neural networks that handle sets~\cite{soelch2019deep,lee2019set,zaheer2017deep,khan2022transformers,tay2022efficient,wagstaff2019limitations,zhang2019deep,wagstaff2022universal}.
We summarize neural networks that deal with sets in Section~\ref{sec:related_works_and_conclusion}.

\begin{figure}
    \centering
    \includegraphics[width=0.95\linewidth]{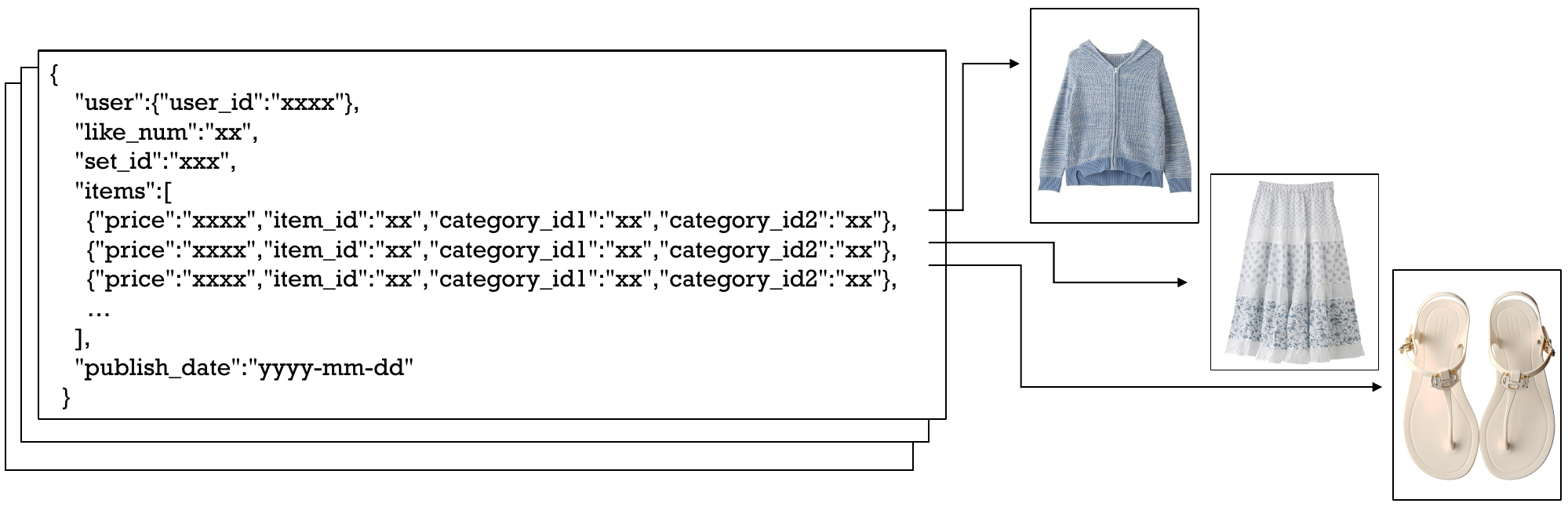}
    \caption{Overview of SHIFT15M dataset.}
    \label{fig:shift15m_overview}
\end{figure}
\begin{figure}
    \centering
    \includegraphics[width=0.98\linewidth]{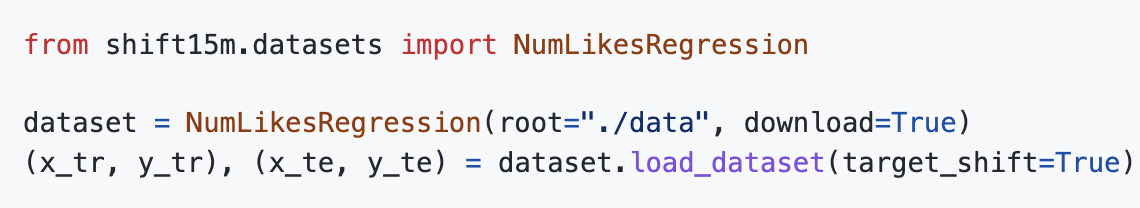}
    \caption{Minimum sample code using SHIFT15M data loader.}
    \label{fig:shift15m_software}
\end{figure}

Another common phenomenon in the domain of fashion is a trend change.
These phenomena are observed at various scales, ranging from annual trend changes such as fashionable colors to seasonal trend changes such as summer to winter clothing.
In the field of machine learning, such an assumption can be defined as a distribution shift (or dataset shift)~\cite{quinonero2008dataset,moreno2012unifying,subbaswamy2021evaluating,shen2021towards,hendrycks2021many,shen2021towards,yang2021generalized,miller2021accuracy}.
We assume that training examples $\{(\bm{x}^{tr}_i, y^{tr}_i)\}^{n_{tr}}_{i=1}$ are independently and identically distributed (i.i.d.) according to some fixed but unknown distribution $p_{tr}(\bm{x}, y)$, which can be decomposed into the marginal distribution and the conditional probability distribution, i.e., $p_{tr}(\bm{x},y)=p_{tr}(\bm{x})p_{tr}(y|\bm{x})$.
We also denote the test examples by $\{(\bm{x}^{te}_i,y^{te}_i)\}^{n_{te}}_{i=1}$ drawn from a test distribution $p_{te}(\bm{x}, y) = p_{te}(\bm{x})p_{te}(y|\bm{x})$.

\begin{figure*}[t]
    \centering
    \includegraphics[width=0.99\linewidth]{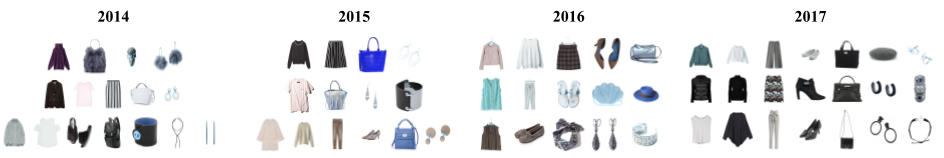}
    \caption{Several sample images from SHIFT15M dataset. See Appendix~\ref{apd:sample_items} for more sample items.}
    \label{fig:shift15m_sample_images}
\end{figure*}

\begin{table*}[t]
\centering
\caption{Statistics on the SHIFT15M dataset.}
\label{tab:number_of_instances}
\scalebox{0.8}{
\begin{tabular}{lcccccccccccc}
\toprule
Property               & Total      & 2010  & 2011   & 2012    & 2013      & 2014      & 2015       & 2016      & 2017    & 2018    & 2019    & 2020 \\
\midrule
\#sets         & 2,555,147  & 1,423 & 4,813  & 131,611 & 466,583   & 730,443   & 617,844    & 299,502   & 137,510 & 92,944  & 59,412  & 13,062 \\
\#items        & 15,218,721 & 8,327 & 29,140 & 756,532 & 2,644,564 & 4,305,802 & 3,731,864  & 1,853,647 & 855,036 & 576,022 & 373,549 & 84,238 \\
mean set size          & 6.03       & 5.85  & 6.05   & 5.74    & 5.66      & 5.89      & 6.04       & 6.18      & 6.21    & 6.19    & 6.28    & 6.44   \\
median set size        & 6.00       & 6.00  & 6.00   & 5.00    & 5.00      & 6.00      & 6.00       & 6.00      & 6.00    & 6.00    & 6.00    & 6.00   \\
mean \#likes   & 26.98      & 0.94  & 2.00   & 15.74   & 16.84     & 23.24     & 37.37      & 35.67     & 32.41   & 24.89   & 21.34   & 16.01 \\
median \#likes & 9.00       & 0.00  & 1.00   & 8.00    & 6.00      & 6.00      & 13.00      & 18.00     & 23.00   & 19.00   & 17.00   & 12.00 \\
\#unique users & 193,574    & 289   & 571    & 16,922  & 52,283    & 80,290    & 49,441     & 18,854    & 7,511   & 4,442   & 2,739   & 853   \\
\bottomrule
\end{tabular}%
}
\end{table*}

\begin{definition}{(Covariate shift~\cite{shimodaira2000improving})}
\label{def:covariate_shift_adaptation}
We consider that the two distributions $p_{tr}(\bm{x}, y)$ and $p_{te}(\bm{x},y)$ satisfy the covariate shift assumption if the following conditions hold:
\begin{align*}
    p_{tr}(\bm{x}) \neq p_{te}(\bm{x}), \quad p(y|\bm{x}) = p_{tr}(y|\bm{x}) = p_{te}(y|\bm{x}).
\end{align*}
\end{definition}

\begin{definition}{(Target shift~\cite{zhang2013domain})}
\label{def:target_shift_assumption}
We consider that the two distributions $p_{tr}(\bm{x}, y)$ and $p_{te}(\bm{x},y)$ satisfy the target shift assumption if the following conditions hold:
\begin{align*}
    p_{tr}(y) \neq p_{te}(y), \quad p(\bm{x}|y) = p_{tr}(\bm{x}|y) = p_{te}(\bm{x}|y).
\end{align*}
\end{definition}

\begin{definition}{(General distribution shift~\cite{subbaswamy2021evaluating})}
Let $\mathcal{Z}\subset\{\mathcal{Z}, \mathcal{Y}\}$ ve a set of immutable variables whose marginal distribution should remain fixed, $\mathcal{W}\subset\{\mathcal{X}, \mathcal{Y}\}\setminus \mathcal{Z}$ be a set of mutable variables whose distribution can be shifted, and $\mathcal{V}=\{\mathcal{X}, \mathcal{Y}\}\setminus\{\mathcal{W}\cup\mathcal{Z}\}$ be the remaining dependent variables.
This partition of the variables defines a factorization of $p_{tr}$ into
\begin{equation}
    p_{tr}(\bm{v} | \bm{w}, \bm{z})p_{tr}(\bm{w}|\bm{z})p_{tr}(\bm{z}),
\end{equation}
where $\bm{z}\in\mathcal{Z}$, $\bm{w}\in\mathcal{W}$ and $\bm{v}\in\mathcal{V}$.
We consider that the two distributions $p_{tr}(\bm{x}, y)$ and $p_{te}(\bm{x}, y)$ satisfy the general dataset shift assumption if the following hold:
\begin{equation}
    p_{tr}(\bm{w}|\bm{z}) \neq p_{te}(\bm{w}|\bm{z}).
\end{equation}
\end{definition}
Notably, this formulation generalizes other dataset shifts.
For example, if we let $\mathcal{Z}=\emptyset$ and $\mathcal{W}=\mathcal{X}$, then this corresponds to a covariate shift.

To address these problems, we provide SHIFT15M, a real-world dataset that can handle the above two problem settings, that is, the set-to-set matching dataset with distribution shift.
Our SHIFT15M dataset is built on data accumulated over the past 10 years in our fashion SNS (see Figure~\ref{fig:shift15m_tsne}).
In this SNS, users could post combinations of their clothing items and other users could bookmark them as favorites.
The data accumulated by this service, which has been in operation for a decade from 2010 to 2020, is very useful for dealing with distribution shifts in the fashion sector.
Figure~\ref{fig:shift15m_overview} shows an overview of the SHIFT15M dataset.
Each column is a set of posted fashion items, with information such as the user who posted, the date of publication, and the price of each item.
We hope that our SHIFT dataset will encourage research on set-to-set matching tasks under the distribution shift.

\subsection{Contribution}
Our contributions are summarized as follows:
\begin{itemize}
    \item We propose SHIFT15M, a fashion-specific dataset that can properly evaluate models for set-to-set matching under the distribution shift assumptions.
    SHIFT15M also enables the performance evaluation of the model under various magnitudes of dataset shifts by switching the magnitude. Figure~\ref{fig:shift15m_sample_images} shows several sample images from the SHIFT15M dataset.
    In Section~\ref{sec:statistics_shift15m}, we introduce overall statistics on the SHIFT15M dataset.
    \item We provide open-source software to handle the SHIFT15M dataset in a very simple way. Figure~\ref{fig:shift15m_software} shows the minimum sample code of our software;
    \item We propose first-step benchmark methods for set-to-set matching under distribution shift, numerical experiments show the usefulness of these methods.
    Section~\ref{sec:benchmarks} presents the proposed benchmark methods and the results of comparative experiments.
\end{itemize}

\begin{figure}
    \centering
    \includegraphics[width=0.95\linewidth]{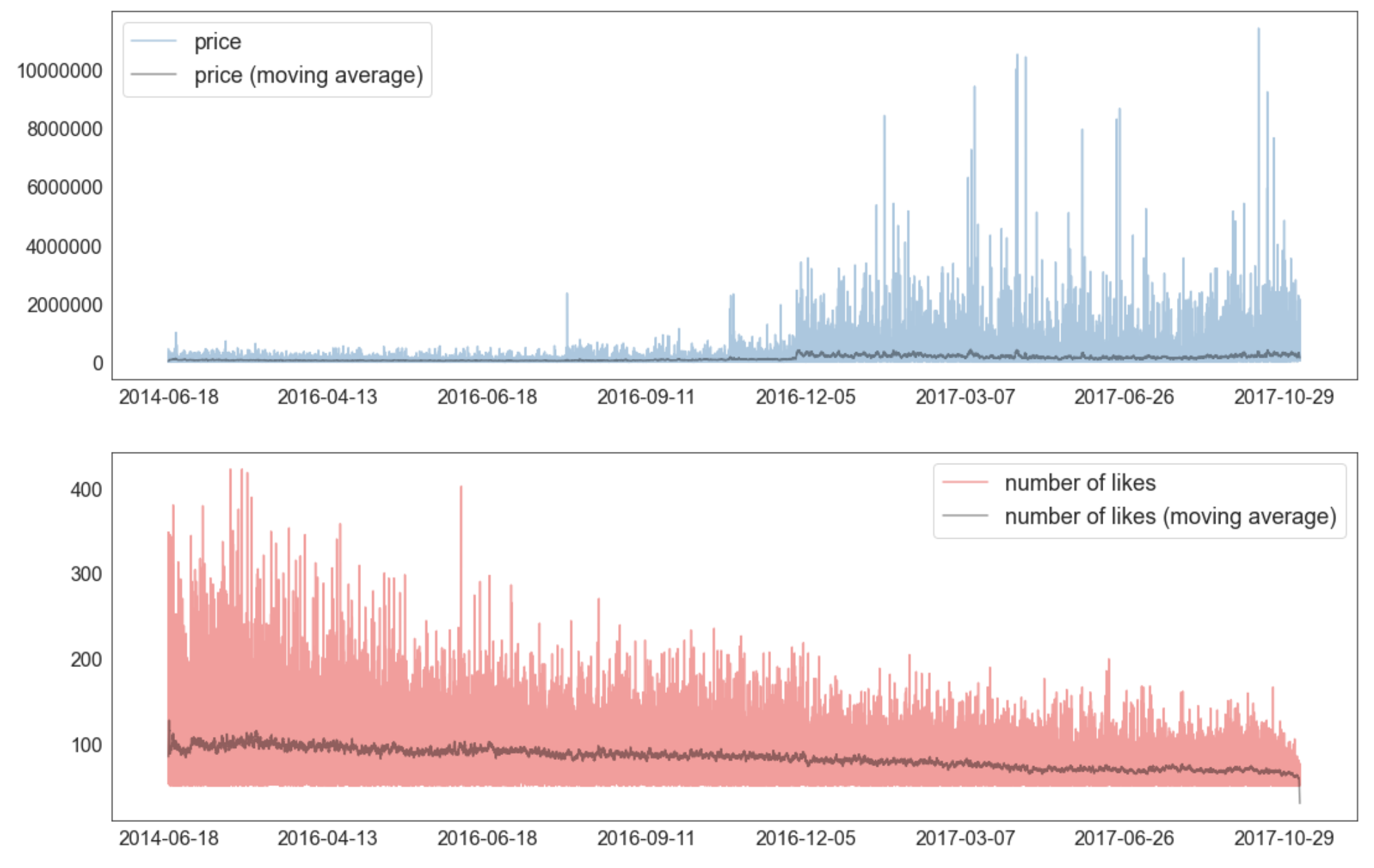}
    \caption{Top panel: the trend of price for items included in SHIFT15M. Bottom panel: the trend of the number of likes for posted sets.}
    \label{fig:shift15m_transition}
\end{figure}

\section{Statistics on the SHIFT15M dataset}
\label{sec:statistics_shift15m}
In this section, we present some statistics for our SHIFT15M dataset.
First, Table~\ref{tab:number_of_instances} shows the overview of statistics on the SHIFT15M dataset.
Since our fashion SNS was launched in 2010, the number of users and posts gradually increased from 2010, reaching a peak around 2014$\sim$2015, and slowly decreasing until 2020, the year the service was terminated.
Also, the number of items in a set tends to increase over the years, indicating that users tend to construct outfits with more and more items.

The top panel of Figure~\ref{fig:shift15m_transition} shows the trend of price for items included in the SHIFT15M dataset. It can be seen that the fashion items posted by users are becoming more expensive every year.
The bottom panel of Figure~\ref{fig:shift15m_transition} shows the trend of the number of likes for posted sets.

Figure~\ref{fig:shift15m_monthly_sets} plots the trend of the number of posted sets by year.
This figure shows that our fashion SNS, the source of the SHIFT15M dataset, was most active around 2014$\sim$2015.

Also, each item from the SHIFT15M dataset has two categories specifying the types of the item.
Figure~\ref{fig:shift15m_categories_years} show the distributions of the number of items belonging to each category.
The figures show that there are categories in which the number of items changes from year to year and categories in which the number of items does not change much throughout the entire period.

\begin{figure}
    \centering
    \includegraphics[width=0.98\linewidth]{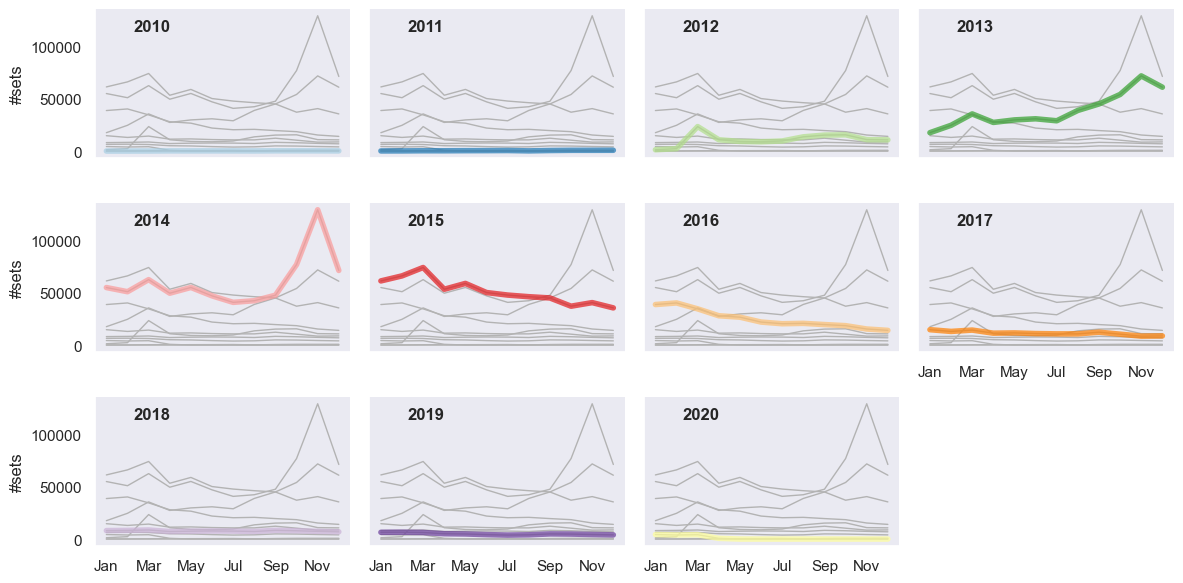}
    \caption{Trends of the number of posted sets by year.}
    \label{fig:shift15m_monthly_sets}
\end{figure}

\begin{figure}[t]
    \centering
    \includegraphics[width=0.98\linewidth]{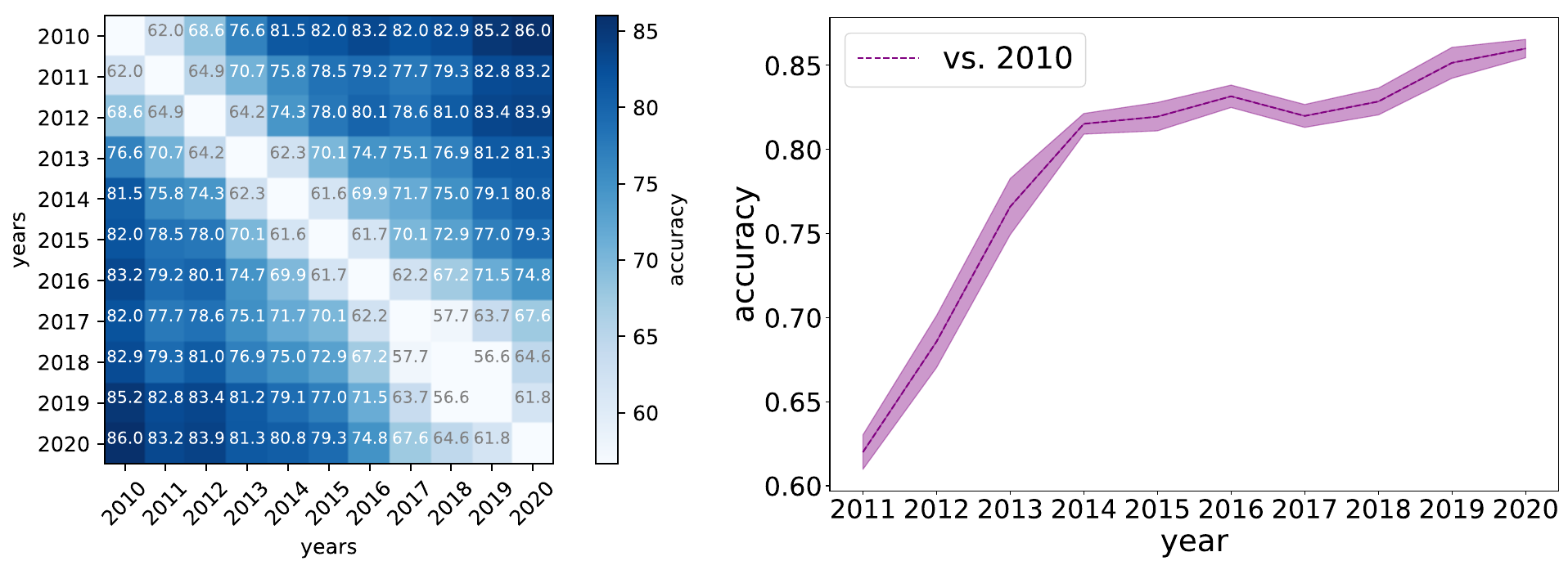}
    \caption{Covariate shift of image features.}
    \label{fig:shift15m_covariate_shift}
\end{figure}

\begin{figure}
    \centering
    \includegraphics[width=0.98\linewidth]{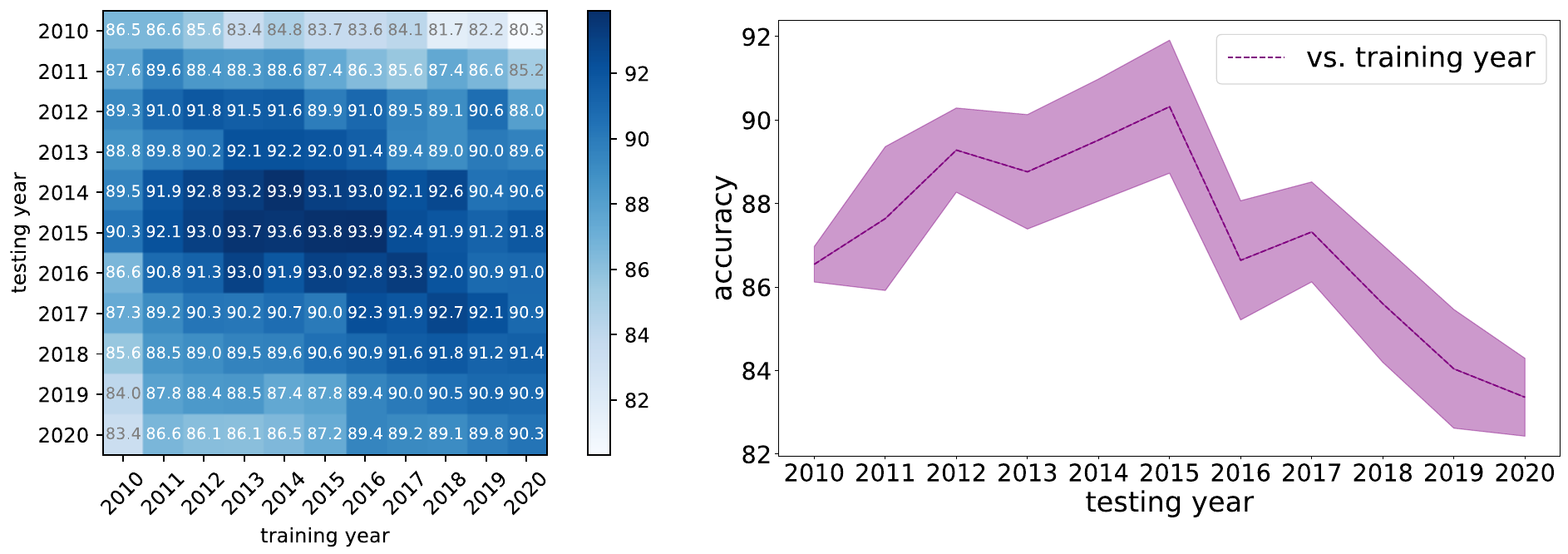}
    \caption{Category classification results under the covariate shift.}
    \label{fig:category_classification}
\end{figure}

Finally, we confirm the covariate shift of the image features included in SHIFT15M.
If covariate shift assumption~\ref{def:covariate_shift_adaptation} holds, we should be able to construct a classifier $f: \bm{x} \mapsto y = \{0, 1\}$, where $\bm{x}$ is the image feature of the item and $y$ is the binary classification output for two years.
Figure~\ref{fig:shift15m_covariate_shift} shows the experimental results.
The results show that classification between distant years (e.g., acc. of 2010 vs. 2020 is $0.85$) is easier, while classification between close years (e.g., acc. of 2010 vs. 2011 is $0.62$) is more difficult, indicating a gradual shift in image features.
Figure~\ref{fig:category_classification} also shows the experimental results of item categorization when the training and test data were generated from different years.
This figure shows that the closer the years of the training and test data are, the higher the classification accuracy.

\begin{figure*}
    \centering
    \includegraphics[width=0.6\linewidth]{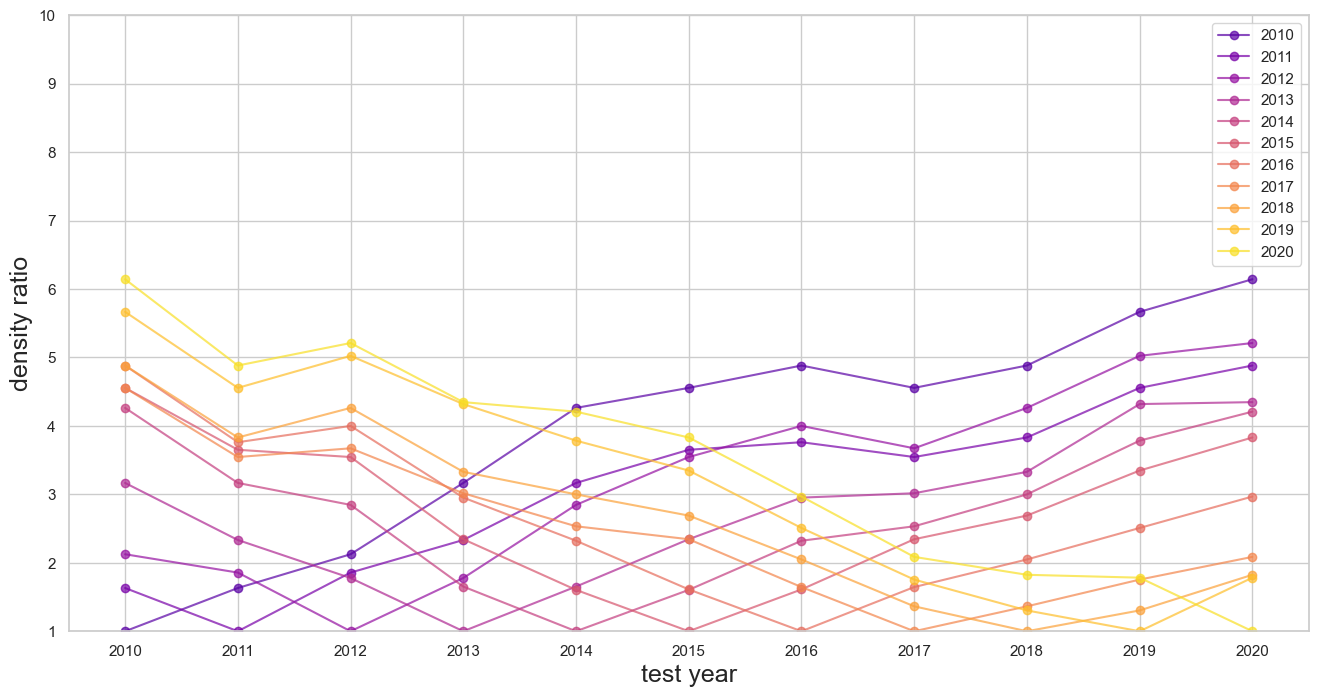}
    \caption{Density ratio estimation by using binary classifiers.}
    \label{fig:shift15m_density_ratio}
\end{figure*}

As we will see in the following sections, density ratio $p_{te}(\bm{x}) / p_{tr}(\bm{x})$ is essential in distribution shift adaptation.
Let $p(\text{train})$ and $p(\text{test})$ be the probability that some data are generated from train and test distributions, respectively.
If we assume that $p(\text{train}) = p(\text{test}) = 0.5$, we can estimate density ratio as $p(\text{test}|\bm{x}) / p(\text{train} | \bm{x})$ via Bayes' theorem:
\begin{align*}
    \frac{p_{te}(\bm{x})}{p_{tr}(\bm{x})} = \frac{p(\bm{x}|\text{test})}{p(\bm{x} | \text{train})} = \frac{p(\text{train})}{p(\text{test})}\frac{p(\text{test} | \bm{x})}{p(\text{train} | \bm{x})} = \frac{p(\text{test} | \bm{x})}{p(\text{train} | \bm{x})}.
\end{align*}
Conversely, the Bayes optimal classifier $g^*(\bm{x})$ can be written as a function of the density ratio $r(\bm{x}) = \frac{p_{te}(\bm{x})}{p_{tr}(\bm{x})}$,
\begin{align*}
    r(\bm{x}) = \frac{p_{te}(\bm{x})}{p_{tr}(\bm{x})} = \frac{g^*(\bm{x})}{1 - g^*(\bm{x})},\ g^*(\bm{x}) = \frac{p_{te}(\bm{x})}{p_{tr}(\bm{x}) + p_{tr}(\bm{x})}.
\end{align*}
Figure~\ref{fig:shift15m_density_ratio} shows the density ratio estimation by using the above binary classifiers.
These density ratios induce the importance weighted set-to-set matching algorithm in the following section.
See Appendix~\ref{apd:density_ratio} for more details.

\section{Benchmarks}
\label{sec:benchmarks}
In this section, we introduce several numerical experiments on the SHIFT15M dataset.

\subsection{Importance weighted set-to-set matching}
As the benchmark strategy for the distribution shift adaptation on the set-to-set matching, we propose importance weighted set-to-set matching which is based on IWERM.
\begin{definition}{(Importance weighted ERM~\cite{shimodaira2000improving})}
\label{def:iwerm}
Importance Weighted Empirical Risk Minimization (IWERM) uses the density ratio $p_{te}(\bm{x})/p_{tr}(\bm{x})$ as the weighting function:
\begin{equation}
    \hat{h} = \argmin_{h\in\mathcal{H}}\frac{1}{n_{tr}}\sum^{n_{tr}}_{i=1}\frac{p_{te}(\bm{x}^{tr}_i)}{p_{tr}(\bm{x}^{tr}_i)}\ell(h(\bm{x}^{tr}_i), y^{tr}_i). \label{eq:iwerm}
\end{equation}
\end{definition}

\begin{figure}
    \centering
    \includegraphics[width=0.95\linewidth]{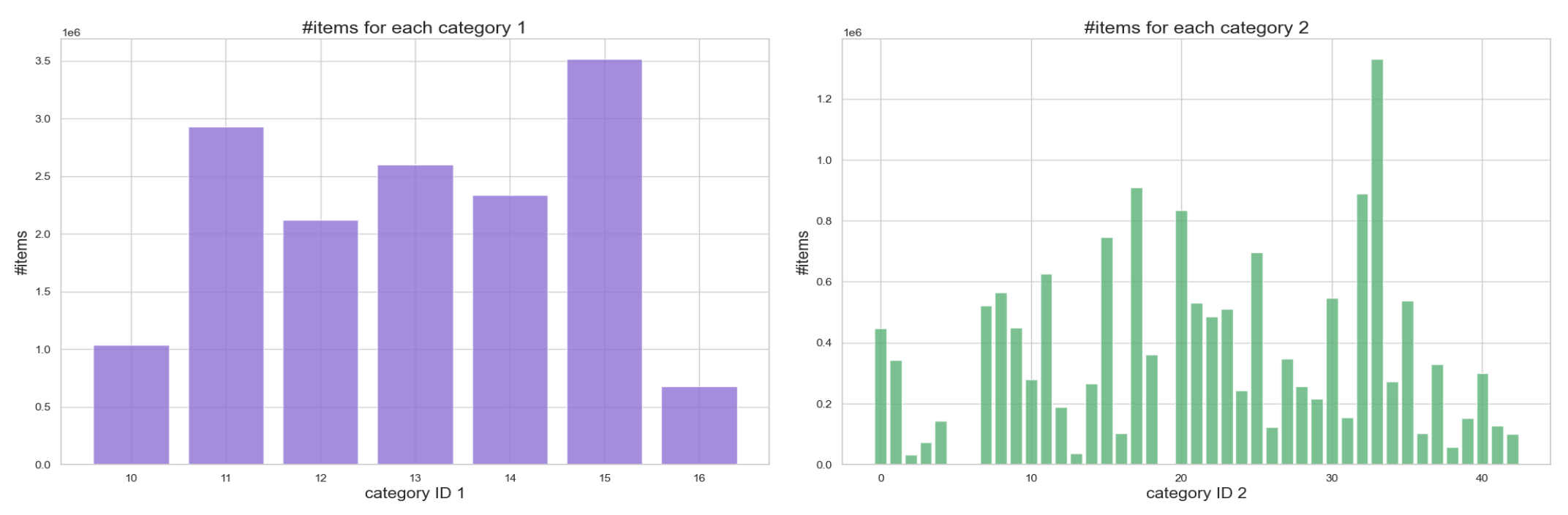}
    \caption{Overall distribution of the number of items in each category.}
    \label{fig:shift15m_categories}
\end{figure}

\begin{figure}
    \centering
    \includegraphics[width=0.95\linewidth]{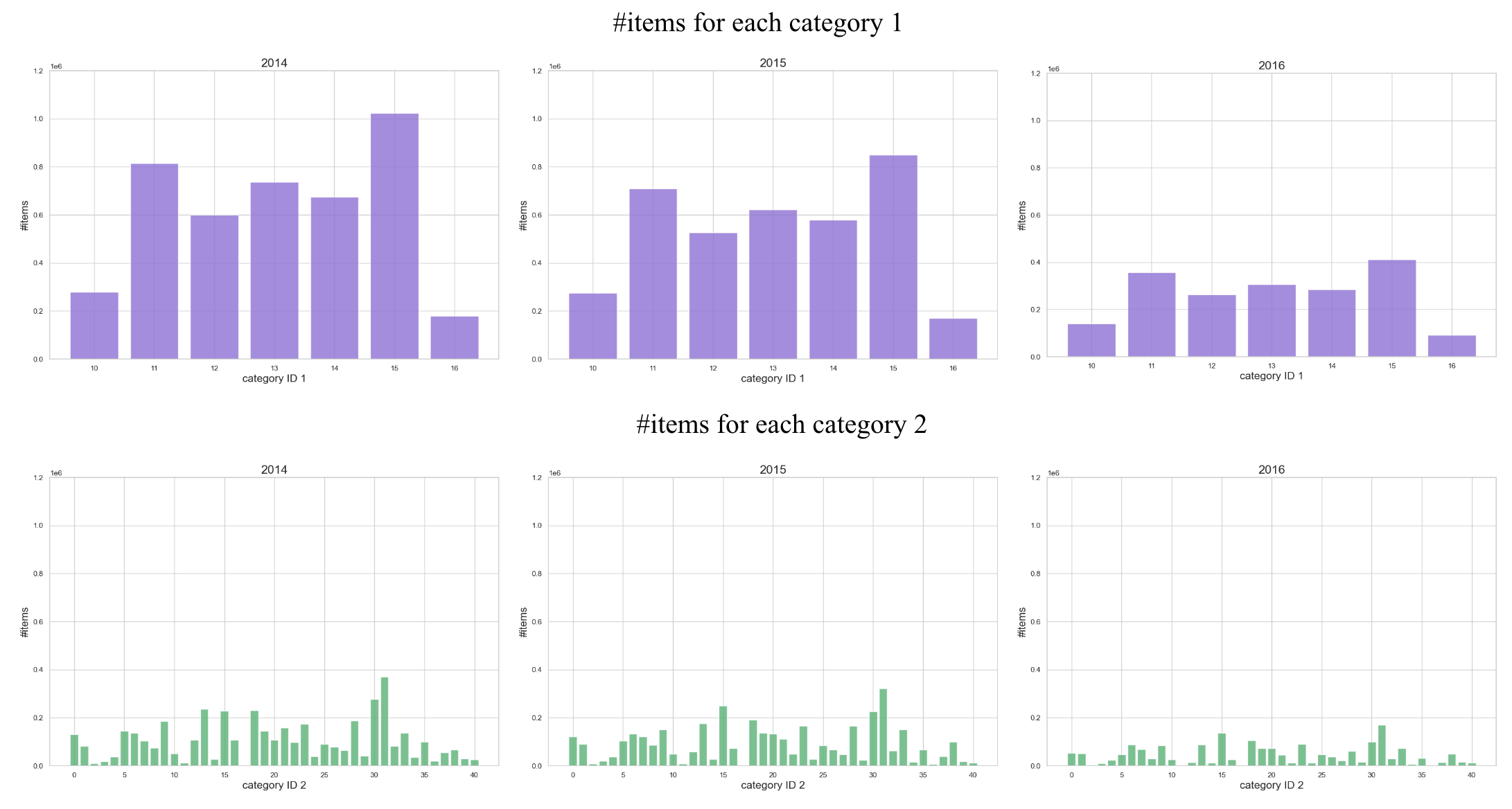}
    \caption{Yearly distribution of the number of items in each category.}
    \label{fig:shift15m_categories_years}
\end{figure}

\begin{figure*}
    \centering
    \includegraphics[width=0.9\linewidth]{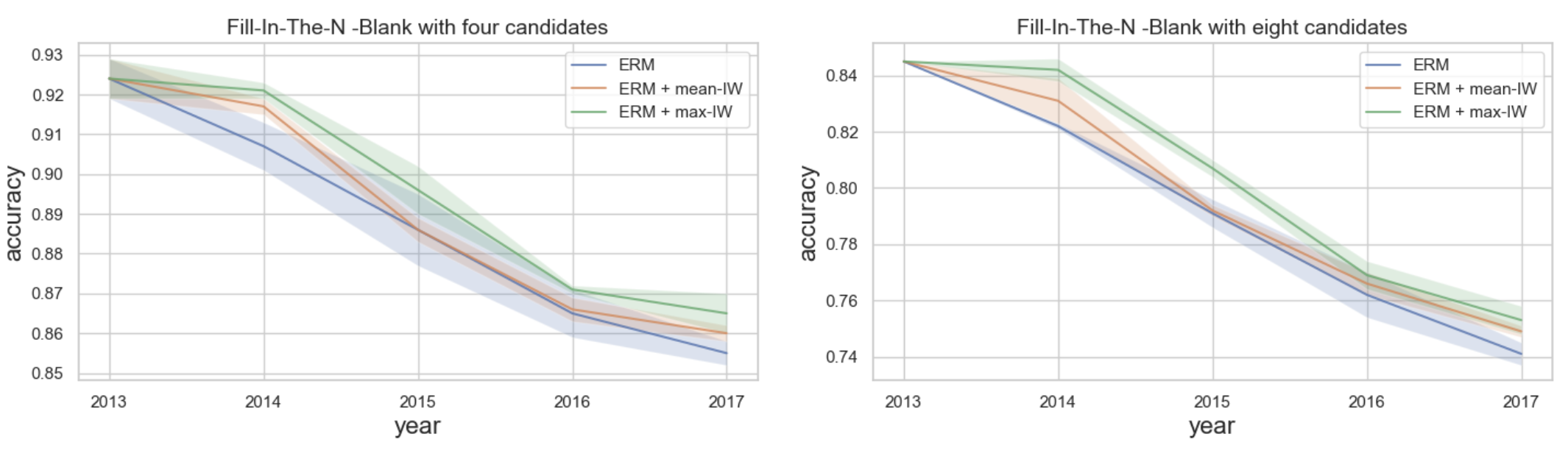}
    \caption{Plots for Fill-In-The-N-Blank experiments.}
    \label{fig:experiments_fill_in_the_n_blank}
\end{figure*}
\begin{table*}[t]
\centering
\caption{Experimental results of the Fill-In-The-$N$-Blank with four candidates. Evaluation metrics are the accuracy [$\%$]}
\label{tab:benchmark_set2set_matching_4}
\scalebox{0.95}{
\begin{tabular}{lccccc}
\toprule
Models        & 2013                & 2014                     & 2015                     & 2016                     & 2017 \\
\midrule
ERM~\cite{saito2020exchangeable}   & $0.924 (\pm 0.005)$ & $0.907(\pm 0.006)$       & $0.886(\pm 0.009)$       & $0.865(\pm 0.006)$       & $0.855(\pm 0.003)$ \\
ERM + mean-IW & $0.924 (\pm 0.005)$ & $0.917(\pm 0.002)$       & $0.886(\pm 0.003)$       & $0.866(\pm 0.003)$       & $0.860(\pm 0.002)$ \\
ERM + max-IW  & $0.924 (\pm 0.005)$ & ${\bf 0.921(\pm 0.002)}$ & ${\bf 0.896(\pm 0.006)}$ & ${\bf 0.871(\pm 0.001)}$ & ${\bf 0.865(\pm 0.005)}$ \\
\bottomrule
\end{tabular}
}
\end{table*}
\begin{table*}[]
\centering
\caption{Experimental results of the Fill-In-The-$N$-Blank with eight candidates.  Evaluation metrics are the accuracy [$\%$]}
\label{tab:benchmark_set2set_matching_8}
\scalebox{0.95}{
\begin{tabular}{lccccc}
\toprule
Models        & 2013               & 2014                     & 2015                     & 2016                     & 2017 \\
\midrule
ERM~\cite{saito2020exchangeable} & $0.845(\pm 0.000)$ & $0.822(\pm 0.001)$       & $0.791(\pm 0.005)$       & $0.762(\pm 0.008)$       & $0.741(\pm 0.004)$ \\
ERM + mean-IW & $0.845(\pm 0.000)$ & $0.831(\pm 0.008)$       & $0.792(\pm 0.002)$       & $0.766(\pm 0.004)$       & $0.749(\pm 0.002)$ \\
ERM + max-IW  & $0.845(\pm 0.000)$ & ${\bf 0.842(\pm 0.004)}$ & ${\bf 0.807(\pm 0.003)}$ & ${\bf 0.769(\pm 0.005)}$ & ${\bf 0.753(\pm 0.005)}$ \\
\bottomrule
\end{tabular}
}
\end{table*}

Adopting the density ratio as the weighting function, as in Definition~\ref{def:iwerm}, leads to the following statistically important property.
\begin{theorem}{(Consistency of IWERM\cite{shimodaira2000improving})}
If we set $w(\bm{x})=p_{te}(\bm{x})/p_{tr}(\bm{x})$ as the weighting function, the empirical error computed by the weighted ERM is a consistent estimator of the expected error in the test distribution.
\end{theorem}
Using the above ideas, we propose a novel covariate shift adaptation method for set-to-set matching.
Let $\mathcal{L}(\mathcal{V}, \mathcal{W}, f)$ be the $K$-pair-set loss~\cite{saito2020exchangeable} function for the set matching, which is defined as follows:
\begin{align}
    \mathcal{L}(\mathcal{V}, \mathcal{W}, f) = -\frac{1}{K} \sum_{i=1}^{K} \sum_{j=1}^{K} \delta_{ij} \log \frac{ \exp (f (\mathcal{V}_i, \mathcal{W}_j))}{\sum_{k=1}^{K} \exp (f(\mathcal{V}_i, \mathcal{W}_k))}, \nonumber
\label{eq:s2s}
\end{align}
where $\delta$ is Kronecker's delta, and we can modify $\mathcal{L}(\mathcal{V}, \mathcal{W}, f)$ as follows:
\begin{align}
    \mathcal{L}_w(\mathcal{V}, \mathcal{W}, f) = -\frac{1}{K} \sum_{i=1}^{K} \sum_{j=1}^{K} \delta_{ij} \Gamma^p_{i,j} \log \frac{\Gamma^f_{i,j}}{\sum_{k=1}^{K} \Gamma^f_{i,k}},
\end{align}
where $\Gamma^p_{i,j} = e^{p(test|\mathcal{V}_i \cup{\mathcal{W}_j})}$ and $\Gamma^f_{i,j} = e^{f(\mathcal{V}_i, \mathcal{W}_j)}$.
This modification can be regarded as a weighting based on the probability that the pair is included in the test set.

Here, we propose two weighting strategies:
\begin{align*}
  {\rm max\mathchar`-IW}&:\  p(test|\mathcal{V}_i \cup{\mathcal{W}_j}) = \max_{\bm{x} \in \mathcal{V}_i \cup{\mathcal{W}_j}} w(\bm{x}), \\
  {\rm mean\mathchar`-IW}&:\ p(test|\mathcal{V}_i \cup{\mathcal{W}_j}) = \frac{1}{|\mathcal{V}_i\cup\mathcal{W}_j|}\sum_{\bm{x} \in \mathcal{V}_i \cup{\mathcal{W}_j}}w(\bm{x}),
\end{align*}
where $w(\bm{x})$ is the weighting function.
Next, we approximate $w(\bm{x})$ by using unlabeled data from both $p_{tr}$ and $p_{te}$.
In IWERM, the squared error can be decomposed as follows:
\begin{align}
    \mathbb{E}_{p_{te}}\Big[\Delta^2\Big] &= \mathbb{E}_{p_{tr}}\Big[w(\bm{x})\Delta^2\Big] \nonumber \\
    &= \mathbb{E}_{p_{tr}}\Big[\hat{w}(\bm{x})\Delta^2\Big] + \mathbb{E}_{p_{tr}}\Big[(w(\bm{x}) - \hat{w}(\bm{x}))\Delta^2\Big], \nonumber
\end{align}
where $\Delta^2 = \|f(\bm{x}) - y\|^2$ and $\hat{w}(\bm{x})$ is the approximator of the weighting function $w(\bm{x})$.
The second term is bounded as
\begin{align}
    & \mathbb{E}_{p_{tr}}\Big[(w(\bm{x}) - \hat{w}(\bm{x}))\Delta^2\Big] \nonumber \\
    &\leq \frac{1}{2}\Big(\mathbb{E}_{p_{tr}}\Big[\Delta^2\Big] + \mathbb{E}_{p_{tr}}\Big[(w(\bm{x}) - \hat{w}(\bm{x}))^2\Big]\Big).
\end{align}

Let $s$ be the indicator of the distributions, where $s=1$ corresponds to the train distribution and $s=0$ corresponds to the test distribution, and we assume that $p(s) = 0.5$.
Then, we also assume that
\begin{align}
    p(\bm{x} | s) = \begin{cases}
    p_{tr}(\bm{x}) & (s=1), \\
    p_{te}(\bm{x}) & (s=0).
    \end{cases}
\end{align}

Then, we have $w(\bm{x}) = \frac{p(\bm{x} | s=0)}{p(\bm{x} | s=1)}$.
Let $g(\bm{x})$ be the optimal source discriminator which identifies whether $\bm{x}$ is generated $p_{tr}$ or $p_{te}$.
Then, we can write as $g(\bm{x}) = p(s=1|\bm{x}) = \frac{1}{1+w(\bm{x})}$.
Suppose that the density ratio $p_{te}(\bm{x})/p_{tr}(\bm{x})$ is bounded by $\beta>0$, we have $\frac{1}{1+\beta}\leq g(\bm{x})\leq 1$ for all $\bm{x}$.
From the unlabeled data generated from $p_{tr}$ and $p_{te}$, we can learn the estimator $\hat{g}$ of $g$.
Then, we can write the weight estimation term as
\begin{align}
    &\mathbb{E}_{p_{tr}}\Big[(w(\bm{x}) - \hat{w}(\bm{x}))^2\Big] = \mathbb{E}_{p_{tr}}\Biggl[\Biggl(\frac{g(\bm{x}) - \hat{g}(\bm{x})}{g(\bm{x})\hat{g}(\bm{x})}\Biggr)^2\Biggr] \nonumber \\
    &\leq (1+\beta)^4\mathbb{E}_{p_{tr}}\Big[(g(\bm{x}) - \hat{g}(\bm{x}))^2\Big] \nonumber \\
    &= (1+\beta)^4\mathbb{E}_{p_{te}}\Biggl[(g(\bm{x}) - \hat{g}(\bm{x}))^2\frac{p_{tr}(\bm{x})}{p_{te}(\bm{x})}\Biggr] \nonumber \\
    &\leq 2(1+\beta)^4\mathbb{E}_{p_{te}}\Big[(g(\bm{x}) - \hat{g}(\bm{x}))^2\Big] \nonumber \\
    &= 2(1+\beta)^4\Biggl\{\mathbb{E}_{p_{te}}\Big[(s - g(\bm{x}))^2\Big] - \mathbb{E}_{p_{te}}\Big[(g(\bm{x}) - \hat{g}(\bm{x}))\Big]\Biggr\}. \nonumber
\end{align}
This indicates that the weighting function is approximated by the function $g(\bm{x})$.

\subsection{Experimental results on set-to-set matching problem under the covariate shift assumption}

We introduce benchmark results for a set-to-set matching under the covariate shift.
The model architecture is the same as the previous work~\cite{saito2020exchangeable}, which is based on the architecture of Transformer~\cite{vaswani2017attention,lee2019set,newell2016stacked}.
Our task can be considered an extended version of a standard task, Fill-In-The-Blank~\cite{cucurull2019context}, which requires us to select an item that best extends an outfit from among four candidates. Because selecting a set corresponds to filling multiple blanks, we consider the set matching problem as Fill-In-The-$N$-Blank~\cite{saito2020exchangeable}.
To construct the correct pair of sets to be matched, we
randomly halve the given outfit $\mathcal{O}$ into two non-empty proper subsets $\mathcal{V}$ and $\mathcal{W}$, as
follows: $\mathcal{O}\to\{\mathcal{V}, \mathcal{W}\}$, where $\mathcal{V}\cap\mathcal{W}=\emptyset$.
Tables~\ref{tab:benchmark_set2set_matching_4}, \ref{tab:benchmark_set2set_matching_8} and Figure~\ref{fig:experiments_fill_in_the_n_blank} show the experimental results of the Fill-In-The-$N$-Blank with four and eight candidates.
In these experiments,  data from 2013 are used as training data, and data from 2013$\sim$2017 are used as test data.
ERM refers to empirical risk minimization~\cite{vapnik2013nature,vapnik1999overview,bousquet2003introduction}, which assumes that $p_{tr}(\bm{x}) = p_{te}(\bm{x})$.
From these results, we can see that the covariate shift adaptive set-to-set matching methods can achieve better performances than the ordinal ERM.
This means that for set-to-set matching on the SHIFT15M dataset, we need to apply some distribution shift adaptation methods.

Table~\ref{tab:benchmark_set2set_matching_4_several_models} and \ref{tab:benchmark_set2set_matching_8_several_models} show the experimental results for the various models.
The models used in the experiments are the same as \cite{saito2020exchangeable}.
To quote, 
\begin{itemize}
    \item Cross Attention and Cross Affinity~\cite{saito2020exchangeable}: Set-to-Set matching models which proposed by \cite{saito2020exchangeable}, with the attention-based and affinity-based functions, respectively.
    \item Set Transformer~\cite{lee2019set}: Set Transformer which introduced by applying a self-attention based Transformer to a set of data. Set Transformer is trained through supervised or unsupervised learning and transforms a set of data into a vector representation to recognize set features. By using Set Transformer $f_{ST}$, we perform the extension by calculating the matching score between the two sets $\mathcal{V}$ and $\mathcal{W}$ via the inner product $f_{ST}(\mathcal{V})^\top f_{ST}(\mathcal{W})$, sharing the weights between the two $f_{ST}$.
    \item We consider a union of two sets as a set-input for the extension of BERT~\cite{devlin2018bert} and omit the individual token embedding.
    We use the segment embedding to designate items of X and Y. We use three variants: $\text{BERT}_{\text{BASE}}$ is the same model as described in \cite{devlin2018bert}; $\text{BERT}_{\text{BASE-AP}}$ uses the average pooling in the last layer; and $\text{BERT}_{\text{SMALL}}$ is a four-layered version of $\text{BERT}_{\text{BASE}}$ with eight heads, and the hidden size is 512.
    \item GNN~\cite{cucurull2019context}: We combine two sets as one input for the extension of GNN~\cite{cucurull2019context}. Because this model is not presented to train in an end-to-end with the feature extractor, we do not finetune the CNN in fashion set matching, where pre-trained CNNs are used, but train it in an end-to-end manner for the group re-id task. Note that we omit the context provided from the external graphs in the evaluation stage to apply this model in the same scenarios of our tasks. We set the training epoch to 256 in the group re-id to enhance the training results of the GNN.
    \item HAP2S~\cite{yu2018hard}: A conventional CNN trained by Hard-Aware Point-to-Set loss.
\end{itemize}
See Appendix~\ref{apd:additional_figures} for the additional figures.

\begin{table*}[t]
\centering
\caption{Experimental results of the Fill-In-The-$N$-Blank with four candidates with several models. Evaluation metrics are the accuracy [$\%$].}
\label{tab:benchmark_set2set_matching_4_several_models}
\scalebox{0.87}{
\begin{tabular}{lccccc}
\toprule
Models                                               & 2013                & 2014                     & 2015                     & 2016                     & 2017              \\
\midrule
Set Transformer~\cite{lee2019set}                    & $0.791 (\pm 0.036)$ & $0.743 (\pm 0.041)$      & $0.710 (\pm 0.047)$      & $0.698 (\pm 0.050)$      & $0.675 (\pm 0.051)$     \\
Set Transformer + mean-IW                            & $0.791 (\pm 0.036)$ & $0.755 (\pm 0.037)$      & ${\bf 0.732(\pm 0.041)}$ & $0.710 (\pm 0.048)$      & $0.696 (\pm 0.049)$     \\
Set Transformer + max-IW                             & $0.791 (\pm 0.036)$ & ${\bf 0.760(\pm 0.037)}$ & $0.731 (\pm 0.038)$      & ${\bf 0.714(\pm 0.045)}$ & ${\bf 0.700(\pm 0.045)}$\\ \hline
$\text{BERT}_{\text{SMALL}}$~\cite{devlin2018bert}   & $0.898 (\pm 0.008)$ & $0.882 (\pm 0.005)$      & $0.860 (\pm 0.005)$      & $0.842 (\pm 0.007)$      & $0.830 (\pm 0.011)$     \\
$\text{BERT}_{\text{SMALL}}$ + mean-IW               & $0.898 (\pm 0.008)$ & $0.893 (\pm 0.003)$      & $0.866 (\pm 0.005)$      & $0.844 (\pm 0.007)$      & $0.839 (\pm 0.008)$     \\
$\text{BERT}_{\text{SMALL}}$ + max-IW                & $0.898 (\pm 0.008)$ & ${\bf 0.895(\pm 0.002)}$ & ${\bf 0.878(\pm 0.003)}$ & ${\bf 0.859(\pm 0.002)}$ & ${\bf 0.851(\pm 0.004)}$\\
$\text{BERT}_{\text{BASE}}$~\cite{devlin2018bert}    & $0.880 (\pm 0.011)$ & $0.875 (\pm 0.009)$      & $0.844 (\pm 0.004)$      & $0.827 (\pm 0.007)$      & $0.817 (\pm 0.010)$     \\
$\text{BERT}_{\text{BASE}}$ + mean-IW                & $0.880 (\pm 0.011)$ & ${\bf 0.878(\pm 0.010)}$ & $0.853 (\pm 0.007)$      & $0.840 (\pm 0.013)$      & $0.822 (\pm 0.004)$     \\
$\text{BERT}_{\text{BASE}}$ + max-IW                 & $0.880 (\pm 0.011)$ & $0.877(\pm 0.003)$       & ${\bf 0.870(\pm 0.012)}$ & ${\bf 0.852(\pm 0.003)}$ & ${\bf 0.830(\pm 0.003)}$\\
$\text{BERT}_{\text{BASE-AP}}$~\cite{devlin2018bert} & $0.869 (\pm 0.008)$ & $0.860 (\pm 0.007)$      & $0.825 (\pm 0.013)$      & $0.802 (\pm 0.010)$      & $0.785 (\pm 0.009)$     \\
$\text{BERT}_{\text{BASE-AP}}$ + mean-IW             & $0.869 (\pm 0.008)$ & $0.864 (\pm 0.009)$      & $0.837 (\pm 0.009)$      & $0.829 (\pm 0.010)$      & $0.797 (\pm 0.012)$     \\
$\text{BERT}_{\text{BASE-AP}}$ + max-IW              & $0.869 (\pm 0.008)$ & ${\bf 0.866(\pm 0.005)}$ & ${\bf 0.844(\pm 0.004)}$ & ${\bf 0.833(\pm 0.008)}$ & ${\bf 0.810(\pm 0.002)}$\\ \hline
GNN~\cite{cucurull2019context}                       & $0.370 (\pm 0.032)$ & $0.361 (\pm 0.035)$      & $0.355 (\pm 0.035)$      & $0.351 (\pm 0.035)$      & $0.349 (\pm 0.033)$     \\
GNN + mean-IW                                        & $0.370 (\pm 0.032)$ & $0.366 (\pm 0.033)$      & $0.360 (\pm 0.032)$      & $0.357 (\pm 0.033)$      & $0.354 (\pm 0.033)$     \\
GNN + max-IW                                         & $0.370 (\pm 0.032)$ & ${\bf 0.367(\pm 0.032)}$ & ${\bf 0.362(\pm 0.030)}$ & ${\bf 0.362(\pm 0.031)}$ & ${\bf 0.359(\pm 0.030)}$\\ \hline
HAP2S~\cite{yu2018hard}                              & $0.433 (\pm 0.041)$ & $0.420 (\pm 0.044)$      & $0.409 (\pm 0.053)$      & $0.385 (\pm 0.053)$      & $0.370 (\pm 0.062)$     \\
HAP2S + mean-IW                                      & $0.433 (\pm 0.041)$ & ${\bf 0.428(\pm 0.042)}$ & $0.415 (\pm 0.045)$      & $0.408 (\pm 0.050)$      & $0.395 (\pm 0.057)$     \\
HAP2S + max-IW                                       & $0.433 (\pm 0.041)$ & $0.427 (\pm 0.042)$      & ${\bf 0.418(\pm 0.042)}$ & ${\bf 0.413(\pm 0.047)}$ & ${\bf 0.400(\pm 0.055)}$\\ \hline
Cross Attention~\cite{saito2020exchangeable}         & $0.920 (\pm 0.007)$ & $0.892(\pm 0.009)$       & $0.867(\pm 0.012)$       & $0.844(\pm 0.012)$       & $0.836(\pm 0.014)$      \\
Cross Attention + mean-IW                            & $0.920 (\pm 0.007)$ & $0.901 (\pm 0.008)$      & $0.873(\pm 0.012)$       & ${\bf 0.850(\pm 0.013)}$ & ${\bf 0.843(\pm 0.014)}$\\
Cross Attention + max-IW                             & $0.920 (\pm 0.007)$ & ${\bf 0.913(\pm 0.009)}$ & ${\bf 0.877(\pm 0.010)}$ & $0.849(\pm 0.011)$       & $0.841(\pm 0.013)$      \\ \hline
Cross Affinity~\cite{saito2020exchangeable}          & $0.924 (\pm 0.005)$ & $0.907(\pm 0.006)$       & $0.886(\pm 0.009)$       & $0.865(\pm 0.006)$       & $0.855(\pm 0.003)$      \\
Cross Affinity + mean-IW                             & $0.924 (\pm 0.005)$ & $0.917(\pm 0.002)$       & $0.886(\pm 0.003)$       & $0.866(\pm 0.003)$       & $0.860(\pm 0.002)$      \\
Cross Affinity + max-IW                              & $0.924 (\pm 0.005)$ & ${\bf 0.921(\pm 0.002)}$ & ${\bf 0.896(\pm 0.006)}$ & ${\bf 0.871(\pm 0.001)}$ & ${\bf 0.865(\pm 0.005)}$\\
\bottomrule
\end{tabular}
}
\end{table*}

\begin{table*}[t]
\centering
\caption{Experimental results of the Fill-In-The-$N$-Blank with eight candidates with several models.  Evaluation metrics are the accuracy [$\%$]}
\label{tab:benchmark_set2set_matching_8_several_models}
\scalebox{0.87}{
\begin{tabular}{lccccc}
\toprule
Models                                               & 2013                & 2014                     & 2015                     & 2016                     & 2017                     \\
\midrule
Set Transformer~\cite{lee2019set}                    & $0.711(\pm 0.042)$  & $0.690(\pm 0.040)$       & $0.663(\pm 0.037)$       & $0.637(\pm 0.044)$       & $0.605(\pm 0.041)$       \\
Set Transformer + mean-IW                            & $0.711(\pm 0.042)$  & $0.702(\pm 0.038)$       & $0.685(\pm 0.043)$       & $0.662(\pm 0.040)$       & $0.630(\pm 0.040)$       \\
Set Transformer + max-IW                             & $0.711(\pm 0.042)$  & ${\bf 0.708(\pm 0.037)}$ & ${\bf 0.699(\pm 0.044)}$ & ${\bf 0.675(\pm 0.040)}$ & ${\bf 0.641(\pm 0.045)}$ \\ \hline
$\text{BERT}_{\text{SMALL}}$~\cite{devlin2018bert}   & $0.824(\pm 0.015)$  & $0.799(\pm 0.023)$       & $0.772(\pm 0.025)$       & $0.740(\pm 0.039)$       & $0.716(\pm 0.039)$       \\
$\text{BERT}_{\text{SMALL}}$ + mean-IW               & $0.824(\pm 0.015)$  & ${\bf 0.815(\pm 0.016)}$ & $0.787(\pm 0.023)$       & $0.759(\pm 0.031)$       & $0.720(\pm 0.030)$       \\
$\text{BERT}_{\text{SMALL}}$ + max-IW                & $0.824(\pm 0.015)$  & $0.813(\pm 0.015)$       & ${\bf 0.805(\pm 0.020)}$ & ${\bf 0.782(\pm 0.035)}$ & ${\bf 0.732(\pm 0.030)}$ \\
$\text{BERT}_{\text{BASE}}$~\cite{devlin2018bert}    & $0.810(\pm 0.017)$  & $0.780(\pm 0.027)$       & $0.764(\pm 0.035)$       & $0.733(\pm 0.043)$       & $0.700(\pm 0.044)$       \\
$\text{BERT}_{\text{BASE}}$ + mean-IW                & $0.810(\pm 0.017)$  & $0.799(\pm 0.020)$       & $0.778(\pm 0.028)$       & $0.750(\pm 0.022)$       & $0.714(\pm 0.033)$       \\
$\text{BERT}_{\text{BASE}}$ + max-IW                 & $0.810(\pm 0.017)$  & ${\bf 0.805(\pm 0.023)}$ & ${\bf 0.794(\pm 0.021)}$ & ${\bf 0.764(\pm 0.020)}$ & ${\bf 0.738(\pm 0.029)}$ \\
$\text{BERT}_{\text{BASE-AP}}$~\cite{devlin2018bert} & $0.801(\pm 0.012)$  & $0.765(\pm 0.028)$       & $0.741(\pm 0.030)$       & $0.719(\pm 0.042)$       & $0.694(\pm 0.048)$       \\
$\text{BERT}_{\text{BASE-AP}}$ + mean-IW             & $0.801(\pm 0.012)$  & $0.788(\pm 0.025)$       & $0.763(\pm 0.026)$       & $0.734(\pm 0.026)$       & $0.709(\pm 0.030)$       \\
$\text{BERT}_{\text{BASE-AP}}$ + max-IW              & $0.801(\pm 0.012)$  & ${\bf 0.795(\pm 0.025)}$ & ${\bf 0.777(\pm 0.020)}$ & ${\bf 0.748(\pm 0.028)}$ & ${\bf 0.722(\pm 0.021)}$ \\ \hline
GNN~\cite{cucurull2019context}                       & $0.346(\pm 0.030)$  & $0.329(\pm 0.033)$       & $0.319(\pm 0.036)$       & $0.301(\pm 0.045)$       & $0.287(\pm 0.050)$       \\
GNN + mean-IW                                        & $0.346(\pm 0.030)$  & $0.337(\pm 0.031)$       & $0.325(\pm 0.040)$       & $0.310(\pm 0.040)$       & $0.299(\pm 0.044)$       \\
GNN + max-IW                                         & $0.346(\pm 0.030)$  & ${\bf 0.342(\pm 0.035)}$ & ${\bf 0.338(\pm 0.036)}$ & ${\bf 0.317(\pm 0.038)}$ & ${\bf 0.306(\pm 0.038)}$ \\ \hline
HAP2S~\cite{yu2018hard}                              & $0.382(\pm 0.044)$  & $0.367(\pm 0.046)$       & $0.340(\pm 0.048)$       & $0.325(\pm 0.045)$       & $0.310(\pm 0.053)$       \\
HAP2S + mean-IW                                      & $0.382(\pm 0.044)$  & $0.371(\pm 0.040)$       & $0.343(\pm 0.048)$       & $0.331(\pm 0.046)$       & $0.312(\pm 0.051)$       \\
HAP2S + max-IW                                       & $0.382(\pm 0.044)$  & ${\bf 0.380(\pm 0.043)}$ & ${\bf 0.359(\pm 0.046)}$ & ${\bf 0.342(\pm 0.046)}$ & ${\bf 0.324(\pm 0.048)}$ \\ \hline
Cross Attention~\cite{saito2020exchangeable}         & $0.839(\pm 0.002)$  & $0.808(\pm 0.005)$       & $0.772(\pm 0.005)$       & $0.740(\pm 0.009)$       & $0.717(\pm 0.010)$       \\
Cross Attention + mean-IW                            & $0.839(\pm 0.002)$  & $0.810(\pm 0.009)$       & $0.788(\pm 0.007)$       & $0.752(\pm 0.006)$       & $0.720(\pm 0.008)$       \\
Cross Attention + max-IW                             & $0.839(\pm 0.002)$  & ${\bf 0.836(\pm 0.003)}$ & ${\bf 0.800(\pm 0.005)}$ & ${\bf 0.783(\pm 0.007)}$ & ${\bf 0.743(\pm 0.003)}$ \\ \hline
Cross Affinity~\cite{saito2020exchangeable}          & $0.845(\pm 0.000)$  & $0.822(\pm 0.001)$       & $0.791(\pm 0.005)$       & $0.762(\pm 0.008)$       & $0.741(\pm 0.004)$       \\
Cross Affinity + mean-IW                             & $0.845(\pm 0.000)$  & $0.831(\pm 0.008)$       & $0.792(\pm 0.002)$       & $0.766(\pm 0.004)$       & $0.749(\pm 0.002)$       \\
Cross Affinity+ max-IW                               & $0.845(\pm 0.000)$  & ${\bf 0.842(\pm 0.004)}$ & ${\bf 0.807(\pm 0.003)}$ & ${\bf 0.769(\pm 0.005)}$ & ${\bf 0.753(\pm 0.005)}$ \\
\bottomrule
\end{tabular}
}
\end{table*}

\subsection{Additional experimental results on regression problem under the target shift assumption}
Additionally, we present benchmark results for a regression problem with the target shift.
Although SHIFT15M is a dataset for set-to-set matching, it can also be used for simple regression problems by setting the input to image features and the output to attributes associated with the items.

The target variable is the number of likes that each instance possesses, and the input variables are the user ID and the prices of the items.
In this experiment, we evaluate the robustness of the model for different shift magnitudes, and the magnitudes of the shift are measured by the Wasserstein distance.
We use the simple linear regression as the ordinal ERM, and we compare this with two other covariate shift adaptation methods, IWERM and AIWERM.
\begin{definition}{(Adaptive importance weighted ERM~\cite{shimodaira2000improving})}
AIWERM uses $(p_{te}(\bm{x})/p_{tr}(\bm{x}))^\alpha$ for $\alpha\in[0,1]$ as the weighting function:
\begin{equation}
    \hat{h} = \argmin_{h\in\mathcal{H}}\frac{1}{n_{tr}}\sum^{n_{tr}}_{i=1}\Big(\frac{p_{te}(\bm{x}^{tr}_i)}{p_{tr}(\bm{x}^{tr}_i)}\Big)^\alpha \ell(h(\bm{x}^{tr}_i), y^{tr}_i). \label{eq:aiwerm}
\end{equation}
\end{definition}

\begin{table*}
\centering
\caption{Experimental results for the regression problem with distribution shift adaptation. Evaluation metrics are the MSE.}
\label{tab:benchmark_num_likes_regression}
\scalebox{0.85}{
\begin{tabular}{lcccccc}
\toprule
Models                 & $W=0$            & $W=10$                 & $W=20$                 & $W=30$                 & $W=40$            & $W=50$ \\
\midrule
ERM                    & $9.36(\pm 0.02)$ & $10.44(\pm 0.04)$      & $17.10(\pm 0.06)$      & $28.80(\pm 0.05)$      & $39.56(\pm 0.05)$ & $48.84(\pm 0.05)$ \\
IWERM (optimal)        & $9.36(\pm 0.02)$ & $25.67(\pm 0.12)$      & $32.58(\pm 0.12)$      & $26.83(\pm 0.11)$      & $20.19(\pm 0.10)$ & $14.52(\pm 0.10)$ \\
RIWERM ($\alpha=0.25$) & $9.36(\pm 0.02)$ & ${\bf 9.34(\pm 0.04)}$ & ${\bf 9.53(\pm 0.03)}$ & $11.37(\pm 0.04)$      & $14.89(\pm 0.09)$ & $17.00(\pm 0.14)$ \\
RIWERM ($\alpha=0.50$) & $9.36(\pm 0.02)$ & $9.73(\pm 0.04)$       & $9.57(\pm 0.03)$       & $9.69(\pm 0.04)$       & $12.37(\pm 0.10)$ & $14.68(\pm 0.15)$ \\
RIWERM ($\alpha=0.75$) & $9.36(\pm 0.02)$ & $11.59(\pm 0.05)$      & $11.35(\pm 0.05)$      & ${\bf 9.39(\pm 0.03)}$ & ${\bf 10.46(\pm 0.09)}$ & ${\bf 12.60(\pm 0.14)}$ \\
\bottomrule
\end{tabular}
}
\end{table*}

\begin{definition}{(Relative importance weighted ERM~\cite{yamada2013relative})}
RIWERM uses $p_{te}(\bm{x})/((1-\alpha)p_{tr}(\bm{x}) + \alpha p_{te}(\bm{x}))$ for $\alpha\in[0,1]$ as the weighting function:
\begin{equation}
    \hat{h} = \argmin_{h\in\mathcal{H}}\frac{1}{n_{tr}}\sum^{n_{tr}}_{i=1} \frac{p_{te}(\bm{x}^{tr}_i)}{m_\alpha(\bm{x}^{tr}_i)} \ell(h(\bm{x}^{tr}_i), y^{tr}_i), \label{eq:riwerm}
\end{equation}
where $m_\alpha(\bm{x}^{tr}_i) = (1-\alpha)p_{tr}(\bm{x}^{tr}_i) + \alpha p_{te}(\bm{x}^{tr}_i)$.
\end{definition}
Note that there are other variants of IWERM, as in \cite{kimura2022information}.

\begin{figure}
    \centering
    \includegraphics[width=0.95\linewidth]{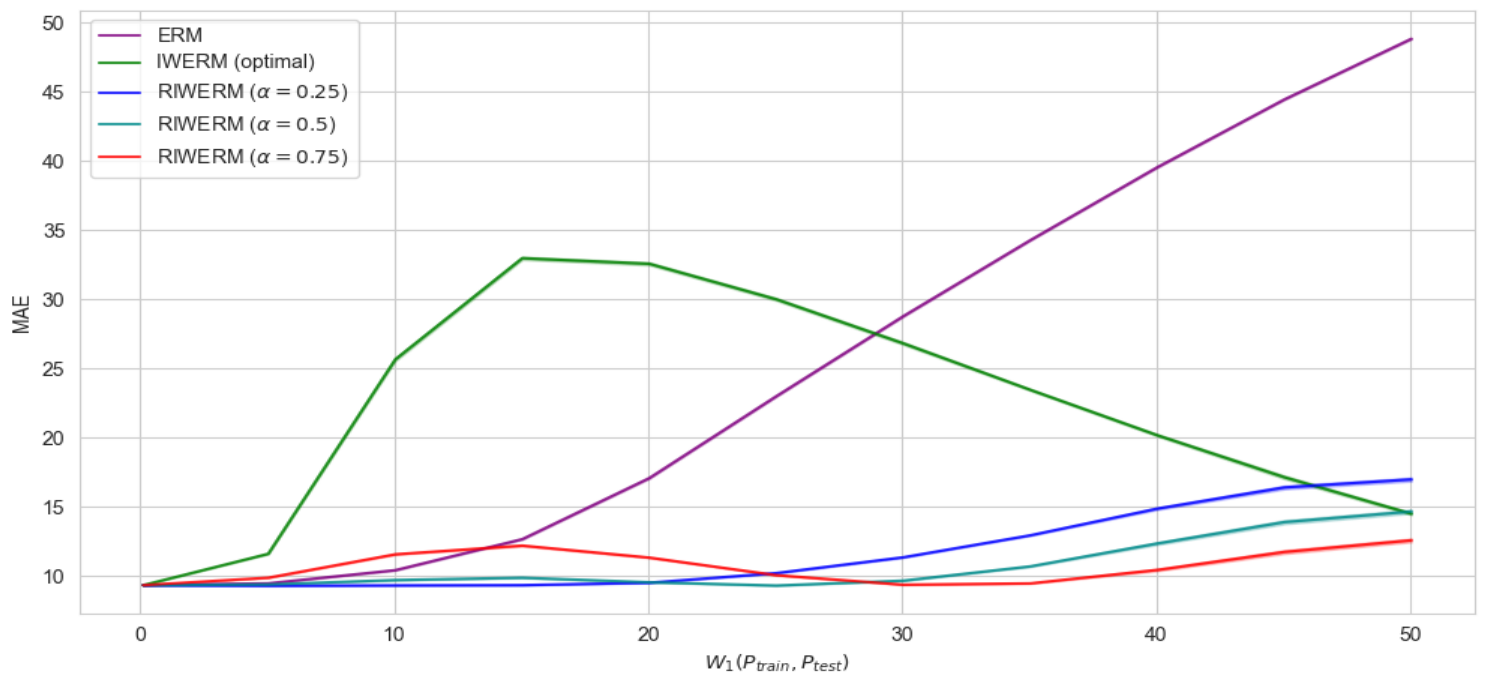}
    \caption{Experimental results of the regression problem for the number of likes}
    \label{fig:num_likes_regression}
\end{figure}

Figure~\ref{fig:num_likes_regression} and Table~\ref{tab:benchmark_num_likes_regression} show the experimental results of the regression task.
These results show that the performance of ERM decreases with increasing shift magnitude, while IWERM and its variants allow the robustness to target shifts.

\section{Related works and conclusion}
\label{sec:related_works_and_conclusion}
There are many studies for handling set data with neural networks.
DeepSets~\cite{zaheer2017deep} proposes a model that satisfies the key concepts of permutation invariant and permutation equivariant for approximating functions that deal with sets.
\begin{definition}[Permutation invariant]
    A set function $f$ is said to be permutation invariant if $f(\mathcal{V}, \mathcal{W}) = f(\pi_w\mathcal{W} , \pi_v\mathcal{V})$ for permutations $\pi_w$ and $\pi_v$.
\end{definition}
\begin{definition}[Permutation equivariance]
    A set function $f$ is said to be permutation equivariant if $f(\pi_w\mathcal{W}, \pi_v\mathcal{V}) = \pi_w f(\mathcal{W}, \mathcal{V})$ for permutations $\pi_w$ and $\pi_v$. Note that $f$ is permutation invariant for permutations within $\mathcal{V}$.
\end{definition}
\begin{theorem}[Sum-decomposable\cite{zaheer2017deep}]
    \label{def:sum_decomposable}
     A function $f$ on set $\mathcal{X}$ from countable particle space is invariant if and only if there exists a decomposition,
     \begin{align*}
         f(\mathcal{X}) = \rho\left(\sum_{\bm{x}\in\mathcal{X}}\phi(\bm{x})\right), \quad f = \rho \circ \sum \circ \phi, \label{eq:sum_decomposable}
     \end{align*}
     with appropriate functions $\phi$ and $\rho$.
\end{theorem}
It is pointed out that the necessary and sufficient of sum-decomposable is guaranteed only for countable sets~\cite{wagstaff2019limitations}.
\begin{theorem}[\cite{wagstaff2019limitations}]
A continuous function $f$ on finite sets $\mathcal{X}$, $|\mathcal{X}| < p$, is invariant if and only if it is sum-decomposable via $\mathbb{R}^p$.
\end{theorem}
That is, for an arbitrary continuous function $f$, the image space of $\phi$ has to have at least dimension $p$, which is both necessary and sufficient.
For more details, see Appendix~\ref{apd:details_sum_decomposable}.

SetTransformer~\cite{lee2019set} is an attention-based neural network module that allows us to handle sets as inputs.
SetVAE~\cite{kim2021setvae}, an extension of VAE to set data, has also been proposed.

Several distribution shift datasets exist for general classification and regression tasks where the input is a vector.
WILDS~\cite{koh2021wilds} is the collection of benchmark datasets~\cite{beery2020iwildcam,bandi2018detection,taylor2019rxrx1,hu2020open,david2020global,borkan2019nuanced,christie2018functional,yeh2020using,ni2019justifying,lu2021codexglue} under the distribution shift, including histopathological images, satellite images or sequence of source code tokens.
PACS~\cite{li2017deeper} and Office-Home~\cite{venkateswara2017deep} adopt the image style to differentiate distributions, and VLCS~\cite{fang2013unbiased} takes data collected independently from four sources as environments.
Also, DomainNet~\cite{zhao2019multi} extends PACS to a far larger scale.

There are also a number of studies that evaluate robustness to distribution shifts by introducing artificial distribution shifts, such as noise corruptions~\cite{geirhos2018generalisation,hendrycks2019benchmarking,mu2019mnist,rusak2020simple,xu2020robust}, spatial
transformations~\cite{engstrom2019exploring,fawzi2015manitest}, ImageNet~\cite{deng2009imagenet} variants (e.g.
ImageNet-A~\cite{hendrycks2021natural}, ImageNet-C~\cite{hendrycks2019benchmarking}, ImageNet-R~\cite{hendrycks2021many}), and adversarial examples~\cite{biggio2013evasion,szegedy2013intriguing,kurakin2018adversarial,goodfellow2014explaining,xiao2018generating,carlini2019evaluating}.
However, a recent study~\cite{taori2020measuring} has indicated that there is no correlation between the robustness of such artificial distribution shifts and the robustness of natural distribution shifts.

We believe that our SHIFT15M is a very useful dataset for evaluating the still underdeveloped task of set-to-set matching under natural distribution shifts.
See Appendix~\ref{apd:known_generalization_bounds} and \ref{apd:related_works} for more related literature including domain adaptation, out-of-domain generalization, concept drift adaptation, or other fashion datasets.
We also provide the datasheet for the SHIFT15M in Appendix~\ref{apd:datasheet}, which is based on \cite{gebru2021datasheets}.

\subsection{Future works}
\begin{itemize}
    \item More ablation studies: more model architectures and parameter influences need to be investigated. 
    \item Experiments on model calibration: model calibration and distribution shift are known to be closely related~\cite{wald2021calibration,ovadia2019can,kumar2022calibrated}. We expect that observing the metrics for evaluating model calibration in experiments on the SHIFT15M will provide meaningful insights.
    \item Additional API development: currently, our API is based on PyTorch~\cite{NEURIPS2019_9015}. In the future, we would like to expand the APIs for other machine learning libraries (e.g., TensorFlow~\cite{abadi2016tensorflow} or Keras~\cite{chollet2015keras}).
\end{itemize}


\clearpage
\appendix
\section{Sample items from the SHIFT15M}
\label{apd:sample_items}

Figure~\ref{fig:shift15m_additional_sample_images_2014},\ref{fig:shift15m_additional_sample_images_2015},\ref{fig:shift15m_additional_sample_images_2016},\ref{fig:shift15m_additional_sample_images_2017} show the additional sample items from the SHIFT15M dataset.

In addition, we provide the year-wise visualization for the SHIFT15M dataset with t-SNE in Figure~\ref{fig:shift15m_tsne_2015},\ref{fig:shift15m_tsne_2016},\ref{fig:shift15m_tsne_2017} like as Figure~\ref{fig:shift15m_tsne}.
t-SNE is a dimensionality reduction technique that is often used for visualizing high-dimensional data in a lower-dimensional space. In this case, the data being visualized is likely a collection of fashion-related features, such as color, texture, and style, that have been extracted from the dataset.
These figures are likely to show the distribution of fashion-related features across different years or time periods, which can help to reveal trends and patterns in the fashion industry over time.

\section{Details and additional figures for the numerical experiments}
Here, we describe the details for numerical experiments and introduce additional figures.
\subsection{Details for numerical experiments}
\paragraph{Training Settings}
We use a stochastic gradient descent method~\cite{robbins1951stochastic,kiefer1952stochastic} with a learning rate of $0.005$, a momentum of $0.5$, and a weight decay of $0.00004$.
We train both the CNN and set-matching model in an end-to-end manner. In each iteration, we randomly swap pairs of sets and items in each set, and randomly flip images horizontally, to learn all the methods stably.

\paragraph{Preparing Set Pairs}
To construct the correct pair of sets to be matched, we randomly halve the given outfit $\mathcal{O}$ into two non-empty proper subsets $\mathcal{X}$ and $\mathcal{Y}$ as follows: $\mathcal{O} \to \{\mathcal{X}, \mathcal{Y}\}$, where $\mathcal{X}\cap\mathcal{Y} = \emptyset$.
Here, we extend this setting to include more general situations.
We select $Q$ outfits ${\mathcal{O}^{(1)},\dots, \mathcal{O}^{(Q)}}$ randomly and split the respective outfits in half $\mathcal{O}^{(q)} \to \{\mathcal{X}^{(q)}, \mathcal{Y}^{(q)}\}$, where $q \in \{1,\dots,Q\}$.
We regard the two sets $\{\mathcal{X}^{(1)},\dots, \mathcal{X}^{(Q)}\}$ and $\{\mathcal{Y}^{(1)},\dots, \mathcal{Y}^{(Q)}\}$as the correct pair, which consists of $Q$ fashion styles. In the training phase, we set $Q = 4$.
Figure~\ref{fig:set_matching} shows the overview of set-to-set matching problem.

\begin{figure}
    \centering
    \includegraphics[width=0.9\linewidth]{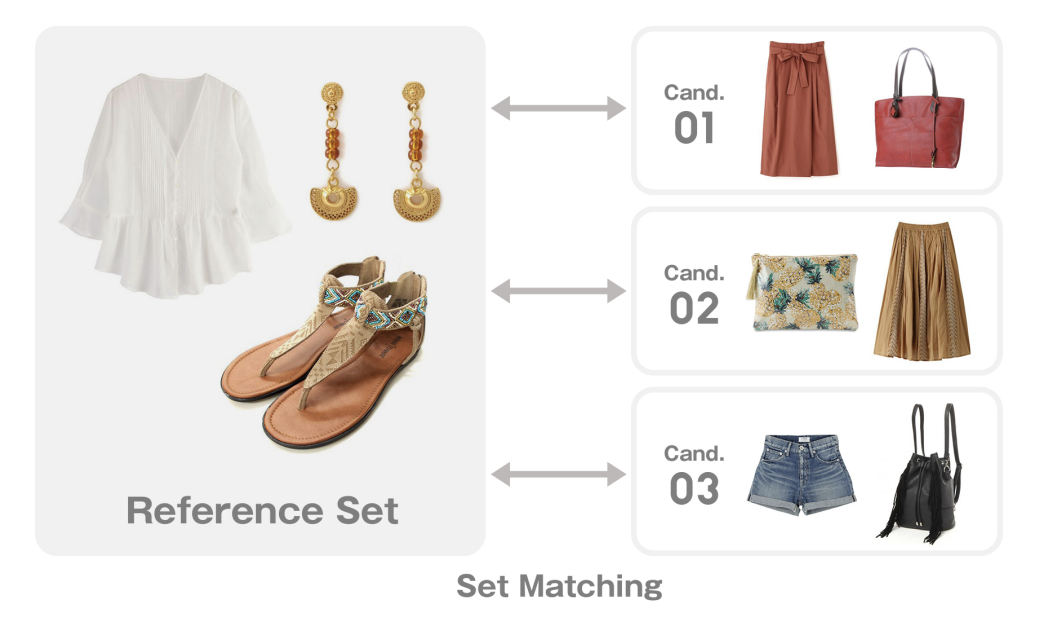}
    \caption{This figure is a citation of Figure 1 from Y.Saito et al.~\cite{saito2020exchangeable}.
    Set-to-set matching aims to answer a fundamental question: which candidate set is more compatible with the reference set than others? In this process, we match the reference set with each candidate set and select the best pair based on some criteria.}
    \label{fig:set_matching}
\end{figure}

\subsection{Additional figures for the numerical experiments}
\label{apd:additional_figures}
Figure~\ref{fig:fill_in_the_n_blank_cross_attention},\ref{fig:fill_in_the_n_blank_set_transformer},\ref{fig:fill_in_the_n_blank_bert},\ref{fig:fill_in_the_n_blank_gnn},\ref{fig:fill_in_the_n_blank_hap2s} show the additional figures for the Fill-In-The-N Blank experiments.
These figures are based on Table~\ref{tab:benchmark_set2set_matching_4_several_models} and \ref{tab:benchmark_set2set_matching_8_several_models}.

Figure~\ref{fig:pca_2d_2015},\ref{fig:pca_2d_2016} and \ref{fig:pca_2d_2017} shows the PCA~\cite{doi:10.1080/14786440109462720} dimensionality reduction for items in each year.
The color of each point corresponds to the item category.

\section{Details for the sum-decomposable set function}
\label{apd:details_sum_decomposable}
From Definition~\ref{def:sum_decomposable}, we say that $(\rho, \phi)$ is a sum-decomposition of $f$.
Given sum-decomposition $(\rho, \phi)$, we write $\Phi(\mathcal{X})=\sum_{\bm{x}\in\mathcal{X}} \phi(\bm{x})$ and $f = \rho \circ \Phi$.
We may also refer to the function $\rho\circ\Phi$ as a sum-decomposition.
\begin{definition}
    Let $(\rho, \phi)$ be a sum-decomposition.
    Write $Z$ for the domain of $\rho$ and the codomain of $\phi$.
    We refer to $Z$ as the latent space of the sum-decomposition $(\rho, \phi)$.
\end{definition}
\begin{definition}
    Given a space $Z$, we say that $f$ is sum-decomposable via $Z$ if $f$ has a sum-decomposition whose latent space is $Z$.
\end{definition}
\begin{definition}
    We say that $f$ is continuously sum-decomposable when there exists a sum-decomposition $(\rho, \phi)$ of $f$ such that both $\rho$ and $\phi$ are continuous.
    $(\rho, \phi)$ is then a continuous sum-decomposition of $f$.
\end{definition}
We give a brief reproduction of the statements and proof of two key theorems from \cite{zaheer2017deep}.
\begin{theorem}
\label{thm:guarantee_sum_decomposable_countable}
    Let $f: 2^{\mathcal{U}} \to \mathbb{R}$ where $\mathcal{U}$ is countable.
    Then $f$ is sum-decomposable via $\mathbb{R}$.
\end{theorem}
\begin{proof}
Since $\mathcal{U}$ is countable, each $\bm{x}\in\mathcal{U}$ can be mapped to a unique element in $\mathbb{N}$ by a bijective function $c:\mathcal{U}\to\mathbb{N}$.
If we can choose $\phi$ so that $\Phi$ is invertible, then we can set $\rho = f\circ\Phi^{-1}$, giving
\begin{align}
    f = \rho \circ \Phi,
\end{align}
so $f$ is sum-decomposable via $\mathbb{R}$.

Considering the mapping $\phi(\bm{x}) = 4^{-c(\bm{x})}$, each set $\mathcal{X}\subset\mathcal{U}$ corresponds to a unique real number $r\coloneqq \Phi(\mathcal{X})$.
The number $r$ can be decoded to the set $\mathcal{X}$ and the element $c^{-1}(n)\in\mathcal{U}$ belongs to $\mathcal{X}$ if and only if the $n$-th digit of $r$ is $1$.
This decoding procedure shows that $\Phi$ is invertible, and the conclusion follows.
\end{proof}
\begin{theorem}
    Let $M\in\mathbb{N}$, and let $f:[0,1]^M\to\mathbb{R}$ be a continuous permutation-invariant function.
    Then, $f$ is continuously sum-decomposable via $\mathbb{R}^{M+1}$.
\end{theorem}
\begin{theorem}
    Deep Sets can represent any continuous permutation-invariant function function of $M$ elements if the dimesion of the latent space of the model is at least $M+1$.
\end{theorem}

Recent work~\cite{wagstaff2022universal} argue that the guarantee of sum-decomposability via $\mathbb{R}$ given by Theorem~\ref{thm:guarantee_sum_decomposable_countable} cannot hold in practice, and prove that the guarantee of sum-decomposability via $\mathbb{R}^{M+1}$ is essentially the bestpossible.
\begin{theorem}
    Let $M,N\in\mathbb{N}$, with $M>N$.
    Then there exist continuous permutation-invariant functions $f:\mathbb{R}^M\to\mathbb{R}$ which are not continuously sum-decomposable via $\mathbb{R}^N$.
\end{theorem}
\begin{theorem}
    Let $M\in\mathbb{N}$, and let $f:\mathbb{R}^M\to\mathbb{R}$ be a continuous permutation-invariant function.
    Then, $f$ is continuously sum-decomposable via $\mathbb{R}^M$.
\end{theorem}
\begin{theorem}
    Denote the set of subsets $[0,1]$ containing at most $M$ elements by $[0,1]^{\leq M}$.
    Let $f:[0,1]^{\leq M}\to\mathbb{R}$ be continuous and permutation-invariant.
    Then, $f$ is continuously sum-decomposable via $\mathbb{R}^M$.
\end{theorem}
Let $M$ be a positive integer, $U\subset\mathbb{R}^M$ be compact, and $f:U\to\mathbb{R}$.
\begin{definition}
    Let $\epsilon > 0$.
    $(\phi, \rho)$ is a within-$\epsilon$ sum-decomposition of $f$ if $|f(\mathcal{U}) - \rho(\Phi(\mathcal{U}))| < \epsilon$ for every $\mathcal{U} \in U$.
\end{definition}
\begin{definition}
    A sequence $(\phi, \rho)_k \coloneqq \{(\phi_k, \rho_k) : k\in\mathbb{N}\}$ is an approximate sum-decomposition of $f$ if, for any $\epsilon > 0$, there is some $K\in\mathbb{N}$ such that $(\phi_K, \rho_K)$ is a within-$\epsilon$ sum-decomposition of $f$.
    We also require that $(\phi_k, \rho_k)$ is a sequence of ever-closer approximations to $f$.
    The existence of an approximate sum-decomposition of $f$ guarantees that $f$ can be approximated arbitrarily closely by sum-decomposition.
\end{definition}
\begin{theorem}
    Let $M, N\in\mathbb{N}$ with $M > N$, and recall that $I_M = [-1,1]^M\subset\mathbb{R}^M$.
    Then there exists a continuous permutation-invariant function $f:I_M\to\mathbb{R}$ which has no continuous approximate sum-decomposition via $\mathbb{R}^N$.
\end{theorem}

\section{Generalization bounds under the distribution shift}
\label{apd:known_generalization_bounds}
The SHIFT15M dataset is a valuable resource for evaluating the performance of machine learning models in settings where the underlying distribution of the data may shift over time. One important aspect of this dataset is that it allows researchers to calculate the distance between the distributions in the train and test splits. This distance measure can provide valuable insight into how much the distribution has shifted between the two sets of data, which in turn can help researchers to better understand the behavior of their models under distributional shifts.

Furthermore, researchers have proposed several generalization bounds that are specific to distribution shifts and that depend on the distance between distributions. These bounds provide a way to quantify the relationship between the performance of a model and the extent of the distributional shift, which is a crucial factor in evaluating the robustness of a model. By considering these generalization bounds when analyzing experimental results on the SHIFT15M dataset, researchers can gain a more precise understanding of how their models perform under distribution shifts and how to improve them. Overall, the SHIFT15M dataset and the associated generalization bounds represent important tools for evaluating and improving the robustness of machine learning models.
Here, we introduce several known generalization bounds under the distribution shift~\cite{redko2019advances}.
\begin{definition}[Total variation distance]
    Denote by $\mathcal{B}$ be the set of measurable subsets under two probability distributions $p_1$ and $p_2$.
    Then, the total variation distance between $p_1$ and $p_2$ is defined as
    \begin{align}
        d_{TV}(p_1, p_2) = 2\sup_{B\in\mathcal{B}}\left|p_1(B) - p_2(B)\right|.
    \end{align}
\end{definition}
Let
\begin{align}
    \mathcal{R}^{\ell}_{\mathfrak{D}}(h) = \mathbb{E}_{(\bm{x},y)\sim\mathfrak{D}}\left[\ell(\bm{x}, y)\right].
\end{align}
\begin{theorem}[\cite{ben2006analysis}]
    Given two domains $\mathfrak{S}$ and $\mathfrak{T}$, and a hypothesis class $\mathcal{H}$ over $\mathcal{X}\times\mathcal{Y}$, the following holds for any $h\in\mathcal{H}$.
    \small
    \begin{align}
        &\mathcal{R}^{\ell_{01}}_{\mathfrak{T}}(h) \leq \mathcal{R}^{\ell_{01}}_{\mathfrak{S}}(h) + d_{TV}(\mathfrak{S}_{\mathcal{X}}, \mathfrak{T}_{\mathcal{X}}) \nonumber \\
        &+ \min\left\{\mathbb{E}_{\bm{x}\sim\mathfrak{S}_{\mathcal{X}}}\left[J(\bm{x}; \mathfrak{S}_{\mathcal{X}}, \mathfrak{T}_{\mathcal{X}})\right], \mathbb{E}_{\bm{x}\sim\mathfrak{T}_{\mathcal{X}}}\left[J(\bm{x}; \mathfrak{S}_{\mathcal{X}}, \mathfrak{T}_{\mathcal{X}})\right]\right\}, \nonumber
    \end{align}
    \normalsize
    where $J(\bm{x}; \mathfrak{S}_{\mathcal{X}}, \mathfrak{T}_{\mathcal{X}}) = |f_{\mathfrak{S}}(\bm{x}) - f_{\mathfrak{T}}(\bm{x})|$, and $f_{\mathfrak{S}}(\bm{x})$ and $f_{\mathfrak{T}}(\bm{x})$ are source and target true labeling functions associated with $\mathfrak{S}$ and $\mathfrak{T}$, respectively.
\end{theorem}
\begin{definition}
    Given marginal distributions of two domains $\mathfrak{S}_\mathcal{X}$ and $\mathfrak{T}_\mathcal{X}$ over the input space $\mathcal{X}$, let $\mathcal{H}$ be a hypothesis class, and denote $\mathcal{H}\Delta\mathcal{H}$ the symmetric difference hypothesis space defined as
    \begin{align}
        h\in\mathcal{H}\Delta\mathcal{H} \Longleftrightarrow h(\bm{x}) = g(\bm{x})\oplus g'(\bm{x})
    \end{align}
    for some $(g, g') \in \mathcal{H}^2$, where $\oplus$ stands for XOR operation.
    Let $I(h)$ denote the set for which $h\in\mathcal{H}\Delta\mathcal{H}$ is the characteristic function, that is $\bm{x}\in I(h) \Leftrightarrow g(\bm{x}) = 1$.
    The $\mathcal{H}\Delta\mathcal{H}$-divergence between $\mathfrak{S}_\mathcal{X}$ and $\mathfrak{T}_\mathcal{X}$ is defined as:
    \small
    \begin{align}
        d_{\mathcal{H}\Delta\mathcal{H}}(\mathfrak{S}_\mathcal{X}, \mathfrak{T}_\mathcal{X}) = 2\sup_{h\in\mathcal{H}\Delta\mathcal{H}}\left|\Pr_{\mathfrak{S}_\mathcal{S}}(I(h)) - \Pr_{\mathfrak{T}_\mathcal{X}}(I(h))\right|.
    \end{align}
    \normalsize
\end{definition}
\begin{lemma}
Let $\mathcal{H}$ be a hypothesis space of VC dimension $VC(\mathcal{H})$.
If $S_u$, $T_u$ are unlabeled samples of size $m$ each, drawn independently from $\mathfrak{S}_\mathcal{X}$ and $\mathfrak{T}_\mathcal{X}$ respectively, then for any $\delta \in (0,1)$ with probability at least $1-\delta$ over the random choice of the samples we have
\footnotesize
\begin{align}
    d_{\mathcal{H}\Delta\mathcal{H}}(\mathfrak{S}_\mathcal{X}, \mathfrak{T}_\mathcal{X}) \leq \hat{d}_{\mathcal{H}\Delta\mathcal{H}}(\mathfrak{S}_u,\mathfrak{T}_u) + 4\sqrt{\frac{2VC(\mathcal{H})\ln 2m + \ln\frac{2}{\delta}}{m}},  \nonumber
\end{align}
\normalsize
where $\hat{d}_{\mathcal{H}\Delta\mathcal{H}}(\mathfrak{S}_u,\mathfrak{T}_u)$ is the empirical $\mathcal{H}\Delta\mathcal{H}$-divergence estimated on $\mathfrak{S}_u$ and $\mathfrak{T}_u$.
\end{lemma}
\begin{lemma}[\cite{ben2010theory}]
    Let $\mathcal{H}$ be a hypothesis space.
    Then, for two unlabeled samples $S_u$, $T_u$ of size $m$ we have
    \begin{align}
        & \hat{d}_{\mathcal{H}\Delta\mathcal{H}}(S_u, T_u) \nonumber \\
        &= 2\left\{1 - \min_{h\in\mathcal{H}\Delta\mathcal{H}}\left(\frac{1}{m}\sum_{\bm{x}\in H_0} 1_{\bm{x} \in S_u} + \frac{1}{m}\sum_{\bm{x}\in H_1}1_{\bm{x} \in T_u}\right)\right\}, \nonumber
    \end{align}
    where $H_0 = \{\bm{x} : h(\bm{x}) = 0\}$ and $H_1 = \{\bm{x} : h(\bm{x}) = 1\}$.
\end{lemma}
\begin{lemma}[\cite{ben2010theory}]
    Let $\mathfrak{S}$ and $\mathfrak{T}$ be two domains on $\mathcal{X}\times\mathcal{Y}$.
    For any pair of hypothesis $(h, h')\in\mathcal{H}\Delta\mathcal{H}^2$, we have
    \begin{align}
        \left|\mathcal{R}^{\ell_{01}}_\mathfrak{T}(h, h') - \mathcal{R}^{\ell_{01}}_\mathfrak{S}(h, h')\right| \leq \frac{1}{2}d_{\mathcal{H}\Delta\mathcal{H}}(\mathfrak{S}_\mathcal{X}, \mathfrak{T}_\mathcal{X}),
    \end{align}
    where
    \begin{align}
        \mathcal{R}^{\ell}_\mathfrak{D}(h, h') = \mathbb{E}_{\bm{x}\sim\mathfrak{D}_\mathcal{X}}\left[\ell(h(\bm{x}), h'(\bm{x}))\right].
    \end{align}
\end{lemma}
\begin{theorem}[\cite{ben2010theory}]
    Let $\mathcal{H}$ be a hypothesis space of VC dimension $VC(\mathcal{H})$.
    If $S_u$ and $T_u$ are unlabeled samples of size $m'$ each, drawn independently from $\mathfrak{S}_\mathcal{X}$ and $\mathfrak{T}_\mathcal{X}$, respectively, then for any $\delta \in (0,1)$ with probability at least $1-\delta$ over the random choice of the samples, we have that for all $h\in\mathcal{H}$
    \begin{align}
        \mathcal{R}^{\ell_{01}}_\mathfrak{T}(h) \leq \mathcal{R}^{\ell_{01}}_\mathfrak{S}(h) &+ \frac{1}{2}\hat{d}_{\mathcal{H}\Delta\mathcal{H}}(S_u, T_u) \nonumber \\
        &+ 4\sqrt{\frac{2VC(\mathcal{H})\ln 2m' + \ln\frac{2}{\delta}}{m'}} + \lambda, \nonumber
    \end{align}
    where $\lambda$ is the combined error of the ideal hypothesis $h^*$.
\end{theorem}
The trade-off between source risk, divergence and the capability to adapt for distribution shift is a crucial phenomenon that plays a significant role in experiments related to distribution shift adaptation using SHIFT15M. This trade-off refers to the balance that needs to be maintained between the risk of the source, the extent of divergence between the source and target domains, and the ability of a model to adapt to distribution shift.
Therefore, it is essential to evaluate the performance of models while keeping this trade-off in mind. By doing so, researchers can ensure that the models they develop are not only accurate but also robust and adaptable to changes in the underlying distribution of the data. Ultimately, this can lead to better real-world performance of machine learning models and improve their utility in various applications.

Since our numerical experiments use the Wasserstein distance $W^p_p$ as the distance between distributions, we also introduce generalization bounds based on this distance:
\begin{align}
    W^p_p(\mathfrak{S}_\mathcal{X}, \mathfrak{T}_\mathcal{X}) \coloneqq \inf_{\gamma\in\Pi(\mathfrak{S}_\mathcal{X}, \mathfrak{T}_\mathcal{X})}\int_{\mathcal{X}\times\mathcal{X}}c(\bm{x}, \bm{x}')^p d\gamma(\bm{x},\bm{x}'), \nonumber
\end{align}
where $c:\mathcal{X}\times\mathcal{X}\to\mathbb{R}_+$ is a cost function for transporting one unit of mass $\bm{x}$ to $\bm{x}'$, and $p\in [1,+\infty]$.
\begin{lemma}[\cite{redko2017theoretical}]
    \label{lem:redko2017}
    Let $\mathfrak{S}_\mathcal{X},\mathfrak{T}_\mathcal{X} \in \mathcal{P}(\mathcal{X})$ be two probability measures on $\mathbb{R}^d$.
    Assume that the cost function $c(\bm{x},\bm{x}') = \|\phi(\bm{x}) - \phi(\bm{x}')\|_{\mathfrak{H}_{k_\ell}}$, where $\mathfrak{H}$ is an RKHS equipped with kernel $k_\ell : \mathcal{X}\times\mathcal{X}\to\mathbb{R}$ induced by $\phi:\mathcal{X}\to\mathfrak{H}_{k_\ell}$ and $k_\ell(\bm{x}, \bm{x}')=\langle\phi(\bm{x}),\phi(\bm{x}')\rangle_{\mathfrak{H}_{k_\ell}}$.
    Assume further that the loss function $\ell_{h,f}:\bm{x}\mapsto\ell(h(\bm{x}), f(\bm{x}))$ is convex, symmetric and bounded and obeys the triangular equality and has the parametric form $|h(\bm{x}) - f(\bm{x})|^q$ for some $q > 0$.
    Assume also that kernel $k_\ell$ in the RKHS $\mathfrak{H}_{k_\ell}$ is square-root integrable with respect to both $\mathfrak{S}_\mathcal{X}$, $\mathfrak{T}_\mathcal{X}$ for all $\mathfrak{S}_\mathcal{X},\mathfrak{T}_\mathcal{X}\in\mathcal{P}(\mathcal{X})$ where $\mathcal{X}$ is separable and $0 \leq k_\ell(\bm{x}, \bm{x}') \leq K, \forall \bm{x},\bm{x}'\in\mathcal{X}$ if $\|\ell\|_{\mathfrak{H}_{k_\ell}} \leq 1$, then the following holds:
    \begin{align}
        \mathcal{R}^{\ell_q}_\mathfrak{T}(h, h') \leq \mathcal{R}^{\ell_q}_\mathfrak{S}(h, h') + W_1(\mathfrak{S}_\mathcal{X}, \mathfrak{T}_\mathcal{X}),
    \end{align}
    for all $(h,h')\in\mathfrak{H}^2_{k_\ell}$.
\end{lemma}
This lemma allows us to relate the source and target errors by using Wasserstein distance.
Also, we have
\begin{align*}
    \|\phi(\bm{x}) - \phi(\bm{x}')\|_{\mathfrak{H}} &= \sqrt{\langle\phi(\bm{x}) - \phi(\bm{x}'), \phi(\bm{x}) - \phi(\bm{x}')\rangle_{\mathfrak{H}}} \\
    &= \sqrt{k(\bm{x},\bm{x}) - 2k(\bm{x}, \bm{x}') + k(\bm{x}', \bm{x}')}.
\end{align*}
\begin{theorem}[\cite{bolley2007quantitative}]
    Let $\mu$ be a probability measure in $\mathbb{R}^d$ so that for some $\alpha > 0$, $\in_{\mathbb{R}^d}\exp\{\alpha\|\bm{x}\|^2\}d\mu < \infty$, and $\hat{\mu} = \frac{1}{N}\sum^N_{i=1}\delta_{\bm{x}_i}$ be its associated empirical measure defined on a sample of independent variables $\{\bm{x}_i\}^N_{i=1}$ drawn from $\mu$.
    Then for any $d' > d$ and $\zeta' < \sqrt{2}$ there exists some constant $N_0$ depending on $d'$ and some square exponential moment of $\mu$ such that, for any $\epsilon>0$ and $N\geq N_0\max(\epsilon^{-(d'+2)}, 1)$,
    \begin{align}
        \Pr\left\{W_1(\mu, \mu') > \epsilon\right\} \leq \exp\left\{-\frac{\zeta'}{2}N\epsilon^2\right\}
    \end{align}
    where $d'$ and $\zeta'$ can be calculated explicitly.
\end{theorem}
\begin{theorem}
    Under the assumption of Lemma~\ref{lem:redko2017}, let $S_u$ and $T_u$ be two samples of size $N_S$ and $N_T$ drawn i.i.d. from $\mathfrak{S}_\mathcal{X}$ and $\mathfrak{T}_\mathcal{X}$, respectively.
    Let $\hat{\mathfrak{S}}_\mathcal{X}=\frac{1}{N_S}\sum^{N_S}_{i=1}\delta_{\bm{x}^S_i}$ and $\hat{\mathfrak{T}}_\mathcal{X}=\frac{1}{N_T}\sum^{N_T}_{i=1}\delta_{\bm{x}^T_i}$ be the associated empirical measures.
    Then for any $d' > d$ and $\zeta' < \sqrt{2}$, there exists some constant $N_0$ depending on $d'$ such that for any $\delta > 0$ and $\min(N_S, N_T) \geq N_0\max(\delta^{-(d' + 2)}, 1)$ with probability at least $1-\delta$ for all $h$, we have
    \begin{align}
        \mathcal{R}^{\ell_q}_\mathfrak{T}(h) \leq \mathcal{R}^{\ell_q}_\mathfrak{S}(h) &+ W_1(\hat{\mathfrak{S}}_\mathcal{X}, \hat{\mathfrak{T}}_\mathcal{X}) \nonumber\\
        &+ \sqrt{\frac{2\ln\frac{1}{\delta}}{\zeta'}}\left(\sqrt{\frac{1}{N_S}} + \sqrt{\frac{1}{N_T}}\right) + \lambda, \nonumber
    \end{align}
    where $\lambda$ is the combined error of the ideal hypothesis $h^*$ that minimizes the combined error of $\mathcal{R}^{\ell_q}_\mathfrak{S}(h) + \mathcal{R}^{\ell_q}_\mathfrak{T}(h)$.
\end{theorem}

\section{Density ratio estimation for the importance weighted set-to-set matching}
\label{apd:density_ratio}
In the experiments, we proposed the importance-weighted set-to-set matching as the baseline method.
Recall that, under the covariate shift assumption, we have
\small
\begin{align*}
    \mathbb{E}_{tr}\left[\frac{p_{te}(\bm{x})}{p_{tr}(\bm{x})}\ell(h(\bm{x}), y)\right] &= \int\frac{p_{te}(\bm{x})}{p_{tr}(\bm{x})}\ell(h(\bm{x}), y) p_{tr}(\bm{x},y)d\bm{x}dy \\
    &= \int \ell(h(\bm{x}), y) p_{te}(\bm{x},y)d\bm{x}dy \\
    &= \mathbb{E}_{te}\left[\ell(h(\bm{x}), y)\right].
\end{align*}
\normalsize
Density ratio estimation is a technique used in machine learning to quantify the difference between the distributions of two datasets. The goal of density ratio estimation is to estimate the density ratio $r(\bm{x})$, which is defined as the ratio of the probability density functions of the test distribution $p_{te}(\bm{x})$ and the training distribution $p_{tr}(\bm{x})$ as $r(\bm{x}) = p_{te}(\bm{x}) / p_{tr}(\bm{x})$.

The density ratio is a powerful tool because it allows us to compare the distributions of the two datasets and measure the extent of the distribution shift. In particular, if we can estimate the density ratio accurately, we can obtain a consistent estimator under the distribution shift. This means that we can use this estimator to accurately predict the performance of our model on the test set, even if the distribution of the test set is different from that of the training set.

To estimate the density ratio, we used a probabilistic classifier. However, there are several other strategies that exist for estimating density ratios.
Each of these methods has its own strengths and weaknesses, and the choice of method will depend on the specific application and the characteristics of the data.

The main idea of moment matching is, to match the moments of $\hat{p}_{te}(\bm{x}) = \hat{r}(\bm{x})p_{tr}(\bm{x})$ and $p_{te}(\bm{x})$.
For example, matching the mean is
\begin{align}
    \int\bm{x}\hat{r}(\bm{x})p_{tr}(\bm{x})d\bm{x} = \int\bm{x}p_{te}(\bm{x})d\bm{x}.
\end{align}
However, matching a finite number of moments does not necessarily yield the true density ratio even asymptotically.
Kernel mean matching~\cite{huang2006correcting,gretton2009covariate} allows that All moments are efficiently matched in Gaussian RKHS $\mathfrak{H}$:
\begin{align}
    \min_{\hat{r}\in\mathfrak{H}}\left\|\int K(\bm{x}, \cdot)\hat{r}(\bm{x})p_{tr}(\bm{x})d\bm{x} - \int K(\bm{x}, \cdot)p_{tr}(\bm{x})d\bm{x}\right\|^2_{\mathfrak{H}}. \nonumber
\end{align}

KLIEP~\cite{nguyen2010estimating,sugiyama2007direct} minimize KL-divergence from $p_{te}(\bm{x})$ to $\hat{p}_{te}(\bm{x}) = \hat{r}(\bm{x})p_{tr}(\bm{x})$ as
\begin{align}
    \min_{\hat{r}}D_{KL}[p_{te}(\bm{x})\|\hat{p}_{te}] = \min_{\hat{r}}\int p_{te}(\bm{x})\frac{p_{te}(\bm{x})}{\hat{r}(\bm{x})p_{tr}(\bm{x})}d\bm{x}. \nonumber
\end{align}

Least-Squares Importance Fitting (LSIF)~\cite{kanamori2009least} minimize squared loss:
\begin{align*}
    \min_{\hat{r}}\int \left(\hat{r}(\bm{x}) - r(\bm{x})\right)^2 p_{tr}(\bm{x})d\bm{x}.
\end{align*}

\section{Other related literature}
\label{apd:related_works}
Here we present the rest of the relevant literature.

\subsection{Concept drift}
In addition to covariate shift and target shift, the following concept drift~\cite{tsymbal2004problem,gama2014survey,lu2018learning} is also well known.
\begin{definition}[Concept drift]
    We consider that the two distributions $p_{tr}(\bm{x}, y)$ and $p_{te}(\bm{x}, y)$ satisfy the concept drift assumption if the following conditions hold:
    \begin{align*}
        p_{tr}(\bm{x} | y) \neq p_{te}(\bm{x} | y), \\
        p_{tr}(y | \bm{x}) \neq p_{tr}(y | \bm{x}).
    \end{align*}
\end{definition}
Unlike covariate shift and target shift, concept drift assumes that the conditional probabilities between the two distributions are different.
This means that a model that is well-specified in the training distribution will be miss-specified in the test distribution, making it the most difficult problem setup.
The strategies for addressing concept drift can be broadly categorized as follows
\begin{itemize}
    \item concept drift detection;
    \item concept drift understanding;
    \item concept drift adaptation.
\end{itemize}

Concept drift detection refers to the strategies that characterize and quantify concept drift via identifying change points or change time intervals~\cite{basseville1993detection}.
Many algorithms focus on tracking changes in the online error rate of base classifiers, such as Drift Detection Method (DDM)~\cite{gama2004learning}, LLDD~\cite{gama2006learning}, EDDM~\cite{baena2006early}, HDDM~\cite{frias2014online} or FW-DDM~\cite{liu2017fuzzy}.
Other strategies based on data distribution~\cite{lu2014concept,song2007statistical,lu2016concept} or multiple hypothesis test~\cite{alippi2008just,wang2015concept,zhang2017three}.

Concept drift understanding refers to retrieving concept drift information about "When" (the time at which the concept drift occurs and how long the drift lasts)~\cite{shao2014prototype,lu2014concept}, "How" (the severity/degree of concept drift)~\cite{nishida2007detecting}, and "Where" (the drift regions of concept drift)~\cite{liu2017regional}.

The main approaches for concept drift adaptation are training new models for global drift~\cite{bach2008paired,manly2000cumulative}, model ensemble for recurring drift~\cite{gomes2017adaptive,kolter2007dynamic,elwell2011incremental}, or adjusting existing models for regional drift~\cite{hulten2001mining,gama2003accurate,yang2012incrementally}.

\subsection{Domain adaptation}
Domain adaptation~\cite{daume2006domain,daum2007frustratingly,ganin2015unsupervised,saunders2022domain} is often used in a similar context to distribution shift adaptation.
It is often referred to as visual domain adaptation~\cite{fernando2013unsupervised,jamal2020rethinking,wang2018deep,patel2015visual}, especially in the field of computer vision.
This concept is often referred to indistinguishably from covariate shift.
Depending on the availability of the source and target domain data, the domain adaptation problem can be defined in many different ways.
\begin{itemize}
    \item {\bf supervised domain adaptation}: In supervised domain adaptation~\cite{motiian2017unified,tzeng2015simultaneous}, labeled data is available in both the source and target domains. The model is trained on the labeled data from the source domain and then adapted to the target domain using the labeled data from the target domain.
    \item {\bf semi-supervised domain adaptation}: In semi-supervised domain adaptation~\cite{kumar2010co,saito2019semi,donahue2013semi,yao2015semi}, a small amount of labeled data is available in the target domain in addition to the labeled data in the source domain. The model is trained on the labeled data from both domains and then adapted to the target domain using both the labeled and unlabeled data from the target domain.
    \item {\bf unsupervised domain adaptation}: In unsupervised domain adaptation~\cite{ganin2015unsupervised,kang2019contrastive,li2020model,sener2016learning,long2016unsupervised}, labeled data is only available in the source domain. The model is trained on the labeled data from the source domain and then adapted to the target domain using only the unlabeled data from the target domain.
    \item {\bf multi-source domain adaptation}: Unlike traditional domain adaptation, which involves adapting from a single source domain, multi-source domain adaptation~\cite{peng2019moment,yang2020curriculum,zhao2020multi} deals with situations where there are multiple source domains with different but related feature distributions.
\end{itemize}

\subsection{Out-of-distribution generalization}
Another similar concept is out-of-distribution generalization~\cite{arjovsky2020out,hendrycks2021many,li2022out,shen2021towards,ye2021towards}.
The main difference with distribution shift problem is that in the out-of-distribution generalization problem setting we do not have any access to the test data during training.
For example, under the covariate shift hypothesis, we often assume that we are allowed to access unlabeled test data.
Major approaches for out-of-distribution generalization are disentangled representation learning~\cite{bengio2013representation,higgins2017beta,kim2018disentangling}, causal representation learning~\cite{yang2021causalvae,scholkopf2021toward}, domain generalization~\cite{zhou2022domain,li2018domain,zhou2021domain,li2017deeper}, invariant learning~\cite{pfister2019invariant,rothenhausler2018anchor,heinze2018invariant}, stable learning~\cite{shen2020stable,zhang2021deep} and distributionally robust optimization~\cite{rahimian2019distributionally,delage2010distributionally,goh2010distributionally}.

\subsection{Other related topics}
\paragraph{Active learning}
Active learning~\cite{settles2009active,prince2004does,kimura2020batch} is a subfield of machine learning that focuses on developing algorithms that can automatically select the most informative data to be labeled by an expert or a human annotator.
The goal of active learning is to achieve high accuracy models with minimal labeled data.
Active learning can help mitigate the effects of distribution shift by actively selecting the most informative data points to be labeled by an expert or a human annotator.
By doing so, the active learning algorithm can ensure that the labeled data used to train the model is representative of the data that the model will encounter in the real world.

\paragraph{Continual learning}
Continual learning~\cite{parisi2019continual}, also known as lifelong learning or incremental learning, is a subfield of machine learning that deals with the problem of learning from a continuous stream of data over time, without forgetting previously learned knowledge.
One of the key challenges in continual learning is avoiding catastrophic forgetting~\cite{kirkpatrick2017overcoming,kemker2018measuring}, which occurs when a model overwrites previously learned knowledge with new information. 
Both catastrophic forgetting and distribution shift can lead to poor model performance and may require the model to be retrained or updated to address these issues.
Continual learning algorithms are designed to address these challenges by enabling models to learn from a continuous stream of data over time, without forgetting previously learned knowledge and without being affected by distribution shift.

\paragraph{Transfer learning}
Transfer learning~\cite{weiss2016survey,torrey2010transfer,zhuang2020comprehensive} is a machine learning technique that involves leveraging knowledge learned from one task to improve performance on a different, but related task. In transfer learning, a model is first trained on a source task, which provides a foundation of knowledge and skills.
Then, the model is fine-tuned on a target task, which is related to the source task but may have different characteristics.
Transfer learning is motivated by the fact that many tasks in machine learning share common features and patterns.
By leveraging knowledge from a related task, a model can learn to recognize patterns and features more effectively, even if the target task has different characteristics.

\subsection{Fashion datasets}
Several fashion datasets have been introduced in the research community to facilitate the development and evaluation of fashion-related machine learning models. These datasets contain various types of fashion-related data, including images, textual descriptions, and attribute labels.

\paragraph{Fashion MNIST~\cite{xiao2017/online}}
Fashion MNIST is a publicly available dataset of Zalando's article images that is widely used for research and development in computer vision and machine learning. The dataset consists of a training set of 60,000 examples and a test set of 10,000 examples, each of which is a 28$\times$28 grayscale image. The images in the dataset are associated with labels from 10 different classes.

\paragraph{DeepFashion~\cite{liuLQWTcvpr16DeepFashion}}
DeepFashion is a dataset containing around 800K diverse fashion images with their rich annotations (46 categories, 1,000 descriptive attributes, bounding boxes and landmark information) ranging from well-posed product images to real-world-like consumer photos.

\paragraph{Fashionpedia~\cite{jia2020fashionpedia}}
Fashionpedia consists of user uploaded 48K street-fashion photos collected from free license websites such as Unsplash, Kaboompics etc. These photos contain people wearing variety of clothes and accessories captured in different background, weather and camera conditions.

\paragraph{Polyvore~\cite{vasileva2018learning}}
Polyvore is a crowd-sourced dataset containing outfits or sets of fashion items that complement each other. It consists of manually labeled 68K outfits that are split into 53K, 10K and 5K into training, validation and testing sets, respectively

\paragraph{Polyvore-disjoint~\cite{vasileva2018learning}}
Polyvore-disjoint is a subset of the Polyvore dataset created by removing outfits that have common items between training, validation and testing sets. The dataset is challenging compared to Polyvore dataset and consists of 32K outfits. During inference, we use only product images patches and do not use the metadata associated with the products.

\paragraph{Fashion IQ~\cite{guo2019fashion}}
The Fashion IQ dataset is a collection of images and associated metadata designed to facilitate research on natural language-based interactive image retrieval in the fashion domain.
It is the fashion dataset that includes human-written relative captions that have been annotated for similar pairs of images, as well as real-world product descriptions and attribute labels as side information. The dataset was created to help researchers develop new methods for retrieving fashion images using natural language queries.

\section{Open questions}
\paragraph{Learning guarantees for set-to-set matching under the distribution shift}
As introduced in Appendix~\ref{apd:known_generalization_bounds}, there are a number of studies on generalized error analysis under distribution shifts for ordinal classification and regression problems.
However, theoretical analysis of set matching under distributional shifts remains unexplored.
Moreover, even under the i.i.d. assumption, there is little research on theoretical analysis of set matching~\cite{kimura2023generalization}.

\paragraph{Correlation of performance on SHIFT15M dataset with performance on other different datasets}
As introduced in Section~\ref{sec:related_works_and_conclusion}, there are many datasets for classification and regression under distribution shifts.
Since SHIFT15M provides data loaders for ordinary classification and regression as well as set matching, it is useful to examine the correlation between performance on these tasks and performance on other data sets.

\newcommand{\dssectionheader}[1]{%
   \noindent\framebox[\linewidth]{%
      {\fontfamily{phv}\selectfont \textbf{\textcolor{blue}{#1}}}
   }
}

\newcommand{\dsquestion}[1]{%
   {\noindent \scriptsize {\fontfamily{phv}\selectfont \textcolor{blue}{\textbf{#1}}}}
}

\newcommand{\dsquestionex}[2]{%
   {\noindent \scriptsize {\fontfamily{phv}\selectfont \textcolor{blue}{\textbf{#1} #2}}}
}

\newcommand{\dsanswer}[1]{%
   {\noindent \footnotesize {#1} \medskip}
}

\clearpage

\begin{figure*}[t]
    \centering
    \includegraphics[width=0.99\linewidth]{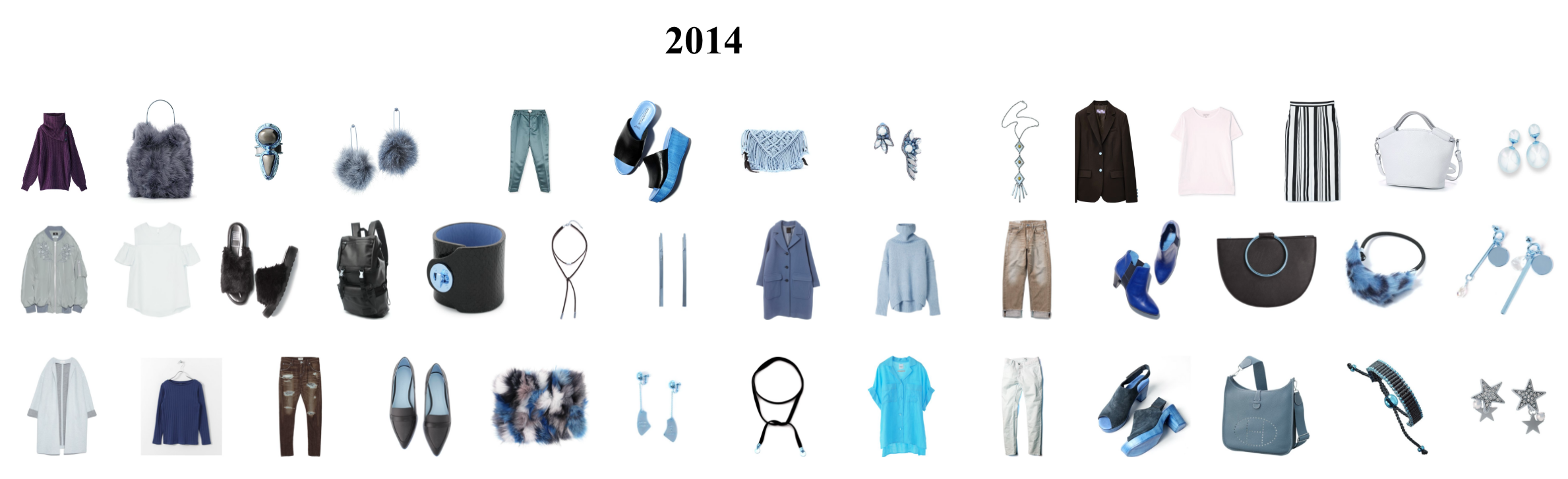}
    \caption{Additional sample items from 2014.}
    \label{fig:shift15m_additional_sample_images_2014}
\end{figure*}
\begin{figure*}[t]
    \centering
    \includegraphics[width=0.99\linewidth]{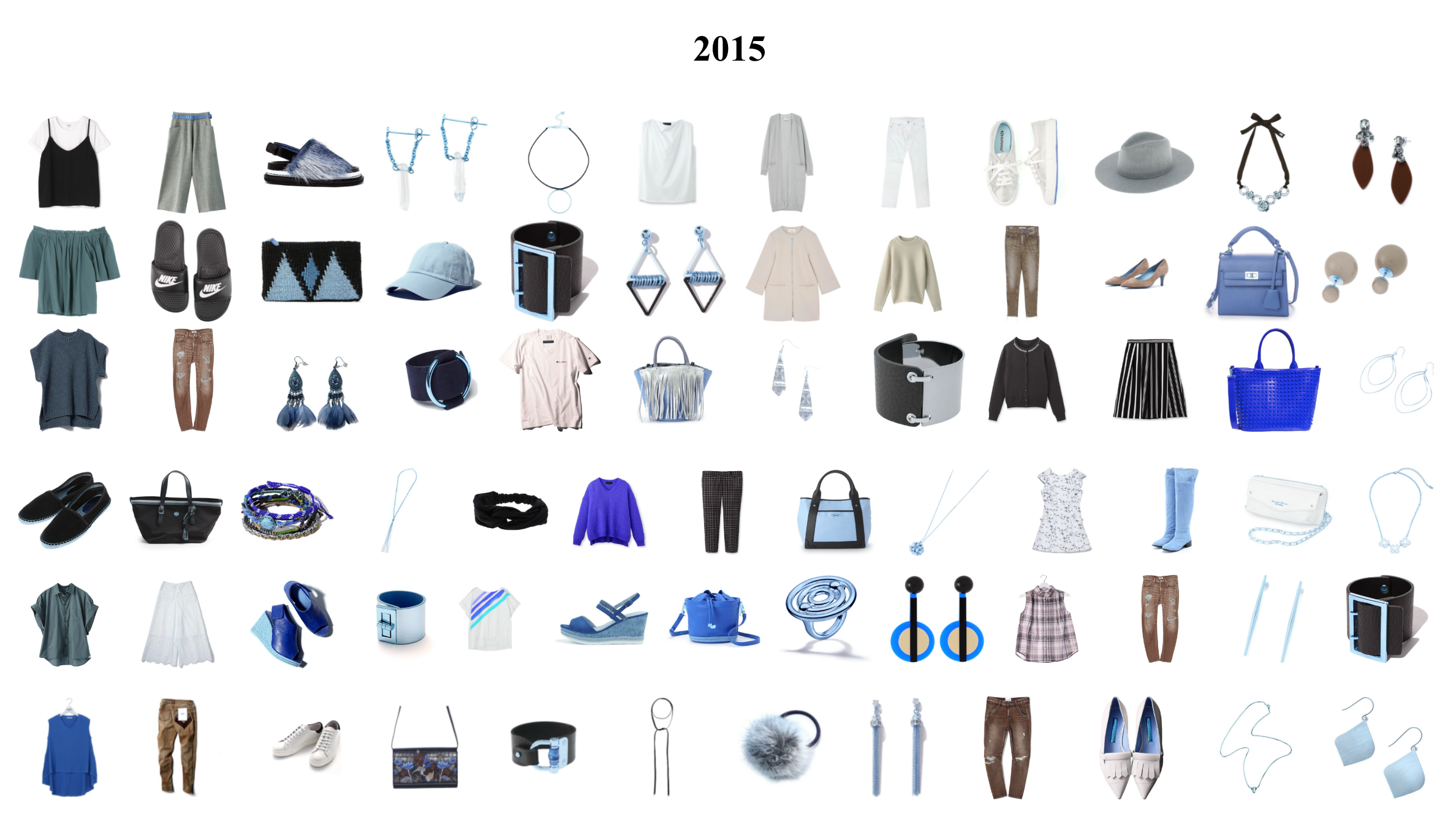}
    \caption{Additional sample items from 2015.}
    \label{fig:shift15m_additional_sample_images_2015}
\end{figure*}
\begin{figure*}[t]
    \centering
    \includegraphics[width=0.99\linewidth]{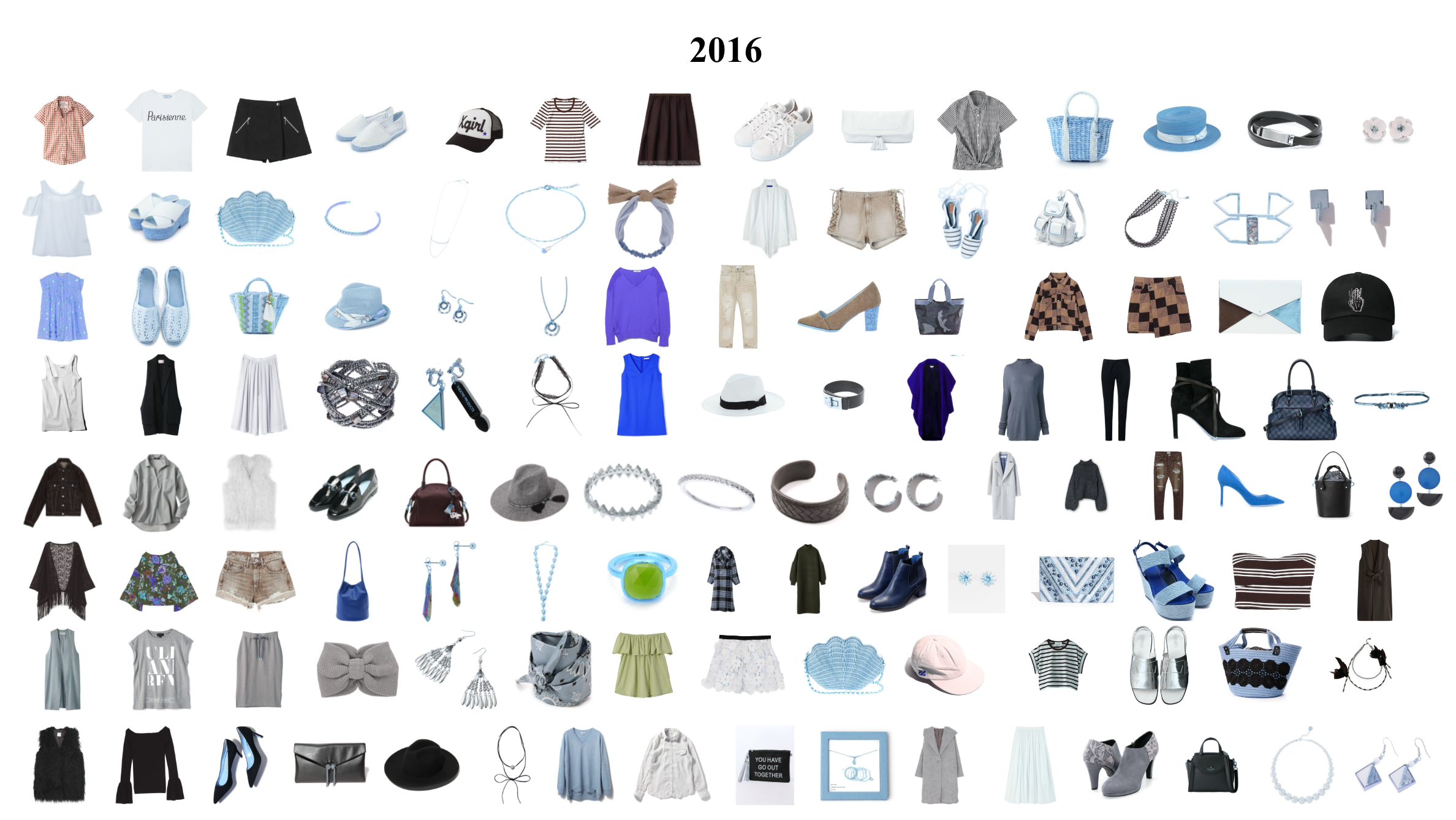}
    \caption{Additional sample items from 2016.}
    \label{fig:shift15m_additional_sample_images_2016}
\end{figure*}
\begin{figure*}[t]
    \centering
    \includegraphics[width=0.99\linewidth]{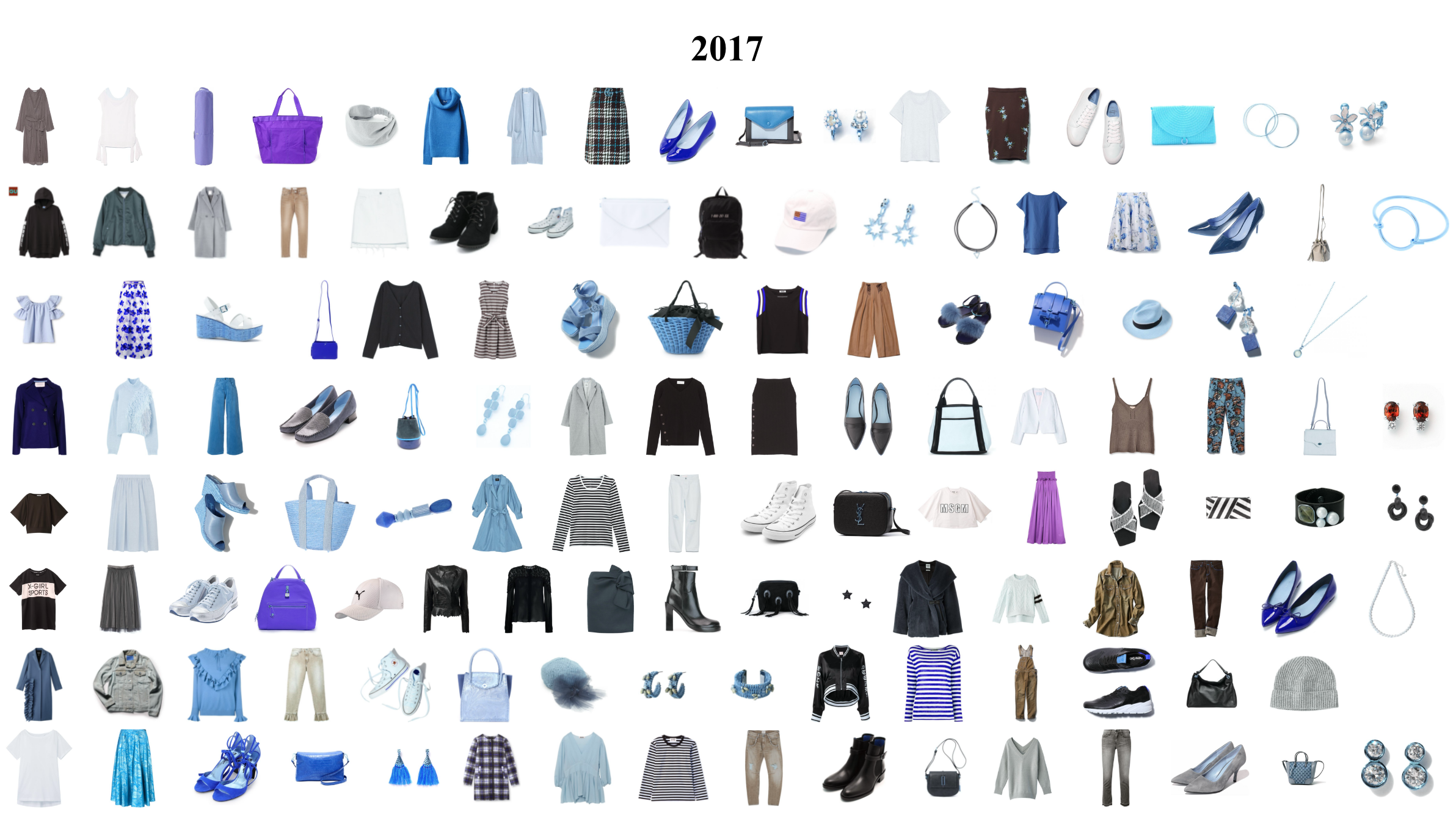}
    \caption{Additional sample items from 2017.}
    \label{fig:shift15m_additional_sample_images_2017}
\end{figure*}

\begin{figure*}
    \centering
    \includegraphics[width=0.95\linewidth]{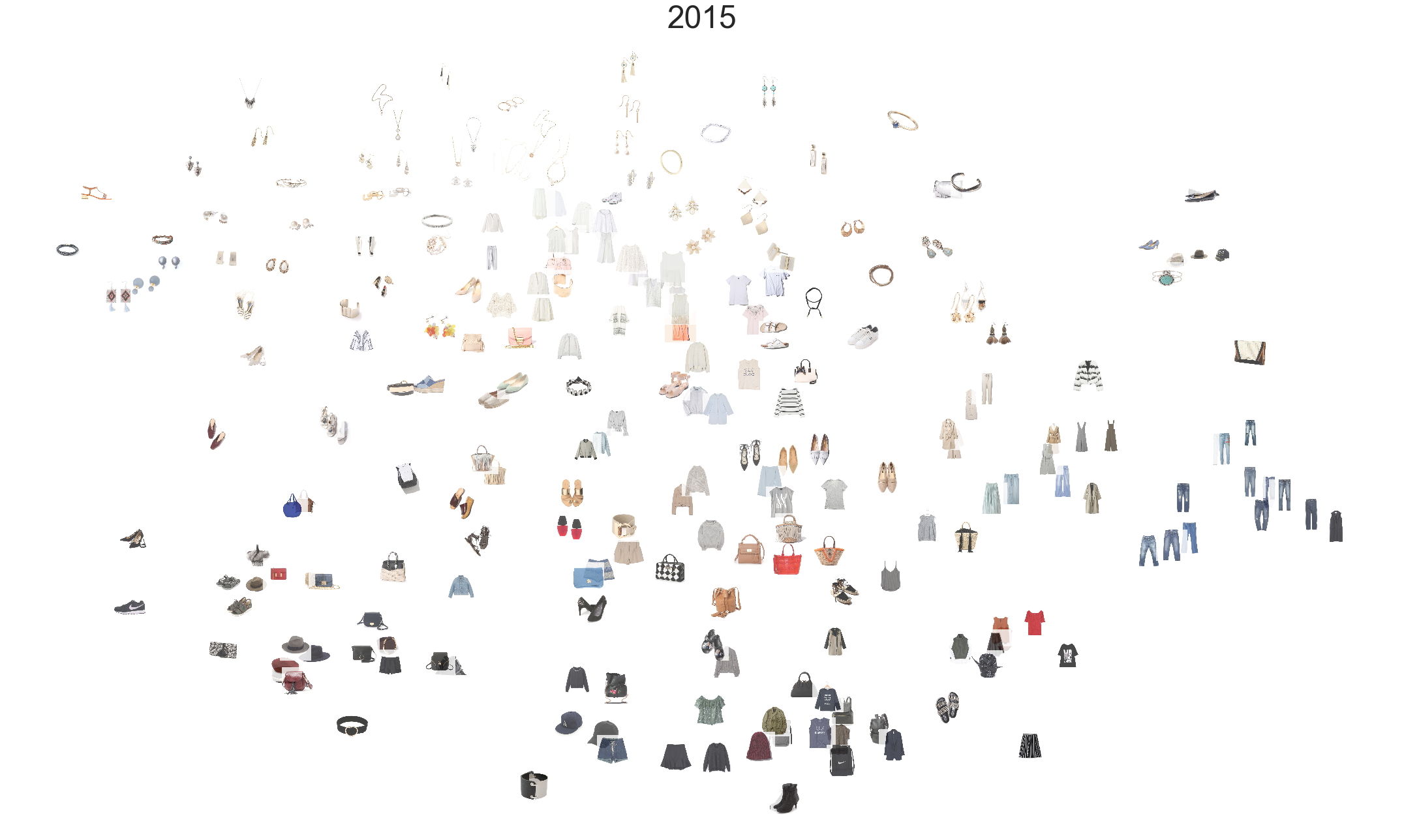}
    \caption{t-SNE visualization for the SHIFT15M (2015).}
    \label{fig:shift15m_tsne_2015}
\end{figure*}
\begin{figure*}
    \centering
    \includegraphics[width=0.95\linewidth]{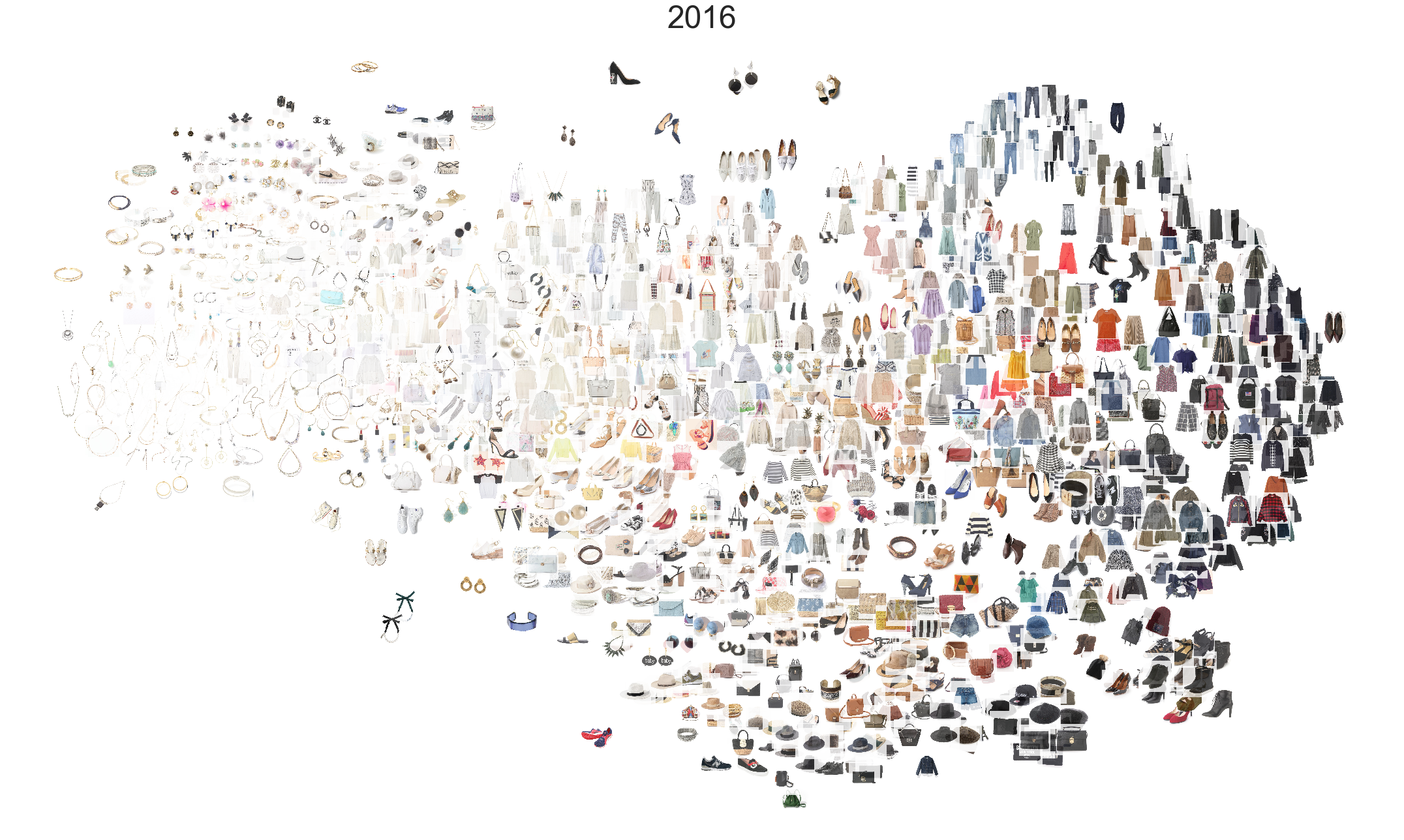}
    \caption{t-SNE visualization for the SHIFT15M (2016).}
    \label{fig:shift15m_tsne_2016}
\end{figure*}
\begin{figure*}
    \centering
    \includegraphics[width=0.95\linewidth]{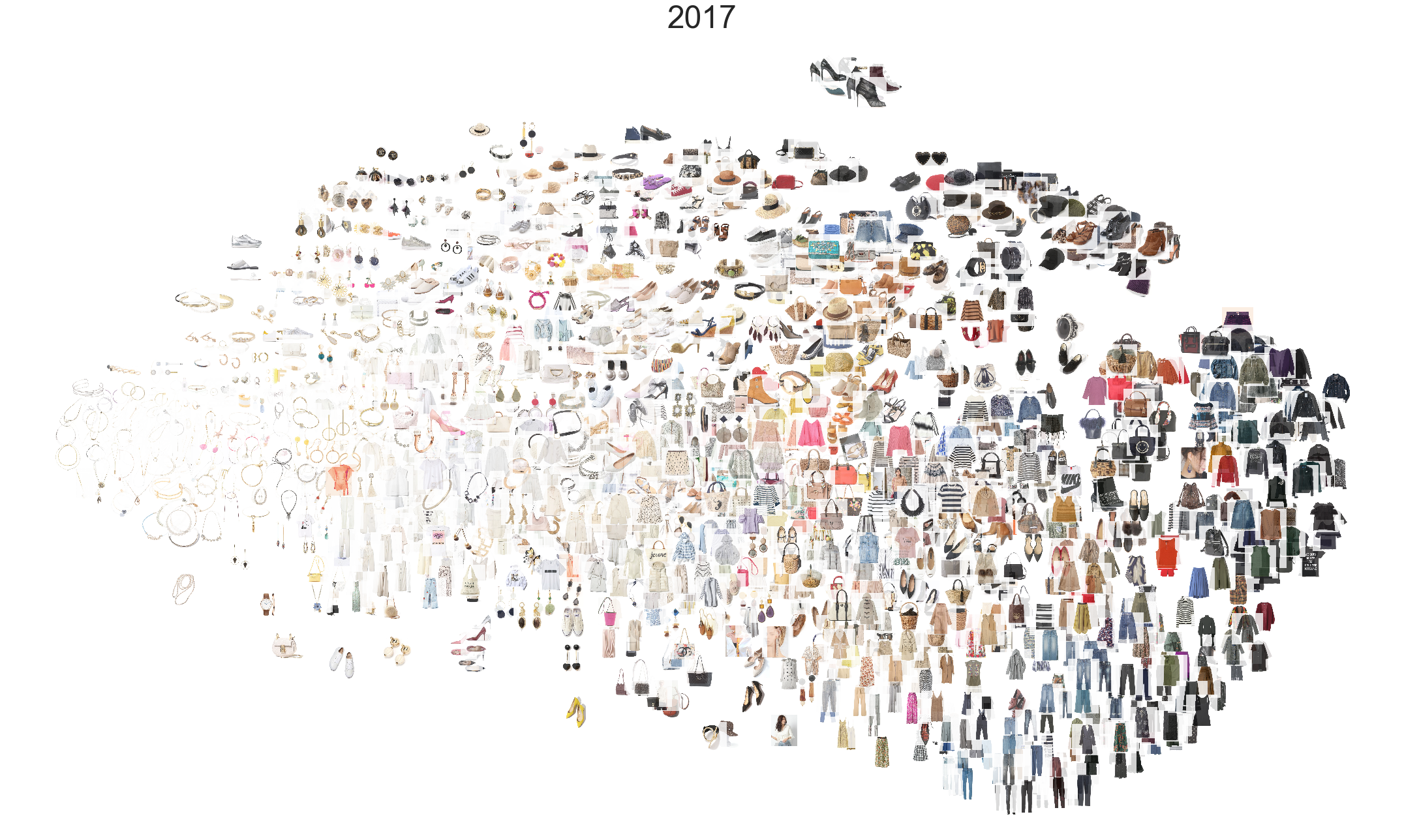}
    \caption{t-SNE visualization for the SHIFT15M (2017).}
    \label{fig:shift15m_tsne_2017}
\end{figure*}

\begin{figure*}
    \centering
    \includegraphics[width=0.95\linewidth]{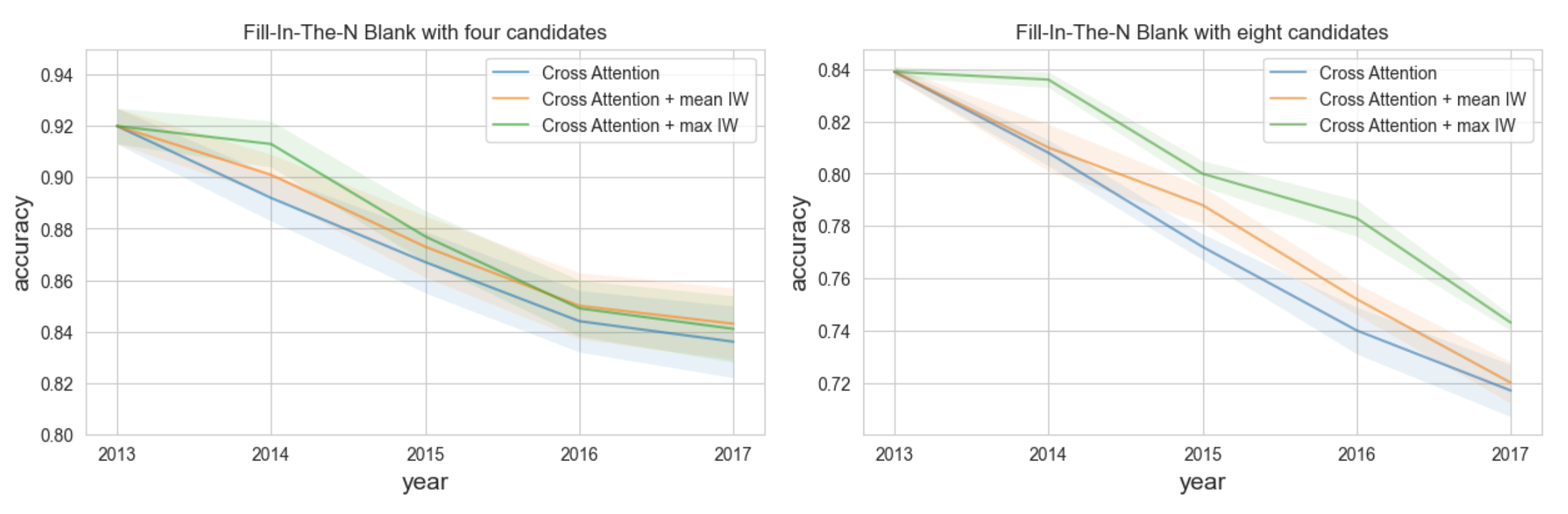}
    \caption{Plots for Fill-In-The-N-Blank experiments with Cross Attention.}
    \label{fig:fill_in_the_n_blank_cross_attention}
\end{figure*}
\begin{figure*}
    \centering
    \includegraphics[width=0.95\linewidth]{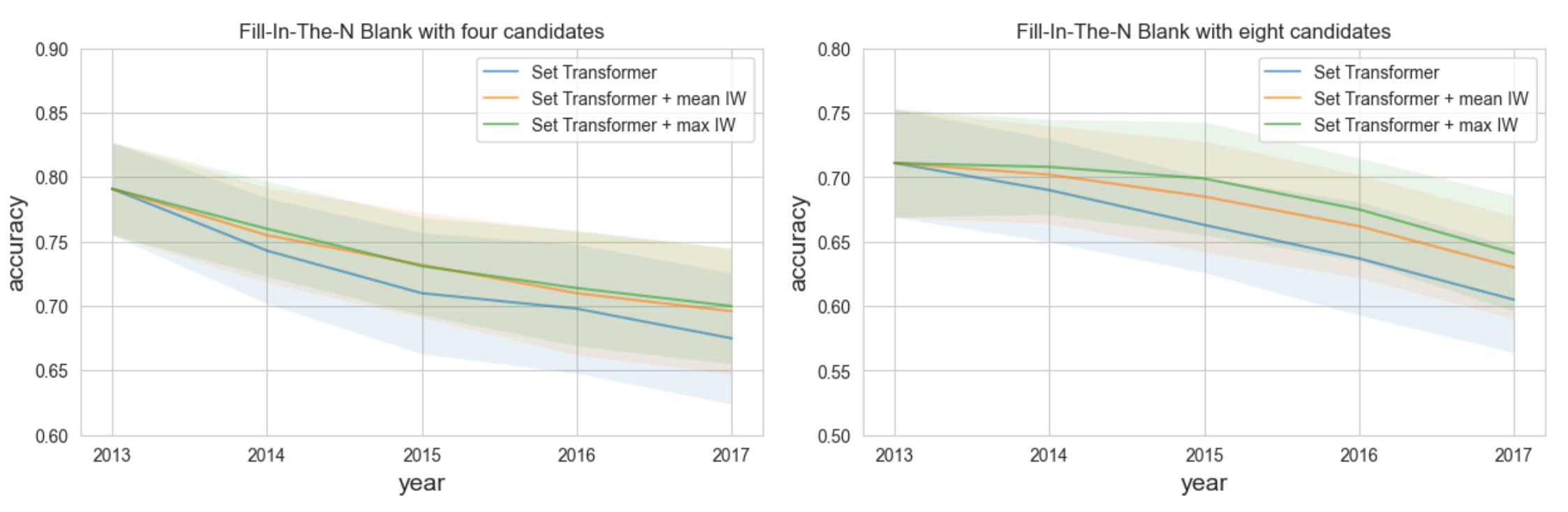}
    \caption{Plots for Fill-In-The-N-Blank experiments with Set Transformer.}
    \label{fig:fill_in_the_n_blank_set_transformer}
\end{figure*}
\begin{figure*}
    \centering
    \includegraphics[width=0.95\linewidth]{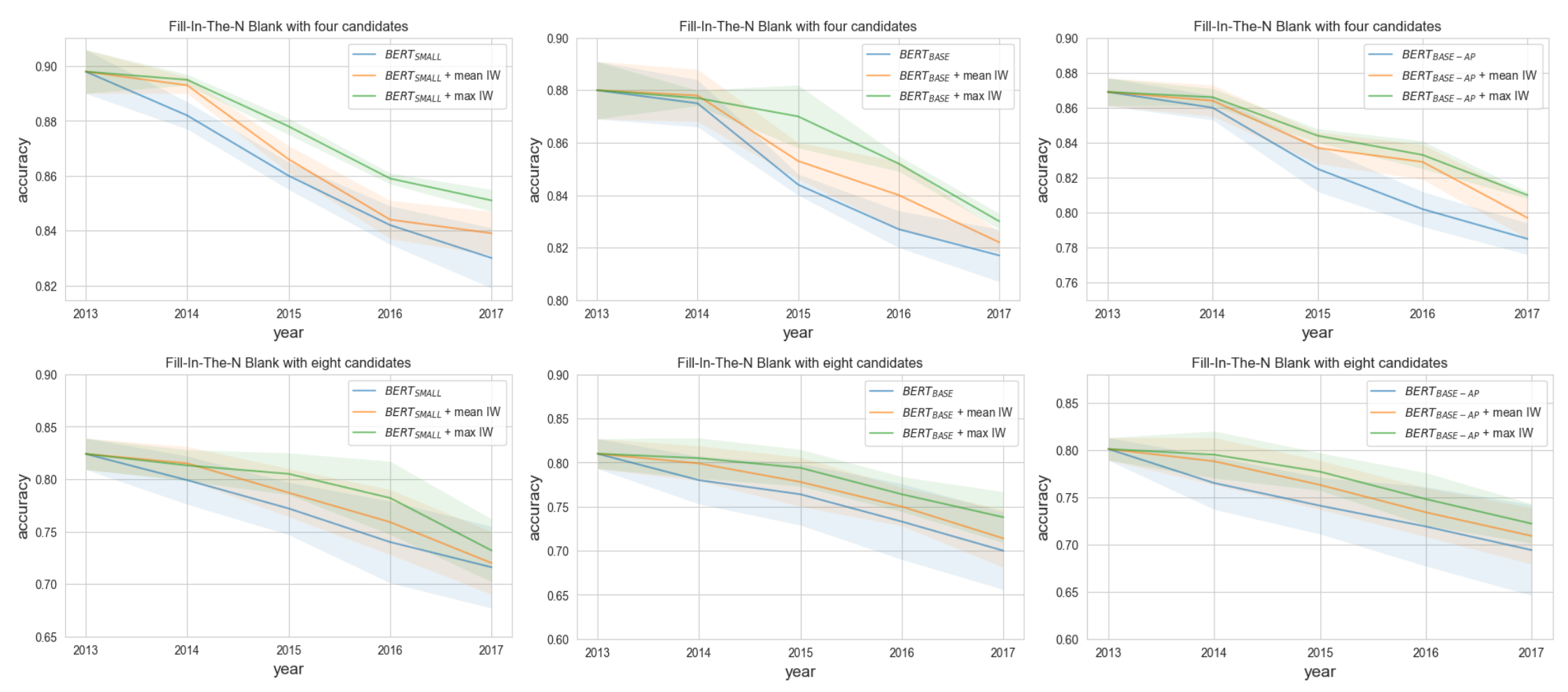}
    \caption{Plots for Fill-In-The-N-Blank experiments with BERT.}
    \label{fig:fill_in_the_n_blank_bert}
\end{figure*}
\begin{figure*}
    \centering
    \includegraphics[width=0.95\linewidth]{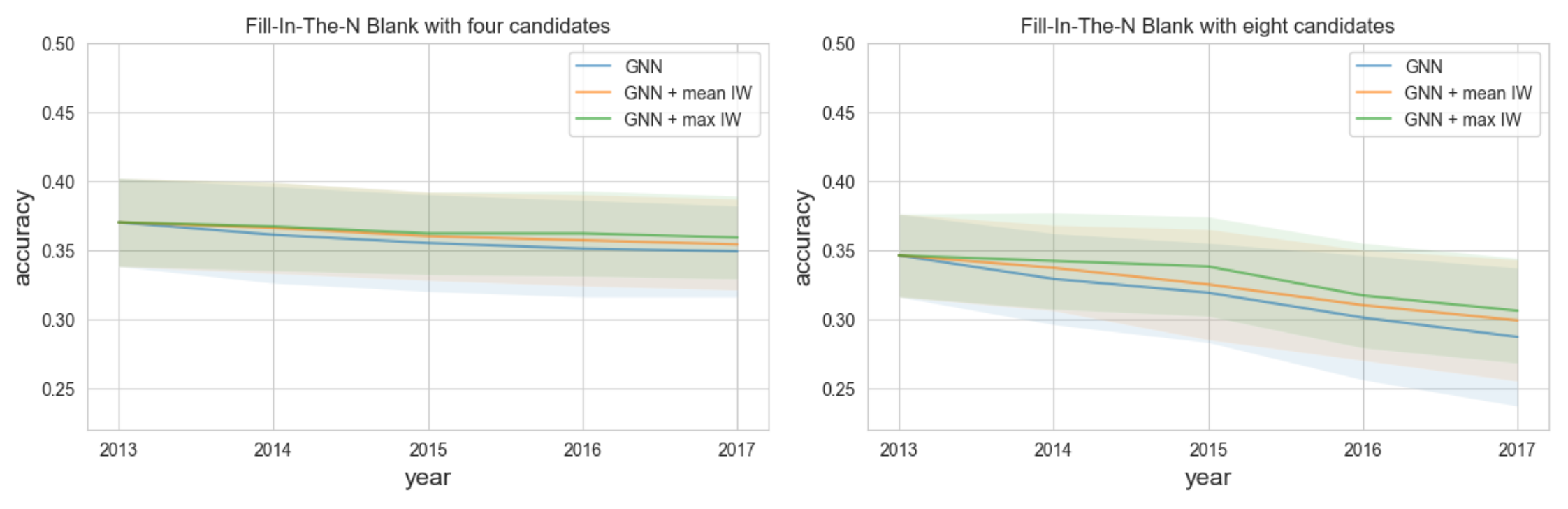}
    \caption{Plots for Fill-In-The-N-Blank experiments with GNN.}
    \label{fig:fill_in_the_n_blank_gnn}
\end{figure*}
\begin{figure*}
    \centering
    \includegraphics[width=0.95\linewidth]{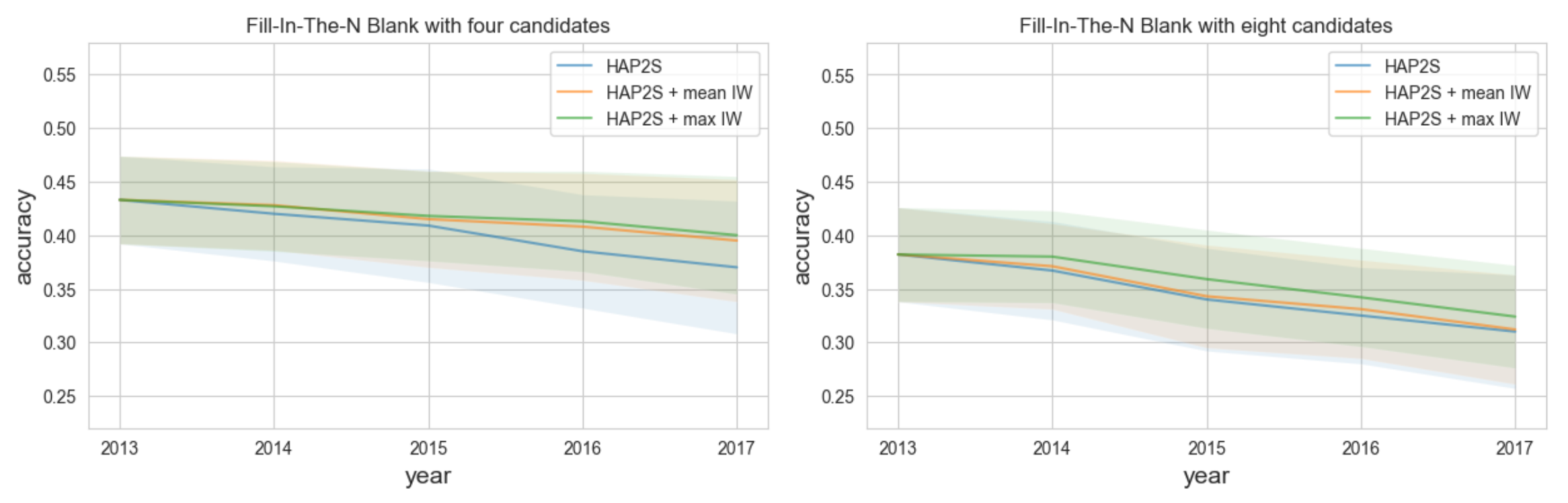}
    \caption{Plots for Fill-In-The-N-Blank experiments with HAP2S.}
    \label{fig:fill_in_the_n_blank_hap2s}
\end{figure*}
\begin{figure*}
    \centering
    \includegraphics[width=0.8\linewidth]{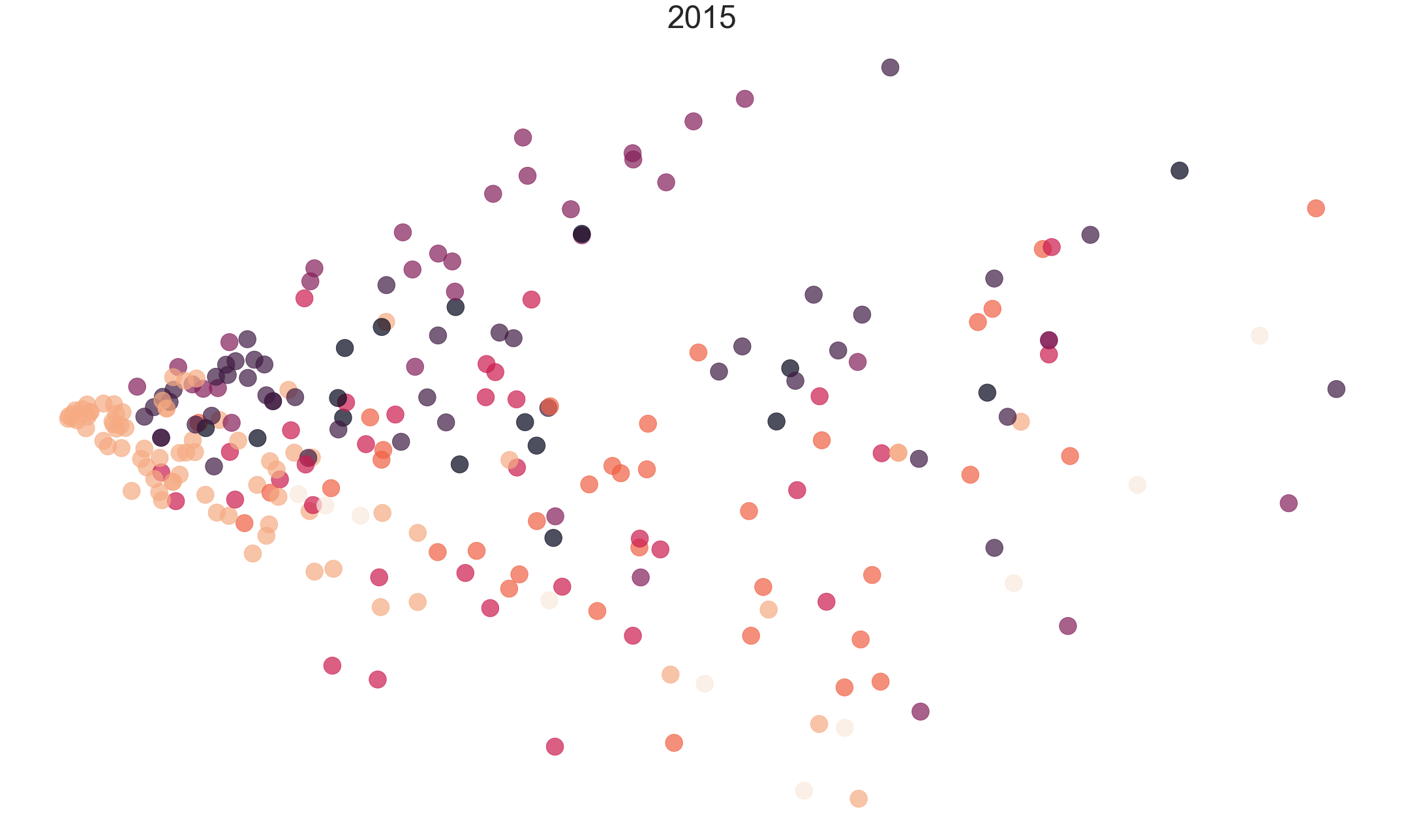}
    \caption{PCA for the SHIFT15M (2015). The color of each point corresponds to the item category.}
    \label{fig:pca_2d_2015}
\end{figure*}
\begin{figure*}
    \centering
    \includegraphics[width=0.8\linewidth]{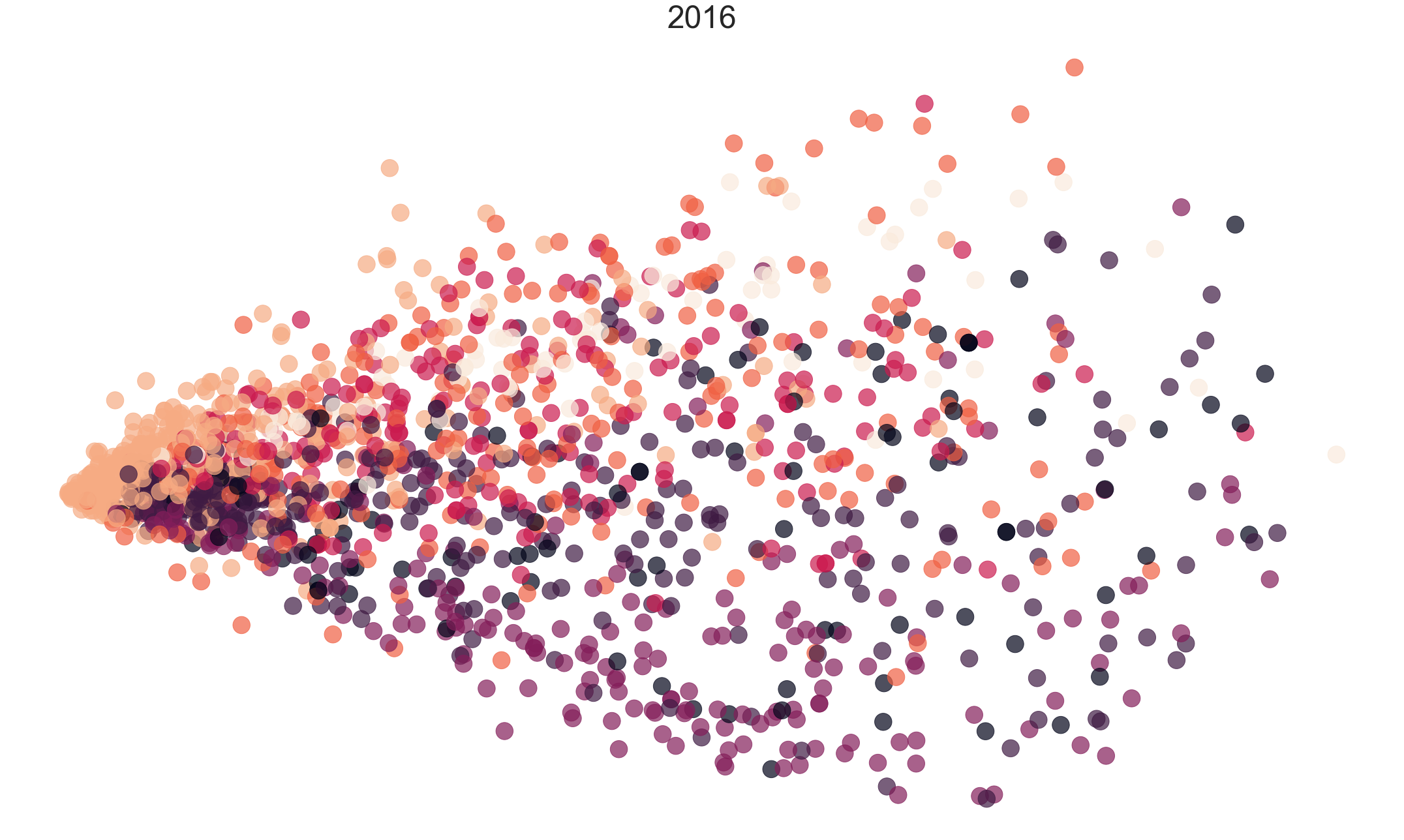}
    \caption{PCA for the SHIFT15M (2016). The color of each point corresponds to the item category.}
    \label{fig:pca_2d_2016}
\end{figure*}
\begin{figure*}
    \centering
    \includegraphics[width=0.8\linewidth]{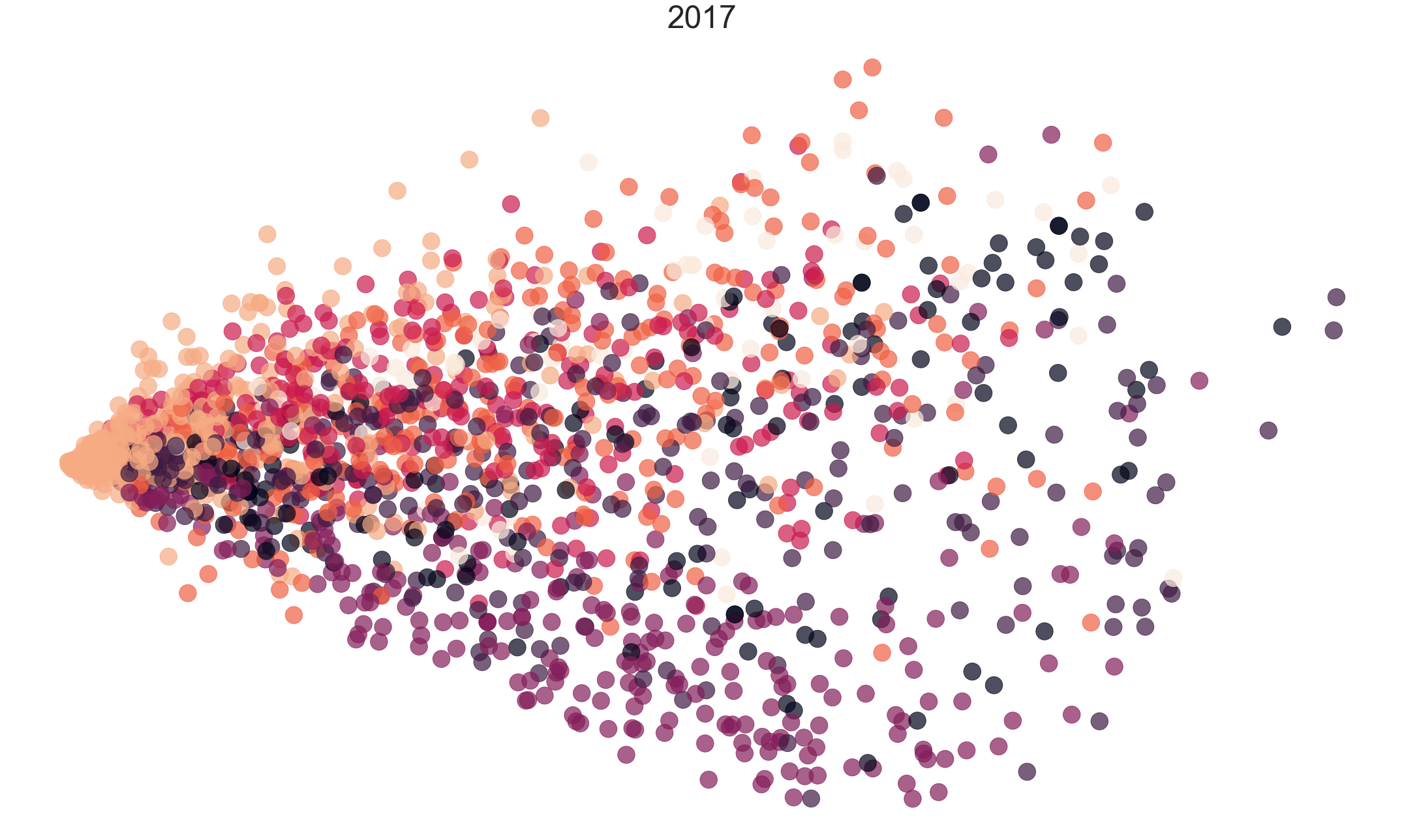}
    \caption{PCA for the SHIFT15M (2017). The color of each point corresponds to the item category.}
    \label{fig:pca_2d_2017}
\end{figure*}

\clearpage

\section{Datasheet for SHIFT15M}
\label{apd:datasheet}
In accordance with \cite{gebru2021datasheets}, we provide the datasheet for SHIFT15M. This datasheet offers an overview of crucial information about the dataset that can aid users in making informed decisions regarding its use. It encompasses a comprehensive range of details, including but not limited to the dataset's creation purpose, sources of data, and data processing methods.
Overall, the datasheet provides a vital resource for those interested in utilizing the SHIFT15M dataset, as it offers a transparent and comprehensive account of the dataset's characteristics and construction.

\paragraph{Motivation}
The purpose of this section is to emphasize the importance of transparency and clarity in the process of dataset creation, particularly with regards to the motivations behind the creation of the dataset and any potential conflicts of interest that may arise. Dataset creators are encouraged to clearly articulate their reasons for creating the dataset, including the research questions or goals that the dataset is intended to address.

\paragraph{Composition}
Dataset creators are advised to review these questions thoroughly prior to commencing any data collection, and provide responses after data collection has concluded. The primary purpose of the questions in this section is to equip dataset consumers with the information necessary to make informed decisions about utilizing the dataset for their specific needs. Additionally, some of the questions have been tailored to obtain information about adherence to the General Data Protection Regulation (GDPR) of the European Union or similar regulatory frameworks in other jurisdictions. Notably, questions specific to datasets involving individuals have been consolidated towards the end of the section. It is recommended a broad interpretation of what constitutes a dataset relating to people. For instance, any dataset containing text that was produced by individuals can be considered to relate to people.

\paragraph{Collection Process}
To ensure potential issues are identified, dataset creators are advised to review the questions in this section before initiating data collection and then furnish responses upon completion of collection, similar to the previous section. Apart from the objectives set out in the preceding section, the questions in this section are aimed at extracting information that could assist researchers and practitioners in producing alternative datasets with comparable attributes. Similarly, inquiries exclusive to datasets pertaining to individuals are categorized towards the end of this section.

\paragraph{Preprocessing/cleaning/labeling}
Before proceeding with any preprocessing, cleaning, or labeling, it is recommended that dataset creators review the questions presented in this section and subsequently provide responses upon completing these tasks. The purpose of the questions in this section is to equip dataset consumers with the requisite information to evaluate whether the "raw" data has been processed in a manner that aligns with their intended use.

\paragraph{Uses}
The questions in this section serve to prompt dataset creators to consider the appropriate and inappropriate uses of their dataset. By explicitly specifying such tasks, dataset creators can assist dataset consumers in making informed decisions, minimizing potential hazards or adverse outcomes.

\paragraph{Distribution}
Dataset creators should provide answers to these questions prior to distributing the dataset either internally within the entity on behalf of which the dataset was created or externally to third parties.

\paragraph{Maintenance}
As with the questions in the previous section, dataset creators should provide answers to these questions prior to distributing the dataset. The questions in this section are intended to encourage dataset creators to plan for dataset maintenance and communicate this plan to dataset consumers.

\clearpage

\dssectionheader{Motivation}

\dsquestionex{For what purpose was the dataset created?}{Was there a specific task in mind? Was there a specific gap that needed to be filled? Please provide a description.}

\dsanswer{
This paper addresses the problem of set-to-set matching, which involves matching two different sets of items based on some criteria, especially in the case of high-dimensional items like images. Although neural networks have been applied to solve this problem, most machine learning-based approaches assume that the training and test data follow the same distribution, which is not always true in real-world scenarios. To address this limitation, we introduce SHIFT15M, a dataset that can be used to evaluate set-to-set matching models when the distribution of data changes between training and testing. We conduct benchmark experiments that demonstrate the performance drop of naive methods due to distribution shift. Additionally, we provide software to handle the SHIFT15M dataset in a simple manner, with the URL for the software to be made available after publication of this manuscript. We believe proposed SHIFT15M dataset provide a valuable resource for evaluating set-to-set matching models under the distribution shift.
}

\dsquestion{Who created this dataset (e.g., which team, research group) and on behalf of which entity (e.g., company, institution, organization)?}

\dsanswer{
\textcolor{red}{Anonymized until after the paper is accepted.}
}

\dsquestionex{Who funded the creation of the dataset?}{If there is an associated grant, please provide the name of the grantor and the grant name and number.}

\dsanswer{
Not applicable.
}

\dsquestion{Any other comments?}

\bigskip
\dssectionheader{Composition}

\dsquestionex{What do the instances that comprise the dataset represent (e.g., documents, photos, people, countries)?}{ Are there multiple types of instances (e.g., movies, users, and ratings; people and interactions between them; nodes and edges)? Please provide a description.}

\dsanswer{
The SHIFT15M dataset is an extensive collection of outfits that were previously shared on a fashion-oriented social networking service. The service is no longer available, but the dataset continues to be a valuable resource for researchers studying set-to-set matching problems, particularly in the context of fashion. The dataset contains a vast array of information about each outfit, including details about the user who posted it and some meta-information. Each record in the dataset comprises five different fields, making it easy to organize and analyze the data. 

\begin{itemize}

\item set\_id: An ID that identifies the outfit that was posted.
\item items: Provides information about the items that comprise the posted outfit and consists of 4 subfields.
\begin{itemize}
\item{item\_id: An ID that identifies an item.}
\item{category\_id1: An ID indicating the item category (e.g., outerwear, tops, ...).}
\item{category\_id2: An ID indicating the item subcategory (e.g., T-shirts, blouses, ...).}
\item{price: Price of the item.
user: Provides information about the user who posted the outfit and consists of 2 subfields.
An ID that identifies the user who posted the outfit.
A list of brands that users have voted for as their favorites. The number is an ID that identifies the brand.}
\end{itemize}
\item{like\_num: the number of times this outfit has been favorited by other users.}
\item{publish\_date: The date the outfit was posted.}
\end{itemize}
}

\dsquestion{How many instances are there in total (of each type, if appropriate)?}

\dsanswer{
The dataset consists of 15,218,721 item images and 2,555,147 outfits which created by users of our fashion SNS.
}

\dsquestionex{Does the dataset contain all possible instances or is it a sample (not necessarily random) of instances from a larger set?}{ If the dataset is a sample, then what is the larger set? Is the sample representative of the larger set (e.g., geographic coverage)? If so, please describe how this representativeness was validated/verified. If it is not representative of the larger set, please describe why not (e.g., to cover a more diverse range of instances, because instances were withheld or unavailable).}

\dsanswer{
The fashion-oriented social networking service from which we collected outfits for the SHIFT15M dataset was a rich source of data that allowed us to obtain insights into the fashion trends and preferences of millions of users. With approximately 2 million users, the website was a bustling hub of activity where fashion enthusiasts could share their outfit ideas, provide feedback to others, and explore the latest styles and trends.

The vast majority of users on the website were women in their 20s and 30s, representing a demographic that is known for their fashion-forward mindset and interest in new trends. This demographic was particularly valuable for our research, as it allowed us to obtain a large amount of data on outfits that were representative of current fashion trends.

Our collection period spanned over a decade, starting on January 1st, 2010 and ending on April 6th, 2020. During this time, we meticulously gathered outfits consisting of multiple items, each of which was carefully categorized into a specific category. Our focus was on outfits that contained four or more items from the main categories, including outerwear, tops, bottoms, shoes, bags, hats, and accessories. This selection criterion was chosen with the aim of creating a dataset that would be useful for set-to-set matching tasks involving fashion items.
}

\dsquestionex{What data does each instance consist of? “Raw” data (e.g., unprocessed text or images) or features?}{In either case, please provide a description.}

\dsanswer{
Each item consists of 4096-dimensional features extracted via the VGG16 model trained using the ILSVRC2012 dataset.
}

\dsquestionex{Is there a label or target associated with each instance?}{If so, please provide a description.}

\dsanswer{
Indeed, each instance in the SHIFT15M dataset contains a wealth of information that can be leveraged for various tasks. Along with the outfit items, each instance also includes several numerical values such as the category ID and number of likes. These values provide additional context that can be used to train models for various set-to-set matching problems.

One of the strengths of the SHIFT15M dataset is its versatility. By choosing one of these numerical values as the target variable, researchers can easily switch between several tasks, each with its own unique set of challenges and opportunities. For example, if the focus is on predicting outfit popularity, the number of likes can be used as the target variable. On the other hand, if the goal is to perform set-to-set matching between outfits, the number of likes can be used as the target variable.
}

\dsquestionex{Is any information missing from individual instances?}{If so, please provide a description, explaining why this information is missing (e.g., because it was unavailable). This does not include intentionally removed information, but might include, e.g., redacted text.}

\dsanswer{
It is important to note that the SHIFT15M dataset only includes items that belong to the main categories, such as outerwear, tops, bottoms, shoes, bags, hats, and accessories. This means that items outside of these categories, such as underwear or background images for collage, are missing from the dataset.

While this selection criterion may seem limiting, it was chosen to ensure that the dataset is focused on items that are most commonly used in fashion-related set-to-set matching tasks. By excluding items outside of the main categories, we were able to curate a dataset that is more manageable and less noisy, while still providing a diverse range of fashion items for analysis.
}

\dsquestionex{Are relationships between individual instances made explicit (e.g., users’ movie ratings, social network links)?}{If so, please describe how these relationships are made explicit.}

\dsanswer{
Each instance is assigned the ID of the user who submitted the outfit.
}

\dsquestionex{Are there recommended data splits (e.g., training, development/validation, testing)?}{If so, please provide a description of these splits, explaining the rationale behind them.}

\dsanswer{
SHIFT15M is a valuable dataset that simulates various types of dataset shifts that are commonly observed in real-world applications. The collected data spans a decade, from 2010 to 2020, and encompasses various shifts that arise due to factors such as changes in user behavior, fashion trends, and cultural preferences. To make it easy for researchers to evaluate their models on the SHIFT15M dataset, we have developed software that allows them to experiment with different types and sizes of shifts. The software automates the train/val/test splitting process, making it easier for researchers to evaluate the performance of their models under various shift scenarios. With this software, researchers can simulate shifts that arise due to various factors and assess their models' robustness to such shifts. By doing so, they can gain insights into how their models perform in real-world settings where data distributions are constantly changing.
}

\dsquestionex{Are there any errors, sources of noise, or redundancies in the dataset?}{If so, please provide a description.}

\dsanswer{
No.
}

\dsquestionex{Is the dataset self-contained, or does it link to or otherwise rely on external resources (e.g., websites, tweets, other datasets)?}{If it links to or relies on external resources, a) are there guarantees that they will exist, and remain constant, over time; b) are there official archival versions of the complete dataset (i.e., including the external resources as they existed at the time the dataset was created); c) are there any restrictions (e.g., licenses, fees) associated with any of the external resources that might apply to a future user? Please provide descriptions of all external resources and any restrictions associated with them, as well as links or other access points, as appropriate.}

\dsanswer{
The dataset is self-contained.
}

\dsquestionex{Does the dataset contain data that might be considered confidential (e.g., data that is protected by legal privilege or by doctor-patient confidentiality, data that includes the content of individuals non-public communications)?}{If so, please provide a description.}

\dsanswer{
No.
}

\dsquestionex{Does the dataset contain data that, if viewed directly, might be offensive, insulting, threatening, or might otherwise cause anxiety?}{If so, please describe why.}

\dsanswer{
No.
}

\dsquestionex{Does the dataset relate to people?}{If not, you may skip the remaining questions in this section.}

\dsanswer{
Yes. Each instance is a combination of outfits created by an individual and preferred by that individual.
}

\dsquestionex{Does the dataset identify any subpopulations (e.g., by age, gender)?}{If so, please describe how these subpopulations are identified and provide a description of their respective distributions within the dataset.}

\dsanswer{
No.
}

\dsquestionex{Is it possible to identify individuals (i.e., one or more natural persons), either directly or indirectly (i.e., in combination with other data) from the dataset?}{If so, please describe how.}

\dsanswer{
It is impossible to identify individuals from the dataset.
}

\dsquestionex{Does the dataset contain data that might be considered sensitive in any way (e.g., data that reveals racial or ethnic origins, sexual orientations, religious beliefs, political opinions or union memberships, or locations; financial or health data; biometric or genetic data; forms of government identification, such as social security numbers; criminal history)?}{If so, please provide a description.}

\dsanswer{
No.
}

\dsquestion{Any other comments?}

\bigskip
\dssectionheader{Collection Process}

\dsquestionex{How was the data associated with each instance acquired?}{Was the data directly observable (e.g., raw text, movie ratings), reported by subjects (e.g., survey responses), or indirectly inferred/derived from other data (e.g., part-of-speech tags, model-based guesses for age or language)? If data was reported by subjects or indirectly inferred/derived from other data, was the data validated/verified? If so, please describe how.}

\dsanswer{
Except for the item attributes, the data was generated by users. Item attributes (category and price) were collected from e-commerce sites that sell the item. All data was viewable on the website.
}

\dsquestionex{What mechanisms or procedures were used to collect the data (e.g., hardware apparatus or sensor, manual human curation, software program, software API)?}{How were these mechanisms or procedures validated?}

\dsanswer{
The fashion-oriented social networking service that we collected the SHIFT15M dataset from provided its users with an easy-to-use outfit editor that allowed them to create and publish their outfits on the platform. This editor featured a wide range of clothing items, accessories, and other fashion-related items that users could choose from to create their outfits. Once a user had selected the items they wanted to include in their outfit, the editor registered this selection as a new outfit on the platform.

To ensure that this outfit creation function was tested appropriately, we followed general software development procedures. This involved conducting thorough testing to ensure that the editor functioned as intended, with no bugs or glitches that could affect the accuracy or reliability of the data collected. By following this rigorous testing process, we were able to gather a high-quality dataset that accurately reflects the fashion choices made by users on the social networking service during the ten-year collection period. 
}

\dsquestion{If the dataset is a sample from a larger set, what was the sampling strategy (e.g., deterministic, probabilistic with specific sampling probabilities)?}

\dsanswer{
We collected a complete dataset without sampling to create our dataset, except for data deleted by the user.
}

\dsquestion{Who was involved in the data collection process (e.g., students, crowdworkers, contractors) and how were they compensated (e.g., how much were crowdworkers paid)?}

\dsanswer{
\textcolor{red}{Anonymized until after the paper is accepted.}
}

\dsquestionex{Over what timeframe was the data collected? Does this timeframe match the creation timeframe of the data associated with the instances (e.g., recent crawl of old news articles)?}{If not, please describe the timeframe in which the data associated with the instances was created.}

\dsanswer{
The dataset was collected in the period of 2010~2020. Each outfit includes a timestamp that describes when the outfit created.
}

\dsquestionex{Were any ethical review processes conducted (e.g., by an institutional review board)?}{If so, please provide a description of these review processes, including the outcomes, as well as a link or other access point to any supporting documentation.}

\dsanswer{
No.
}

\dsquestionex{Does the dataset relate to people?}{If not, you may skip the remaining questions in this section.}

\dsanswer{
Yes. Each instance is a combination of outfits created by an individual and preferred by that individual.
}

\dsquestion{Did you collect the data from the individuals in question directly, or obtain it via third parties or other sources (e.g., websites)?}

\dsanswer{
Collected directly through the website.
}

\dsquestionex{Were the individuals in question notified about the data collection?}{If so, please describe (or show with screenshots or other information) how notice was provided, and provide a link or other access point to, or otherwise reproduce, the exact language of the notification itself.}

\dsanswer{
Notified in the Terms of Service.
}

\dsquestionex{Did the individuals in question consent to the collection and use of their data?}{If so, please describe (or show with screenshots or other information) how consent was requested and provided, and provide a link or other access point to, or otherwise reproduce, the exact language to which the individuals consented.}

\dsanswer{
The use of the service was deemed as consent.
}

\dsquestionex{If consent was obtained, were the consenting individuals provided with a mechanism to revoke their consent in the future or for certain uses?}{If so, please provide a description, as well as a link or other access point to the mechanism (if appropriate).}

\dsanswer{
It is possible to contact the company that provided the service.
}

\dsquestionex{Has an analysis of the potential impact of the dataset and its use on data subjects (e.g., a data protection impact analysis) been conducted?}{If so, please provide a description of this analysis, including the outcomes, as well as a link or other access point to any supporting documentation.}

\dsanswer{
No, there had been no potential impact analysis conducted.
}

\dsquestion{Any other comments?}

\dsanswer{
}

\bigskip
\dssectionheader{Preprocessing/cleaning/labeling}

\dsquestionex{Was any preprocessing/cleaning/labeling of the data done (e.g., discretization or bucketing, tokenization, part-of-speech tagging, SIFT feature extraction, removal of instances, processing of missing values)?}{If so, please provide a description. If not, you may skip the remainder of the questions in this section.}

\dsanswer{
We extracted the CNN features from images and treated them as input data in our image-based tasks. As a result, our dataset contains the features but does not include raw photos, making them anonymized. The CNN we used is an official pre-trained VGG16, and we adopted the outputs of the 'fc6' layer before applying ReLU as the feature. We exclude the outfits that contain less than four items. Other than that, we did not remove any instances in creating our dataset. However, we excluded some data in each independent task. In detail, please refer to each task description.
}

\dsquestionex{Was the “raw” data saved in addition to the preprocessed/cleaned/labeled data (e.g., to support unanticipated future uses)?}{If so, please provide a link or other access point to the “raw” data.}

\dsanswer{
No.
}

\dsquestionex{Is the software used to preprocess/clean/label the instances available?}{If so, please provide a link or other access point.}

\dsanswer{
All software are provided on the SHIFT15M repository.
}

\dsquestion{Any other comments?}

\dsanswer{
}

\bigskip
\dssectionheader{Uses}

\dsquestionex{Has the dataset been used for any tasks already?}{If so, please provide a description.}

\dsanswer{
Benchmarks using this dataset and the specified evaluation protocol are listed in GitHub page.
}

\dsquestionex{Is there a repository that links to any or all papers or systems that use the dataset?}{If so, please provide a link or other access point.}

\dsanswer{
All benchmarks that use this dataset will be available at GitHub page.
}

\dsquestion{What (other) tasks could the dataset be used for?}

\dsanswer{
Here, we list candidate tasks for which SHIFT15M can be applied as follows:
\begin{itemize}
    \item set-to-set matching;
    \item regression (e.g., number of likes or sum of prices);
    \item classification (e.g., category ids or publish years).
\end{itemize}
}

\dsquestionex{Is there anything about the composition of the dataset or the way it was collected and preprocessed/cleaned/labeled that might impact future uses?}{For example, is there anything that a future user might need to know to avoid uses that could result in unfair treatment of individuals or groups (e.g., stereotyping, quality of service issues) or other undesirable harms (e.g., financial harms, legal risks) If so, please provide a description. Is there anything a future user could do to mitigate these undesirable harms?}

\dsanswer{
No.
}

\dsquestionex{Are there tasks for which the dataset should not be used?}{If so, please provide a description.}

\dsanswer{
This dataset is distributed in a way that excluding raw images and anonymizing the users/brands. Therefore, it requires the dataset users not to reconstruct raw images from the image features or restore the anonymized parts in a future task.
}

\dsquestion{Any other comments?}

\bigskip
\dssectionheader{Distribution}

\dsquestionex{Will the dataset be distributed to third parties outside of the entity (e.g., company, institution, organization) on behalf of which the dataset was created?}{If so, please provide a description.}

\dsanswer{
Yes. The dataset will be distributed to third parties based on the licence.
}

\dsquestionex{How will the dataset will be distributed (e.g., tarball on website, API, GitHub)}{Does the dataset have a digital object identifier (DOI)?}

\dsanswer{
The dataset will be distributed via a website or the links indicated in our Github repository. 
We will add DOI for the SHIFT15M dataset.
}

\dsquestion{When will the dataset be distributed?}

\dsanswer{
The dataset will be first released in August 2021.
}

\dsquestionex{Will the dataset be distributed under a copyright or other intellectual property (IP) license, and/or under applicable terms of use (ToU)?}{If so, please describe this license and/or ToU, and provide a link or other access point to, or otherwise reproduce, any relevant licensing terms or ToU, as well as any fees associated with these restrictions.}

\dsanswer{
The SHIFT15M dataset will be made available for distribution under the Creative Commons Attribution-NonCommercial 4.0 International (CC BY-NC 4.0) license. This means that users are free to share and adapt the dataset, as long as they provide attribution and do not use it for commercial purposes. For more information about the license, please refer to the following link: \url{https://creativecommons.org/licenses/by-nc/4.0/}. This license ensures that the dataset can be used by the academic community for research purposes, and that any derivative works or publications based on the dataset will be properly attributed.
}

\dsquestionex{Have any third parties imposed IP-based or other restrictions on the data associated with the instances?}{If so, please describe these restrictions, and provide a link or other access point to, or otherwise reproduce, any relevant licensing terms, as well as any fees associated with these restrictions.}

\dsanswer{
There are no fees or restrictions.
}

\dsquestionex{Do any export controls or other regulatory restrictions apply to the dataset or to individual instances?}{If so, please describe these restrictions, and provide a link or other access point to, or otherwise reproduce, any supporting documentation.}

\dsanswer{
Unknown.
}

\dsquestion{Any other comments?}

\dsanswer{
}

\bigskip
\newpage
\dssectionheader{Maintenance}

\dsquestion{Who will be supporting/hosting/maintaining the dataset?}

\dsanswer{
\textcolor{red}{Anonymized until after the paper is accepted.}
}

\dsquestion{How can the owner/curator/manager of the dataset be contacted (e.g., email address)?}

\dsanswer{
All changes to the dataset will be announced through the GitHub Releases.
}

\dsquestionex{Is there an erratum?}{If so, please provide a link or other access point.}

\dsanswer{
To ensure the accuracy and transparency of the SHIFT15M dataset, any changes made to the dataset will be immediately announced through the GitHub Releases page. This will include updates to the dataset's documentation, modifications to the dataset's format or metadata, or any other changes that may impact the dataset's use. Additionally, any errors or issues found in the dataset will be listed in the "Errata" section of the SHIFT15M repository. This will allow users to stay informed of any updates or issues related to the dataset and ensure that they are working with the most accurate and up-to-date version of the data.
}

\dsquestionex{Will the dataset be updated (e.g., to correct labeling errors, add new instances, delete instances)?}{If so, please describe how often, by whom, and how updates will be communicated to users (e.g., mailing list, GitHub)?}

\dsanswer{
All changes to the dataset will be announced through the GitHub Releases.
}

\dsquestionex{If the dataset relates to people, are there applicable limits on the retention of the data associated with the instances (e.g., were individuals in question told that their data would be retained for a fixed period of time and then deleted)?}{If so, please describe these limits and explain how they will be enforced.}

\dsanswer{
No.
}

\dsquestionex{Will older versions of the dataset continue to be supported/hosted/maintained?}{If so, please describe how. If not, please describe how its obsolescence will be communicated to users.}

\dsanswer{
They will continue to be supported with all information on SHIFT15M repository. We also provide the contribution guides for software that supports the dataset.
}

\dsquestionex{If others want to extend/augment/build on/contribute to the dataset, is there a mechanism for them to do so?}{If so, please provide a description. Will these contributions be validated/verified? If so, please describe how. If not, why not? Is there a process for communicating/distributing these contributions to other users? If so, please provide a description.}

\dsanswer{
Others may do so and should contact the original authors about incorporating fixes/extensions.
}

\dsquestion{Any other comments?}

\dsanswer{
}


{\small

\begin{thebibliography}{999}\itemsep=-1pt

\bibitem{abadi2016tensorflow}
Mart{\'\i}n Abadi, Paul Barham, Jianmin Chen, Zhifeng Chen, Andy Davis, Jeffrey
  Dean, Matthieu Devin, Sanjay Ghemawat, Geoffrey Irving, Michael Isard, et~al.
\newblock Tensorflow: A system for large-scale machine learning.
\newblock In {\em 12th $\{$USENIX$\}$ Symposium on Operating Systems Design and
  Implementation ($\{$OSDI$\}$ 16)}, pages 265--283, 2016.

\bibitem{alippi2008just}
Cesare Alippi and Manuel Roveri.
\newblock Just-in-time adaptive classifiers—part i: Detecting nonstationary
  changes.
\newblock {\em IEEE Transactions on Neural Networks}, 19(7):1145--1153, 2008.

\bibitem{arandjelovic2014discriminative}
Ognjen Arandjelovi{\'c}.
\newblock Discriminative extended canonical correlation analysis for pattern
  set matching.
\newblock {\em Machine Learning}, 94:353--370, 2014.

\bibitem{arjovsky2020out}
Martin Arjovsky.
\newblock {\em Out of distribution generalization in machine learning}.
\newblock PhD thesis, New York University, 2020.

\bibitem{bach2008paired}
Stephen~H Bach and Marcus~A Maloof.
\newblock Paired learners for concept drift.
\newblock In {\em 2008 Eighth IEEE International Conference on Data Mining},
  pages 23--32. IEEE, 2008.

\bibitem{baena2006early}
Manuel Baena-Garc{\i}a, Jos{\'e} del Campo-{\'A}vila, Raul Fidalgo, Albert
  Bifet, Ricard Gavalda, and Rafael Morales-Bueno.
\newblock Early drift detection method.
\newblock In {\em Fourth international workshop on knowledge discovery from
  data streams}, volume~6, pages 77--86. Citeseer, 2006.

\bibitem{bai2018convolutional}
Yunsheng Bai, Hao Ding, Yizhou Sun, and Wei Wang.
\newblock Convolutional set matching for graph similarity.
\newblock {\em arXiv preprint arXiv:1810.10866}, 2018.

\bibitem{bandi2018detection}
Peter Bandi, Oscar Geessink, Quirine Manson, Marcory Van~Dijk, Maschenka
  Balkenhol, Meyke Hermsen, Babak~Ehteshami Bejnordi, Byungjae Lee, Kyunghyun
  Paeng, Aoxiao Zhong, et~al.
\newblock From detection of individual metastases to classification of lymph
  node status at the patient level: the camelyon17 challenge.
\newblock {\em IEEE Transactions on Medical Imaging}, 2018.

\bibitem{basseville1993detection}
Michele Basseville, Igor~V Nikiforov, et~al.
\newblock {\em Detection of abrupt changes: theory and application}, volume
  104.
\newblock prentice Hall Englewood Cliffs, 1993.

\bibitem{beery2020iwildcam}
Sara Beery, Elijah Cole, and Arvi Gjoka.
\newblock The iwildcam 2020 competition dataset.
\newblock {\em arXiv preprint arXiv:2004.10340}, 2020.

\bibitem{ben2010theory}
Shai Ben-David, John Blitzer, Koby Crammer, Alex Kulesza, Fernando Pereira, and
  Jennifer~Wortman Vaughan.
\newblock A theory of learning from different domains.
\newblock {\em Machine learning}, 79:151--175, 2010.

\bibitem{ben2006analysis}
Shai Ben-David, John Blitzer, Koby Crammer, and Fernando Pereira.
\newblock Analysis of representations for domain adaptation.
\newblock {\em Advances in neural information processing systems}, 19, 2006.

\bibitem{bengio2013representation}
Yoshua Bengio, Aaron Courville, and Pascal Vincent.
\newblock Representation learning: A review and new perspectives.
\newblock {\em IEEE transactions on pattern analysis and machine intelligence},
  35(8):1798--1828, 2013.

\bibitem{biggio2013evasion}
Battista Biggio, Igino Corona, Davide Maiorca, Blaine Nelson, Nedim
  {\v{S}}rndi{\'c}, Pavel Laskov, Giorgio Giacinto, and Fabio Roli.
\newblock Evasion attacks against machine learning at test time.
\newblock In {\em Machine Learning and Knowledge Discovery in Databases:
  European Conference, ECML PKDD 2013, Prague, Czech Republic, September 23-27,
  2013, Proceedings, Part III 13}, pages 387--402. Springer, 2013.

\bibitem{bolley2007quantitative}
Fran{\c{c}}ois Bolley, Arnaud Guillin, and C{\'e}dric Villani.
\newblock Quantitative concentration inequalities for empirical measures on
  non-compact spaces.
\newblock {\em Probability Theory and Related Fields}, 137:541--593, 2007.

\bibitem{borkan2019nuanced}
Daniel Borkan, Lucas Dixon, Jeffrey Sorensen, Nithum Thain, and Lucy Vasserman.
\newblock Nuanced metrics for measuring unintended bias with real data for text
  classification.
\newblock In {\em Companion Proceedings of The 2019 World Wide Web Conference},
  2019.

\bibitem{bousquet2003introduction}
Olivier Bousquet, St{\'e}phane Boucheron, and G{\'a}bor Lugosi.
\newblock Introduction to statistical learning theory.
\newblock In {\em Summer School on Machine Learning}, pages 169--207. Springer,
  2003.

\bibitem{carlini2019evaluating}
Nicholas Carlini, Anish Athalye, Nicolas Papernot, Wieland Brendel, Jonas
  Rauber, Dimitris Tsipras, Ian Goodfellow, Aleksander Madry, and Alexey
  Kurakin.
\newblock On evaluating adversarial robustness.
\newblock {\em arXiv preprint arXiv:1902.06705}, 2019.

\bibitem{chollet2015keras}
Francois Chollet et~al.
\newblock Keras, 2015.

\bibitem{christie2018functional}
Gordon Christie, Neil Fendley, James Wilson, and Ryan Mukherjee.
\newblock Functional map of the world.
\newblock In {\em Proceedings of the IEEE Conference on Computer Vision and
  Pattern Recognition}, 2018.

\bibitem{cucurull2019context}
Guillem Cucurull, Perouz Taslakian, and David Vazquez.
\newblock Context-aware visual compatibility prediction.
\newblock In {\em Proceedings of the IEEE/CVF Conference on Computer Vision and
  Pattern Recognition}, pages 12617--12626, 2019.

\bibitem{daum2007frustratingly}
Hal Daum{\'e}~III.
\newblock Frustratingly easy domain adaptation.
\newblock In {\em Proc. of the 45th Annual Meeting of the Association of
  Computational Linguistics, 2007}, pages 256--263, 2007.

\bibitem{daume2006domain}
Hal Daum{\'e}~III and Daniel Marcu.
\newblock Domain adaptation for statistical classifiers.
\newblock {\em Journal of artificial Intelligence research}, 26:101--126, 2006.

\bibitem{david2020global}
Etienne David, Simon Madec, Pouria Sadeghi-Tehran, Helge Aasen, Bangyou Zheng,
  Shouyang Liu, Norbert Kirchgessner, Goro Ishikawa, Koichi Nagasawa,
  Minhajul~A Badhon, Curtis Pozniak, Benoit de Solan, Andreas Hund, Scott~C.
  Chapman, Frederic Baret, Ian Stavness, and Wei Guo.
\newblock Global wheat head detection (gwhd) dataset: a large and diverse
  dataset of high-resolution rgb-labelled images to develop and benchmark wheat
  head detection methods.
\newblock {\em Plant Phenomics}, 2020, 2020.

\bibitem{delage2010distributionally}
Erick Delage and Yinyu Ye.
\newblock Distributionally robust optimization under moment uncertainty with
  application to data-driven problems.
\newblock {\em Operations research}, 58(3):595--612, 2010.

\bibitem{deng2009imagenet}
Jia Deng, Wei Dong, Richard Socher, Li-Jia Li, Kai Li, and Li Fei-Fei.
\newblock Imagenet: A large-scale hierarchical image database.
\newblock In {\em 2009 IEEE conference on computer vision and pattern
  recognition}, pages 248--255. Ieee, 2009.

\bibitem{devlin2018bert}
Jacob Devlin, Ming-Wei Chang, Kenton Lee, and Kristina Toutanova.
\newblock Bert: Pre-training of deep bidirectional transformers for language
  understanding.
\newblock {\em arXiv preprint arXiv:1810.04805}, 2018.

\bibitem{donahue2013semi}
Jeff Donahue, Judy Hoffman, Erik Rodner, Kate Saenko, and Trevor Darrell.
\newblock Semi-supervised domain adaptation with instance constraints.
\newblock In {\em Proceedings of the IEEE conference on computer vision and
  pattern recognition}, pages 668--675, 2013.

\bibitem{elwell2011incremental}
Ryan Elwell and Robi Polikar.
\newblock Incremental learning of concept drift in nonstationary environments.
\newblock {\em IEEE Transactions on Neural Networks}, 22(10):1517--1531, 2011.

\bibitem{engstrom2019exploring}
Logan Engstrom, Brandon Tran, Dimitris Tsipras, Ludwig Schmidt, and Aleksander
  Madry.
\newblock Exploring the landscape of spatial robustness.
\newblock In {\em International conference on machine learning}, pages
  1802--1811. PMLR, 2019.

\bibitem{fang2013unbiased}
Chen Fang, Ye Xu, and Daniel~N Rockmore.
\newblock Unbiased metric learning: On the utilization of multiple datasets and
  web images for softening bias.
\newblock In {\em Proceedings of the IEEE International Conference on Computer
  Vision}, pages 1657--1664, 2013.

\bibitem{fawzi2015manitest}
Alhussein Fawzi and Pascal Frossard.
\newblock Manitest: Are classifiers really invariant?
\newblock {\em arXiv preprint arXiv:1507.06535}, 2015.

\bibitem{fernando2013unsupervised}
Basura Fernando, Amaury Habrard, Marc Sebban, and Tinne Tuytelaars.
\newblock Unsupervised visual domain adaptation using subspace alignment.
\newblock In {\em Proceedings of the IEEE international conference on computer
  vision}, pages 2960--2967, 2013.

\bibitem{frias2014online}
Isvani Frias-Blanco, Jos{\'e} del Campo-{\'A}vila, Gonzalo Ramos-Jimenez,
  Rafael Morales-Bueno, Agustin Ortiz-Diaz, and Yail{\'e} Caballero-Mota.
\newblock Online and non-parametric drift detection methods based on
  hoeffding’s bounds.
\newblock {\em IEEE Transactions on Knowledge and Data Engineering},
  27(3):810--823, 2014.

\bibitem{doi:10.1080/14786440109462720}
Karl~Pearson F.R.S.
\newblock Liii. on lines and planes of closest fit to systems of points in
  space.
\newblock {\em The London, Edinburgh, and Dublin Philosophical Magazine and
  Journal of Science}, 2(11):559--572, 1901.

\bibitem{gama2006learning}
Joao Gama and Gladys Castillo.
\newblock Learning with local drift detection.
\newblock In {\em Advanced Data Mining and Applications: Second International
  Conference, ADMA 2006, Xi’an, China, August 14-16, 2006 Proceedings 2},
  pages 42--55. Springer, 2006.

\bibitem{gama2004learning}
Joao Gama, Pedro Medas, Gladys Castillo, and Pedro Rodrigues.
\newblock Learning with drift detection.
\newblock In {\em Advances in Artificial Intelligence--SBIA 2004: 17th
  Brazilian Symposium on Artificial Intelligence, Sao Luis, Maranhao, Brazil,
  September 29-Ocotber 1, 2004. Proceedings 17}, pages 286--295. Springer,
  2004.

\bibitem{gama2003accurate}
Joao Gama, Ricardo Rocha, and Pedro Medas.
\newblock Accurate decision trees for mining high-speed data streams.
\newblock In {\em Proceedings of the ninth ACM SIGKDD international conference
  on Knowledge discovery and data mining}, pages 523--528, 2003.

\bibitem{gama2014survey}
Jo{\~a}o Gama, Indr{\.e} {\v{Z}}liobait{\.e}, Albert Bifet, Mykola Pechenizkiy,
  and Abdelhamid Bouchachia.
\newblock A survey on concept drift adaptation.
\newblock {\em ACM computing surveys (CSUR)}, 46(4):1--37, 2014.

\bibitem{ganin2015unsupervised}
Yaroslav Ganin and Victor Lempitsky.
\newblock Unsupervised domain adaptation by backpropagation.
\newblock In {\em International conference on machine learning}, pages
  1180--1189. PMLR, 2015.

\bibitem{gebru2021datasheets}
Timnit Gebru, Jamie Morgenstern, Briana Vecchione, Jennifer~Wortman Vaughan,
  Hanna Wallach, Hal~Daum{\'e} Iii, and Kate Crawford.
\newblock Datasheets for datasets.
\newblock {\em Communications of the ACM}, 64(12):86--92, 2021.

\bibitem{geirhos2018generalisation}
Robert Geirhos, Carlos~RM Temme, Jonas Rauber, Heiko~H Sch{\"u}tt, Matthias
  Bethge, and Felix~A Wichmann.
\newblock Generalisation in humans and deep neural networks.
\newblock {\em Advances in neural information processing systems}, 31, 2018.

\bibitem{goh2010distributionally}
Joel Goh and Melvyn Sim.
\newblock Distributionally robust optimization and its tractable
  approximations.
\newblock {\em Operations research}, 58(4-part-1):902--917, 2010.

\bibitem{gomes2017adaptive}
Heitor~M Gomes, Albert Bifet, Jesse Read, Jean~Paul Barddal, Fabr{\'\i}cio
  Enembreck, Bernhard Pfharinger, Geoff Holmes, and Talel Abdessalem.
\newblock Adaptive random forests for evolving data stream classification.
\newblock {\em Machine Learning}, 106:1469--1495, 2017.

\bibitem{goodfellow2014explaining}
Ian~J Goodfellow, Jonathon Shlens, and Christian Szegedy.
\newblock Explaining and harnessing adversarial examples.
\newblock {\em arXiv preprint arXiv:1412.6572}, 2014.

\bibitem{gretton2009covariate}
Arthur Gretton, Alex Smola, Jiayuan Huang, Marcel Schmittfull, Karsten
  Borgwardt, and Bernhard Sch{\"o}lkopf.
\newblock Covariate shift by kernel mean matching.
\newblock {\em Dataset shift in machine learning}, 3(4):5, 2009.

\bibitem{heinze2018invariant}
Christina Heinze-Deml, Jonas Peters, and Nicolai Meinshausen.
\newblock Invariant causal prediction for nonlinear models.
\newblock {\em Journal of Causal Inference}, 6(2), 2018.

\bibitem{hendrycks2021many}
Dan Hendrycks, Steven Basart, Norman Mu, Saurav Kadavath, Frank Wang, Evan
  Dorundo, Rahul Desai, Tyler Zhu, Samyak Parajuli, Mike Guo, et~al.
\newblock The many faces of robustness: A critical analysis of
  out-of-distribution generalization.
\newblock In {\em Proceedings of the IEEE/CVF International Conference on
  Computer Vision}, pages 8340--8349, 2021.

\bibitem{hendrycks2019benchmarking}
Dan Hendrycks and Thomas Dietterich.
\newblock Benchmarking neural network robustness to common corruptions and
  perturbations.
\newblock {\em arXiv preprint arXiv:1903.12261}, 2019.

\bibitem{hendrycks2021natural}
Dan Hendrycks, Kevin Zhao, Steven Basart, Jacob Steinhardt, and Dawn Song.
\newblock Natural adversarial examples.
\newblock In {\em Proceedings of the IEEE/CVF Conference on Computer Vision and
  Pattern Recognition}, pages 15262--15271, 2021.

\bibitem{higgins2017beta}
Irina Higgins, Loic Matthey, Arka Pal, Christopher Burgess, Xavier Glorot,
  Matthew Botvinick, Shakir Mohamed, and Alexander Lerchner.
\newblock beta-vae: Learning basic visual concepts with a constrained
  variational framework.
\newblock In {\em International conference on learning representations}, 2017.

\bibitem{hu2020open}
Weihua Hu, Matthias Fey, Marinka Zitnik, Yuxiao Dong, Hongyu Ren, Bowen Liu,
  Michele Catasta, and Jure Leskovec.
\newblock Open graph benchmark: Datasets for machine learning on graphs.
\newblock In {\em Advances in Neural Information Processing Systems (NeurIPS)},
  2020.

\bibitem{huang2006correcting}
Jiayuan Huang, Arthur Gretton, Karsten Borgwardt, Bernhard Sch{\"o}lkopf, and
  Alex Smola.
\newblock Correcting sample selection bias by unlabeled data.
\newblock {\em Advances in neural information processing systems}, 19, 2006.

\bibitem{hulten2001mining}
Geoff Hulten, Laurie Spencer, and Pedro Domingos.
\newblock Mining time-changing data streams.
\newblock In {\em Proceedings of the seventh ACM SIGKDD international
  conference on Knowledge discovery and data mining}, pages 97--106, 2001.

\bibitem{jamal2020rethinking}
Muhammad~Abdullah Jamal, Matthew Brown, Ming-Hsuan Yang, Liqiang Wang, and
  Boqing Gong.
\newblock Rethinking class-balanced methods for long-tailed visual recognition
  from a domain adaptation perspective.
\newblock In {\em Proceedings of the IEEE/CVF Conference on Computer Vision and
  Pattern Recognition}, pages 7610--7619, 2020.

\bibitem{jia2020fashionpedia}
Menglin Jia, Mengyun Shi, Mikhail Sirotenko, Yin Cui, Claire Cardie, Bharath
  Hariharan, Hartwig Adam, and Serge Belongie.
\newblock Fashionpedia: Ontology, segmentation, and an attribute localization
  dataset.
\newblock In {\em Computer Vision--ECCV 2020: 16th European Conference,
  Glasgow, UK, August 23--28, 2020, Proceedings, Part I 16}, pages 316--332.
  Springer, 2020.

\bibitem{kanamori2009least}
Takafumi Kanamori, Shohei Hido, and Masashi Sugiyama.
\newblock A least-squares approach to direct importance estimation.
\newblock {\em The Journal of Machine Learning Research}, 10:1391--1445, 2009.

\bibitem{kang2019contrastive}
Guoliang Kang, Lu Jiang, Yi Yang, and Alexander~G Hauptmann.
\newblock Contrastive adaptation network for unsupervised domain adaptation.
\newblock In {\em Proceedings of the IEEE/CVF conference on computer vision and
  pattern recognition}, pages 4893--4902, 2019.

\bibitem{kemker2018measuring}
Ronald Kemker, Marc McClure, Angelina Abitino, Tyler Hayes, and Christopher
  Kanan.
\newblock Measuring catastrophic forgetting in neural networks.
\newblock In {\em Proceedings of the AAAI conference on artificial
  intelligence}, volume~32, 2018.

\bibitem{khan2022transformers}
Salman Khan, Muzammal Naseer, Munawar Hayat, Syed~Waqas Zamir, Fahad~Shahbaz
  Khan, and Mubarak Shah.
\newblock Transformers in vision: A survey.
\newblock {\em ACM computing surveys (CSUR)}, 54(10s):1--41, 2022.

\bibitem{kiefer1952stochastic}
Jack Kiefer and Jacob Wolfowitz.
\newblock Stochastic estimation of the maximum of a regression function.
\newblock {\em The Annals of Mathematical Statistics}, pages 462--466, 1952.

\bibitem{kim2018disentangling}
Hyunjik Kim and Andriy Mnih.
\newblock Disentangling by factorising.
\newblock In {\em International Conference on Machine Learning}, pages
  2649--2658. PMLR, 2018.

\bibitem{kim2021setvae}
Jinwoo Kim, Jaehoon Yoo, Juho Lee, and Seunghoon Hong.
\newblock Setvae: Learning hierarchical composition for generative modeling of
  set-structured data.
\newblock In {\em Proceedings of the IEEE/CVF Conference on Computer Vision and
  Pattern Recognition}, pages 15059--15068, 2021.

\bibitem{kimura2023generalization}
Masanari Kimura.
\newblock Generalization bounds for set-to-set matching with negative sampling.
\newblock {\em arXiv preprint arXiv:2302.12991}, 2023.

\bibitem{kimura2022information}
Masanari Kimura and Hideitsu Hino.
\newblock Information geometrically generalized covariate shift adaptation.
\newblock {\em Neural Computation}, 34(9):1944--1977, 2022.

\bibitem{kimura2020batch}
Masanari Kimura, Kei Wakabayashi, and Atsuyuki Morishima.
\newblock Batch prioritization of data labeling tasks for training classifiers.
\newblock In {\em Proceedings of the AAAI Conference on Human Computation and
  Crowdsourcing}, volume~8, pages 163--167, 2020.

\bibitem{kirkpatrick2017overcoming}
James Kirkpatrick, Razvan Pascanu, Neil Rabinowitz, Joel Veness, Guillaume
  Desjardins, Andrei~A Rusu, Kieran Milan, John Quan, Tiago Ramalho, Agnieszka
  Grabska-Barwinska, et~al.
\newblock Overcoming catastrophic forgetting in neural networks.
\newblock {\em Proceedings of the national academy of sciences},
  114(13):3521--3526, 2017.

\bibitem{koh2021wilds}
Pang~Wei Koh, Shiori Sagawa, Henrik Marklund, Sang~Michael Xie, Marvin Zhang,
  Akshay Balsubramani, Weihua Hu, Michihiro Yasunaga, Richard~Lanas Phillips,
  Irena Gao, et~al.
\newblock Wilds: A benchmark of in-the-wild distribution shifts.
\newblock In {\em International Conference on Machine Learning}, pages
  5637--5664. PMLR, 2021.

\bibitem{kolter2007dynamic}
J~Zico Kolter and Marcus~A Maloof.
\newblock Dynamic weighted majority: An ensemble method for drifting concepts.
\newblock {\em The Journal of Machine Learning Research}, 8:2755--2790, 2007.

\bibitem{kumar2022calibrated}
Ananya Kumar, Tengyu Ma, Percy Liang, and Aditi Raghunathan.
\newblock Calibrated ensembles can mitigate accuracy tradeoffs under
  distribution shift.
\newblock In {\em Uncertainty in Artificial Intelligence}, pages 1041--1051.
  PMLR, 2022.

\bibitem{kumar2010co}
Abhishek Kumar, Avishek Saha, and Hal Daume.
\newblock Co-regularization based semi-supervised domain adaptation.
\newblock {\em Advances in neural information processing systems}, 23, 2010.

\bibitem{kurakin2018adversarial}
Alexey Kurakin, Ian~J Goodfellow, and Samy Bengio.
\newblock Adversarial examples in the physical world.
\newblock In {\em Artificial intelligence safety and security}, pages 99--112.
  Chapman and Hall/CRC, 2018.

\bibitem{lee2019set}
Juho Lee, Yoonho Lee, Jungtaek Kim, Adam Kosiorek, Seungjin Choi, and Yee~Whye
  Teh.
\newblock Set transformer: A framework for attention-based
  permutation-invariant neural networks.
\newblock In {\em International Conference on Machine Learning}, pages
  3744--3753. PMLR, 2019.

\bibitem{li2017deeper}
Da Li, Yongxin Yang, Yi-Zhe Song, and Timothy~M Hospedales.
\newblock Deeper, broader and artier domain generalization.
\newblock In {\em Proceedings of the IEEE international conference on computer
  vision}, pages 5542--5550, 2017.

\bibitem{li2018domain}
Haoliang Li, Sinno~Jialin Pan, Shiqi Wang, and Alex~C Kot.
\newblock Domain generalization with adversarial feature learning.
\newblock In {\em Proceedings of the IEEE conference on computer vision and
  pattern recognition}, pages 5400--5409, 2018.

\bibitem{li2022out}
Haoyang Li, Xin Wang, Ziwei Zhang, and Wenwu Zhu.
\newblock Out-of-distribution generalization on graphs: A survey.
\newblock {\em arXiv preprint arXiv:2202.07987}, 2022.

\bibitem{li2020model}
Rui Li, Qianfen Jiao, Wenming Cao, Hau-San Wong, and Si Wu.
\newblock Model adaptation: Unsupervised domain adaptation without source data.
\newblock In {\em Proceedings of the IEEE/CVF conference on computer vision and
  pattern recognition}, pages 9641--9650, 2020.

\bibitem{liu2017regional}
Anjin Liu, Yiliao Song, Guangquan Zhang, and Jie Lu.
\newblock Regional concept drift detection and density synchronized drift
  adaptation.
\newblock In {\em IJCAI International Joint Conference on Artificial
  Intelligence}, 2017.

\bibitem{liu2017fuzzy}
Anjin Liu, Guangquan Zhang, and Jie Lu.
\newblock Fuzzy time windowing for gradual concept drift adaptation.
\newblock In {\em 2017 IEEE International Conference on Fuzzy Systems
  (FUZZ-IEEE)}, pages 1--6. IEEE, 2017.

\bibitem{liuLQWTcvpr16DeepFashion}
Ziwei Liu, Ping Luo, Shi Qiu, Xiaogang Wang, and Xiaoou Tang.
\newblock Deepfashion: Powering robust clothes recognition and retrieval with
  rich annotations.
\newblock In {\em Proceedings of IEEE Conference on Computer Vision and Pattern
  Recognition (CVPR)}, June 2016.

\bibitem{long2016unsupervised}
Mingsheng Long, Han Zhu, Jianmin Wang, and Michael~I Jordan.
\newblock Unsupervised domain adaptation with residual transfer networks.
\newblock {\em Advances in neural information processing systems}, 29, 2016.

\bibitem{lu2018learning}
Jie Lu, Anjin Liu, Fan Dong, Feng Gu, Joao Gama, and Guangquan Zhang.
\newblock Learning under concept drift: A review.
\newblock {\em IEEE transactions on knowledge and data engineering},
  31(12):2346--2363, 2018.

\bibitem{lu2016concept}
Ning Lu, Jie Lu, Guangquan Zhang, and Ramon~Lopez De~Mantaras.
\newblock A concept drift-tolerant case-base editing technique.
\newblock {\em Artificial Intelligence}, 230:108--133, 2016.

\bibitem{lu2014concept}
Ning Lu, Guangquan Zhang, and Jie Lu.
\newblock Concept drift detection via competence models.
\newblock {\em Artificial Intelligence}, 209:11--28, 2014.

\bibitem{lu2021codexglue}
Shuai Lu, Daya Guo, Shuo Ren, Junjie Huang, Alexey Svyatkovskiy, Ambrosio
  Blanco, Colin Clement, Dawn Drain, Daxin Jiang, Duyu Tang, et~al.
\newblock Codexglue: A machine learning benchmark dataset for code
  understanding and generation.
\newblock {\em arXiv preprint arXiv:2102.04664}, 2021.

\bibitem{manly2000cumulative}
Bryan~FJ Manly and Darryl Mackenzie.
\newblock A cumulative sum type of method for environmental monitoring.
\newblock {\em Environmetrics: The official journal of the International
  Environmetrics Society}, 11(2):151--166, 2000.

\bibitem{miller2021accuracy}
John~P Miller, Rohan Taori, Aditi Raghunathan, Shiori Sagawa, Pang~Wei Koh,
  Vaishaal Shankar, Percy Liang, Yair Carmon, and Ludwig Schmidt.
\newblock Accuracy on the line: on the strong correlation between
  out-of-distribution and in-distribution generalization.
\newblock In {\em International Conference on Machine Learning}, pages
  7721--7735. PMLR, 2021.

\bibitem{moreno2012unifying}
Jose~G Moreno-Torres, Troy Raeder, Roc{\'\i}o Alaiz-Rodr{\'\i}guez, Nitesh~V
  Chawla, and Francisco Herrera.
\newblock A unifying view on dataset shift in classification.
\newblock {\em Pattern recognition}, 45(1):521--530, 2012.

\bibitem{motiian2017unified}
Saeid Motiian, Marco Piccirilli, Donald~A Adjeroh, and Gianfranco Doretto.
\newblock Unified deep supervised domain adaptation and generalization.
\newblock In {\em Proceedings of the IEEE international conference on computer
  vision}, pages 5715--5725, 2017.

\bibitem{mu2019mnist}
Norman Mu and Justin Gilmer.
\newblock Mnist-c: A robustness benchmark for computer vision.
\newblock {\em arXiv preprint arXiv:1906.02337}, 2019.

\bibitem{newell2016stacked}
Alejandro Newell, Kaiyu Yang, and Jia Deng.
\newblock Stacked hourglass networks for human pose estimation.
\newblock In {\em European conference on computer vision}, pages 483--499.
  Springer, 2016.

\bibitem{nguyen2010estimating}
XuanLong Nguyen, Martin~J Wainwright, and Michael~I Jordan.
\newblock Estimating divergence functionals and the likelihood ratio by convex
  risk minimization.
\newblock {\em IEEE Transactions on Information Theory}, 56(11):5847--5861,
  2010.

\bibitem{ni2019justifying}
Jianmo Ni, Jiacheng Li, and Julian McAuley.
\newblock Justifying recommendations using distantly-labeled reviews and
  fine-grained aspects.
\newblock In {\em Proceedings of the 2019 Conference on Empirical Methods in
  Natural Language Processing and the 9th International Joint Conference on
  Natural Language Processing (EMNLP-IJCNLP)}, 2019.

\bibitem{nishida2007detecting}
Kyosuke Nishida and Koichiro Yamauchi.
\newblock Detecting concept drift using statistical testing.
\newblock In {\em Discovery science}, volume 4755, pages 264--269. Springer,
  2007.

\bibitem{ovadia2019can}
Yaniv Ovadia, Emily Fertig, Jie Ren, Zachary Nado, David Sculley, Sebastian
  Nowozin, Joshua Dillon, Balaji Lakshminarayanan, and Jasper Snoek.
\newblock Can you trust your model's uncertainty? evaluating predictive
  uncertainty under dataset shift.
\newblock {\em Advances in neural information processing systems}, 32, 2019.

\bibitem{parisi2019continual}
German~I Parisi, Ronald Kemker, Jose~L Part, Christopher Kanan, and Stefan
  Wermter.
\newblock Continual lifelong learning with neural networks: A review.
\newblock {\em Neural networks}, 113:54--71, 2019.

\bibitem{NEURIPS2019_9015}
Adam Paszke, Sam Gross, Francisco Massa, Adam Lerer, James Bradbury, Gregory
  Chanan, Trevor Killeen, Zeming Lin, Natalia Gimelshein, Luca Antiga, Alban
  Desmaison, Andreas Kopf, Edward Yang, Zachary DeVito, Martin Raison, Alykhan
  Tejani, Sasank Chilamkurthy, Benoit Steiner, Lu Fang, Junjie Bai, and Soumith
  Chintala.
\newblock Pytorch: An imperative style, high-performance deep learning library.
\newblock In {\em Advances in Neural Information Processing Systems 32}, pages
  8024--8035. Curran Associates, Inc., 2019.

\bibitem{patel2015visual}
Vishal~M Patel, Raghuraman Gopalan, Ruonan Li, and Rama Chellappa.
\newblock Visual domain adaptation: A survey of recent advances.
\newblock {\em IEEE signal processing magazine}, 32(3):53--69, 2015.

\bibitem{peng2019moment}
Xingchao Peng, Qinxun Bai, Xide Xia, Zijun Huang, Kate Saenko, and Bo Wang.
\newblock Moment matching for multi-source domain adaptation.
\newblock In {\em Proceedings of the IEEE/CVF international conference on
  computer vision}, pages 1406--1415, 2019.

\bibitem{pfister2019invariant}
Niklas Pfister, Peter B{\"u}hlmann, and Jonas Peters.
\newblock Invariant causal prediction for sequential data.
\newblock {\em Journal of the American Statistical Association},
  114(527):1264--1276, 2019.

\bibitem{prince2004does}
Michael Prince.
\newblock Does active learning work? a review of the research.
\newblock {\em Journal of engineering education}, 93(3):223--231, 2004.

\bibitem{quinonero2008dataset}
Joaquin Quinonero-Candela, Masashi Sugiyama, Anton Schwaighofer, and Neil~D
  Lawrence.
\newblock {\em Dataset shift in machine learning}.
\newblock Mit Press, 2008.

\bibitem{rahimian2019distributionally}
Hamed Rahimian and Sanjay Mehrotra.
\newblock Distributionally robust optimization: A review.
\newblock {\em arXiv preprint arXiv:1908.05659}, 2019.

\bibitem{redko2017theoretical}
Ievgen Redko, Amaury Habrard, and Marc Sebban.
\newblock Theoretical analysis of domain adaptation with optimal transport.
\newblock In {\em Machine Learning and Knowledge Discovery in Databases:
  European Conference, ECML PKDD 2017, Skopje, Macedonia, September 18--22,
  2017, Proceedings, Part II 10}, pages 737--753. Springer, 2017.

\bibitem{redko2019advances}
Ievgen Redko, Emilie Morvant, Amaury Habrard, Marc Sebban, and Younes Bennani.
\newblock {\em Advances in domain adaptation theory}.
\newblock Elsevier, 2019.

\bibitem{robbins1951stochastic}
Herbert Robbins and Sutton Monro.
\newblock A stochastic approximation method.
\newblock {\em The annals of mathematical statistics}, pages 400--407, 1951.

\bibitem{rothenhausler2018anchor}
Dominik Rothenh{\"a}usler, Nicolai Meinshausen, Peter B{\"u}hlmann, and Jonas
  Peters.
\newblock Anchor regression: heterogeneous data meets causality.
\newblock {\em arXiv preprint arXiv:1801.06229}, 2018.

\bibitem{rusak2020simple}
Evgenia Rusak, Lukas Schott, Roland~S Zimmermann, Julian Bitterwolf, Oliver
  Bringmann, Matthias Bethge, and Wieland Brendel.
\newblock A simple way to make neural networks robust against diverse image
  corruptions.
\newblock In {\em Computer Vision--ECCV 2020: 16th European Conference,
  Glasgow, UK, August 23--28, 2020, Proceedings, Part III 16}, pages 53--69.
  Springer, 2020.

\bibitem{saito2019semi}
Kuniaki Saito, Donghyun Kim, Stan Sclaroff, Trevor Darrell, and Kate Saenko.
\newblock Semi-supervised domain adaptation via minimax entropy.
\newblock In {\em Proceedings of the IEEE/CVF international conference on
  computer vision}, pages 8050--8058, 2019.

\bibitem{saito2020exchangeable}
Yuki Saito, Takuma Nakamura, Hirotaka Hachiya, and Kenji Fukumizu.
\newblock Exchangeable deep neural networks for set-to-set matching and
  learning.
\newblock In {\em Computer Vision--ECCV 2020: 16th European Conference,
  Glasgow, UK, August 23--28, 2020, Proceedings, Part XVII}, pages 626--646.
  Springer, 2020.

\bibitem{saunders2022domain}
Danielle Saunders.
\newblock Domain adaptation and multi-domain adaptation for neural machine
  translation: A survey.
\newblock {\em Journal of Artificial Intelligence Research}, 75:351--424, 2022.

\bibitem{scholkopf2021toward}
Bernhard Sch{\"o}lkopf, Francesco Locatello, Stefan Bauer, Nan~Rosemary Ke, Nal
  Kalchbrenner, Anirudh Goyal, and Yoshua Bengio.
\newblock Toward causal representation learning.
\newblock {\em Proceedings of the IEEE}, 109(5):612--634, 2021.

\bibitem{sener2016learning}
Ozan Sener, Hyun~Oh Song, Ashutosh Saxena, and Silvio Savarese.
\newblock Learning transferrable representations for unsupervised domain
  adaptation.
\newblock {\em Advances in neural information processing systems}, 29, 2016.

\bibitem{settles2009active}
Burr Settles.
\newblock Active learning literature survey.
\newblock 2009.

\bibitem{shao2014prototype}
Junming Shao, Zahra Ahmadi, and Stefan Kramer.
\newblock Prototype-based learning on concept-drifting data streams.
\newblock In {\em Proceedings of the 20th ACM SIGKDD international conference
  on Knowledge discovery and data mining}, pages 412--421, 2014.

\bibitem{shen2020stable}
Zheyan Shen, Peng Cui, Jiashuo Liu, Tong Zhang, Bo Li, and Zhitang Chen.
\newblock Stable learning via differentiated variable decorrelation.
\newblock In {\em Proceedings of the 26th acm sigkdd international conference
  on knowledge discovery \& data mining}, pages 2185--2193, 2020.

\bibitem{shen2021towards}
Zheyan Shen, Jiashuo Liu, Yue He, Xingxuan Zhang, Renzhe Xu, Han Yu, and Peng
  Cui.
\newblock Towards out-of-distribution generalization: A survey.
\newblock {\em arXiv preprint arXiv:2108.13624}, 2021.

\bibitem{shimodaira2000improving}
Hidetoshi Shimodaira.
\newblock Improving predictive inference under covariate shift by weighting the
  log-likelihood function.
\newblock {\em Journal of statistical planning and inference}, 90(2):227--244,
  2000.

\bibitem{soelch2019deep}
Maximilian Soelch, Adnan Akhundov, Patrick van~der Smagt, and Justin Bayer.
\newblock On deep set learning and the choice of aggregations.
\newblock In {\em Artificial Neural Networks and Machine Learning--ICANN 2019:
  Theoretical Neural Computation: 28th International Conference on Artificial
  Neural Networks, Munich, Germany, September 17--19, 2019, Proceedings, Part I
  28}, pages 444--457. Springer, 2019.

\bibitem{song2007statistical}
Xiuyao Song, Mingxi Wu, Christopher Jermaine, and Sanjay Ranka.
\newblock Statistical change detection for multi-dimensional data.
\newblock In {\em Proceedings of the 13th ACM SIGKDD international conference
  on Knowledge discovery and data mining}, pages 667--676, 2007.

\bibitem{subbaswamy2021evaluating}
Adarsh Subbaswamy, Roy Adams, and Suchi Saria.
\newblock Evaluating model robustness and stability to dataset shift.
\newblock In {\em International Conference on Artificial Intelligence and
  Statistics}, pages 2611--2619. PMLR, 2021.

\bibitem{sugiyama2007direct}
Masashi Sugiyama, Shinichi Nakajima, Hisashi Kashima, Paul Buenau, and Motoaki
  Kawanabe.
\newblock Direct importance estimation with model selection and its application
  to covariate shift adaptation.
\newblock {\em Advances in neural information processing systems}, 20, 2007.

\bibitem{szegedy2013intriguing}
Christian Szegedy, Wojciech Zaremba, Ilya Sutskever, Joan Bruna, Dumitru Erhan,
  Ian Goodfellow, and Rob Fergus.
\newblock Intriguing properties of neural networks.
\newblock {\em arXiv preprint arXiv:1312.6199}, 2013.

\bibitem{taori2020measuring}
Rohan Taori, Achal Dave, Vaishaal Shankar, Nicholas Carlini, Benjamin Recht,
  and Ludwig Schmidt.
\newblock Measuring robustness to natural distribution shifts in image
  classification.
\newblock {\em Advances in Neural Information Processing Systems},
  33:18583--18599, 2020.

\bibitem{tay2022efficient}
Yi Tay, Mostafa Dehghani, Dara Bahri, and Donald Metzler.
\newblock Efficient transformers: A survey.
\newblock {\em ACM Computing Surveys}, 55(6):1--28, 2022.

\bibitem{taylor2019rxrx1}
J. Taylor, B. Earnshaw, B. Mabey, M. Victors, and J. Yosinski.
\newblock Rxrx1: An image set for cellular morphological variation across many
  experimental batches.
\newblock In {\em International Conference on Learning Representations (ICLR)},
  2019.

\bibitem{torrey2010transfer}
Lisa Torrey and Jude Shavlik.
\newblock Transfer learning.
\newblock In {\em Handbook of research on machine learning applications and
  trends: algorithms, methods, and techniques}, pages 242--264. IGI global,
  2010.

\bibitem{tsymbal2004problem}
Alexey Tsymbal.
\newblock The problem of concept drift: definitions and related work.
\newblock {\em Computer Science Department, Trinity College Dublin}, 106(2):58,
  2004.

\bibitem{tzeng2015simultaneous}
Eric Tzeng, Judy Hoffman, Trevor Darrell, and Kate Saenko.
\newblock Simultaneous deep transfer across domains and tasks.
\newblock In {\em Proceedings of the IEEE international conference on computer
  vision}, pages 4068--4076, 2015.

\bibitem{van2008visualizing}
Laurens Van~der Maaten and Geoffrey Hinton.
\newblock Visualizing data using t-sne.
\newblock {\em Journal of machine learning research}, 9(11), 2008.

\bibitem{vapnik2013nature}
Vladimir Vapnik.
\newblock {\em The nature of statistical learning theory}.
\newblock Springer science \& business media, 2013.

\bibitem{vapnik1999overview}
Vladimir~N Vapnik.
\newblock An overview of statistical learning theory.
\newblock {\em IEEE transactions on neural networks}, 10(5):988--999, 1999.

\bibitem{vasileva2018learning}
Mariya~I Vasileva, Bryan~A Plummer, Krishna Dusad, Shreya Rajpal, Ranjitha
  Kumar, and David Forsyth.
\newblock Learning type-aware embeddings for fashion compatibility.
\newblock In {\em Proceedings of the European conference on computer vision
  (ECCV)}, pages 390--405, 2018.

\bibitem{vaswani2017attention}
Ashish Vaswani, Noam Shazeer, Niki Parmar, Jakob Uszkoreit, Llion Jones,
  Aidan~N Gomez, {\L}ukasz Kaiser, and Illia Polosukhin.
\newblock Attention is all you need.
\newblock In {\em Advances in neural information processing systems}, pages
  5998--6008, 2017.

\bibitem{venkateswara2017deep}
Hemanth Venkateswara, Jose Eusebio, Shayok Chakraborty, and Sethuraman
  Panchanathan.
\newblock Deep hashing network for unsupervised domain adaptation.
\newblock In {\em Proceedings of the IEEE conference on computer vision and
  pattern recognition}, pages 5018--5027, 2017.

\bibitem{wagstaff2019limitations}
Edward Wagstaff, Fabian Fuchs, Martin Engelcke, Ingmar Posner, and Michael~A
  Osborne.
\newblock On the limitations of representing functions on sets.
\newblock In {\em International Conference on Machine Learning}, pages
  6487--6494. PMLR, 2019.

\bibitem{wagstaff2022universal}
Edward Wagstaff, Fabian~B Fuchs, Martin Engelcke, Michael~A Osborne, and Ingmar
  Posner.
\newblock Universal approximation of functions on sets.
\newblock {\em Journal of Machine Learning Research}, 23(151):1--56, 2022.

\bibitem{wald2021calibration}
Yoav Wald, Amir Feder, Daniel Greenfeld, and Uri Shalit.
\newblock On calibration and out-of-domain generalization.
\newblock {\em Advances in neural information processing systems},
  34:2215--2227, 2021.

\bibitem{wang2015concept}
Heng Wang and Zubin Abraham.
\newblock Concept drift detection for streaming data.
\newblock In {\em 2015 international joint conference on neural networks
  (IJCNN)}, pages 1--9. IEEE, 2015.

\bibitem{wang2018deep}
Mei Wang and Weihong Deng.
\newblock Deep visual domain adaptation: A survey.
\newblock {\em Neurocomputing}, 312:135--153, 2018.

\bibitem{weiss2016survey}
Karl Weiss, Taghi~M Khoshgoftaar, and DingDing Wang.
\newblock A survey of transfer learning.
\newblock {\em Journal of Big data}, 3(1):1--40, 2016.

\bibitem{guo2019fashion}
Hui Wu, Yupeng Gao, Xiaoxiao Guo, Ziad Al-Halah, Steven Rennie, Kristen
  Grauman, and Rogerio Feris.
\newblock The fashion iq dataset: Retrieving images by combining side
  information and relative natural language feedback.
\newblock {\em CVPR}, 2021.

\bibitem{xiao2018generating}
Chaowei Xiao, Bo Li, Jun-Yan Zhu, Warren He, Mingyan Liu, and Dawn Song.
\newblock Generating adversarial examples with adversarial networks.
\newblock {\em arXiv preprint arXiv:1801.02610}, 2018.

\bibitem{xiao2017/online}
Han Xiao, Kashif Rasul, and Roland Vollgraf.
\newblock Fashion-mnist: a novel image dataset for benchmarking machine
  learning algorithms, 2017.

\bibitem{xu2020robust}
Zhenlin Xu, Deyi Liu, Junlin Yang, Colin Raffel, and Marc Niethammer.
\newblock Robust and generalizable visual representation learning via random
  convolutions.
\newblock {\em arXiv preprint arXiv:2007.13003}, 2020.

\bibitem{yamada2013relative}
Makoto Yamada, Taiji Suzuki, Takafumi Kanamori, Hirotaka Hachiya, and Masashi
  Sugiyama.
\newblock Relative density-ratio estimation for robust distribution comparison.
\newblock {\em Neural computation}, 25(5):1324--1370, 2013.

\bibitem{yang2012incrementally}
Hang Yang and Simon Fong.
\newblock Incrementally optimized decision tree for noisy big data.
\newblock In {\em Proceedings of the 1st International Workshop on Big Data,
  Streams and Heterogeneous Source Mining: Algorithms, Systems, Programming
  Models and Applications}, pages 36--44, 2012.

\bibitem{yang2021generalized}
Jingkang Yang, Kaiyang Zhou, Yixuan Li, and Ziwei Liu.
\newblock Generalized out-of-distribution detection: A survey.
\newblock {\em arXiv preprint arXiv:2110.11334}, 2021.

\bibitem{yang2020curriculum}
Luyu Yang, Yogesh Balaji, Ser-Nam Lim, and Abhinav Shrivastava.
\newblock Curriculum manager for source selection in multi-source domain
  adaptation.
\newblock In {\em Computer Vision--ECCV 2020: 16th European Conference,
  Glasgow, UK, August 23--28, 2020, Proceedings, Part XIV 16}, pages 608--624.
  Springer, 2020.

\bibitem{yang2021causalvae}
Mengyue Yang, Furui Liu, Zhitang Chen, Xinwei Shen, Jianye Hao, and Jun Wang.
\newblock Causalvae: Disentangled representation learning via neural structural
  causal models.
\newblock In {\em Proceedings of the IEEE/CVF conference on computer vision and
  pattern recognition}, pages 9593--9602, 2021.

\bibitem{yao2015semi}
Ting Yao, Yingwei Pan, Chong-Wah Ngo, Houqiang Li, and Tao Mei.
\newblock Semi-supervised domain adaptation with subspace learning for visual
  recognition.
\newblock In {\em Proceedings of the IEEE conference on Computer Vision and
  Pattern Recognition}, pages 2142--2150, 2015.

\bibitem{ye2021towards}
Haotian Ye, Chuanlong Xie, Tianle Cai, Ruichen Li, Zhenguo Li, and Liwei Wang.
\newblock Towards a theoretical framework of out-of-distribution
  generalization.
\newblock {\em Advances in Neural Information Processing Systems},
  34:23519--23531, 2021.

\bibitem{yeh2020using}
Christopher Yeh, Anthony Perez, Anne Driscoll, George Azzari, Zhongyi Tang,
  David Lobell, Stefano Ermon, and Marshall Burke.
\newblock Using publicly available satellite imagery and deep learning to
  understand economic well-being in africa.
\newblock {\em Nature Communications}, 2020.

\bibitem{yu2018hard}
Rui Yu, Zhiyong Dou, Song Bai, Zhaoxiang Zhang, Yongchao Xu, and Xiang Bai.
\newblock Hard-aware point-to-set deep metric for person re-identification.
\newblock In {\em Proceedings of the European conference on computer vision
  (ECCV)}, pages 188--204, 2018.

\bibitem{zaheer2017deep}
Manzil Zaheer, Satwik Kottur, Siamak Ravanbakhsh, Barnabas Poczos, Russ~R
  Salakhutdinov, and Alexander~J Smola.
\newblock Deep sets.
\newblock {\em Advances in neural information processing systems}, 30, 2017.

\bibitem{zhang2013domain}
Kun Zhang, Bernhard Sch{\"o}lkopf, Krikamol Muandet, and Zhikun Wang.
\newblock Domain adaptation under target and conditional shift.
\newblock In {\em International Conference on Machine Learning}, pages
  819--827. PMLR, 2013.

\bibitem{zhang2021deep}
Xingxuan Zhang, Peng Cui, Renzhe Xu, Linjun Zhou, Yue He, and Zheyan Shen.
\newblock Deep stable learning for out-of-distribution generalization.
\newblock In {\em Proceedings of the IEEE/CVF Conference on Computer Vision and
  Pattern Recognition}, pages 5372--5382, 2021.

\bibitem{zhang2017three}
Yuhong Zhang, Guang Chu, Peipei Li, Xuegang Hu, and Xindong Wu.
\newblock Three-layer concept drifting detection in text data streams.
\newblock {\em Neurocomputing}, 260:393--403, 2017.

\bibitem{zhang2019deep}
Yan Zhang, Jonathon Hare, and Adam Prugel-Bennett.
\newblock Deep set prediction networks.
\newblock {\em Advances in Neural Information Processing Systems}, 32, 2019.

\bibitem{zhao2019multi}
Sicheng Zhao, Bo Li, Xiangyu Yue, Yang Gu, Pengfei Xu, Runbo Hu, Hua Chai, and
  Kurt Keutzer.
\newblock Multi-source domain adaptation for semantic segmentation.
\newblock {\em Advances in Neural Information Processing Systems}, 32, 2019.

\bibitem{zhao2020multi}
Sicheng Zhao, Guangzhi Wang, Shanghang Zhang, Yang Gu, Yaxian Li, Zhichao Song,
  Pengfei Xu, Runbo Hu, Hua Chai, and Kurt Keutzer.
\newblock Multi-source distilling domain adaptation.
\newblock In {\em Proceedings of the AAAI Conference on Artificial
  Intelligence}, volume~34, pages 12975--12983, 2020.

\bibitem{zhou2022domain}
Kaiyang Zhou, Ziwei Liu, Yu Qiao, Tao Xiang, and Chen~Change Loy.
\newblock Domain generalization: A survey.
\newblock {\em IEEE Transactions on Pattern Analysis and Machine Intelligence},
  2022.

\bibitem{zhou2021domain}
Kaiyang Zhou, Yongxin Yang, Yu Qiao, and Tao Xiang.
\newblock Domain generalization with mixstyle.
\newblock {\em arXiv preprint arXiv:2104.02008}, 2021.

\bibitem{zhuang2020comprehensive}
Fuzhen Zhuang, Zhiyuan Qi, Keyu Duan, Dongbo Xi, Yongchun Zhu, Hengshu Zhu, Hui
  Xiong, and Qing He.
\newblock A comprehensive survey on transfer learning.
\newblock {\em Proceedings of the IEEE}, 109(1):43--76, 2020.

\end{thebibliography}


}
\end{document}